\documentclass[11pt]{article}
\usepackage{amsthm}

\usepackage{graphics}
\usepackage{graphicx}
\usepackage{amssymb}
\usepackage{verbatim}     
\usepackage{amsxtra}
\usepackage{epsfig}
\usepackage{subfigure}
\usepackage{amsmath}
\usepackage{latexsym}
\usepackage{ifthen}
\usepackage{psfrag}
\usepackage{subfigure}
\usepackage{color}
\usepackage{float}
\usepackage{tikz}
\usepackage{mathtools}
\usepackage{etoolbox}
\usepackage{algorithm}
\usepackage{algpseudocode}
\let\labelindent\relax
\usepackage{enumitem}
\usepackage{import}
\usepackage{pdfpages}
\usepackage{./macros/rgsEnvironments}
\usepackage{./macros/rgsMacros}  
\usepackage{cite} 
\usepackage{hyperref} 

\newbool{conf}
\setbool{conf}{true}
 \newcounter{theorem}
\newtheorem{problem}[theorem]{Problem}

\newlist{steps}{enumerate}{1}
\setlist[steps, 1]{leftmargin=45pt, label = \textbf{Step \arabic*}:}

\DeclareMathOperator{\interior}{int}
\DeclareMathOperator*{\argmin}{arg\,min}

\usetikzlibrary{shapes,arrows}

\ifbool{conf}{
	
}{
	\addtolength{\oddsidemargin}{-.875in}
	\addtolength{\evensidemargin}{-.875in}
	\addtolength{\textwidth}{1.75in}
	\addtolength{\topmargin}{-.875in}
	\addtolength{\textheight}{1.75in} 
}

\begin{document}

\ifbool{conf}{
	\title{\LARGE \bf
		A Rapidly-Exploring Random Trees Motion Planning Algorithm for Hybrid Dynamical Systems*
	}
	\author{Nan Wang and Ricardo G. Sanfelice
		\thanks{*Research partially supported by NSF Grants no. ECS-1710621, CNS-2039054, and CNS-2111688, by AFOSR Grants no. FA9550-19-1-0053, FA9550-19-1-0169, and FA9550-20-1-0238, and by ARO Grant no. W911NF-20-1-0253. }
		\thanks{Nan Wang and Ricardo G. Sanfelice are with the Department of Electrical and Computer Engineering, University of California, Santa Cruz, CA 95064, USA;
			{\tt\small nanwang@ucsc.edu, ricardo@ucsc.edu}}%
	}
	\maketitle
	}{
	\ititle{A Rapidly-Exploring Random Trees Motion Planning Algorithm for Hybrid Dynamical Systems}
	\iauthor{
		Nan Wang \\
		{\normalsize nanwang@ucsc.edu} \\
		Ricardo Sanfelice \\
		{\normalsize ricardo@ucsc.edu}
	}
	\idate{\today{}} 
	\iyear{2019}
	\irefnr{14}
	\makeititle
}
\begin{abstract}
	This paper proposes a rapidly-exploring random trees (RRT) algorithm to solve the motion planning problem for hybrid systems. At each iteration, the proposed algorithm, called HyRRT, randomly picks a state sample and extends the search tree by flow or jump, which is also chosen randomly when both regimes are possible. Through a definition of concatenation of functions defined on hybrid time domains, we show that HyRRT is probabilistically complete, namely, the probability of failing to find a motion plan approaches zero as the number of iterations of the algorithm increases. This property is guaranteed under mild conditions on the data defining the motion plan, which include a relaxation of the usual positive clearance assumption imposed in the literature of classical systems.
	The motion plan is computed through the solution of two optimization problems, one associated with the flow and the other with the jumps of the system. The proposed algorithm is applied to \ifbool{conf}{a walking robot}{an actuated bouncing ball system and a walking robot} so as to highlight its generality and computational features.
\end{abstract}
 
\section{Introduction}
Motion planning consists of finding a state trajectory and associated inputs, connecting the initial and final state while satisfying the dynamics of the system as well as a given safety criterion. Motion planning problems for purely continuous-time systems and purely discrete-time systems have been well studied in the literature; see, e.g., \cite{lavalle2006planning}. In recent years, various planning algorithms have been developed to solve motion planning problems, from  graph search algorithms \cite{likhachev2005anytime} to artificial potential field methods \cite{khatib1986real}. A main drawback of graph search algorithms is that the number of vertices grows exponentially as the dimension of states grows, which makes computing motion plans inefficient for high-dimensional systems. The artificial potential field method suffers from getting stuck at local minimum. Arguably, the most successful algorithm to solve  motion planning problems for purely continuous-time systems and purely discrete-time systems is the sampling-based RRT algorithm \cite{lavalle1998rapidly}. This algorithm incrementally constructs a tree of state trajectories toward random samples in the state space. Similar to graph search algorithms, RRT suffers from the curse of dimensionality, but, in practice, achieves rapid exploration in solving high-dimensional motion planning problems \cite{cheng2005sampling}. Compared with the artificial potential field method, RRT is probabilistically complete \cite{lavalle2001randomized}, which means that the probability of failing to find a motion plan converges to zero, as the number of samples approaches infinity.

While RRT algorithms have been used to solve motion planning problems for purely continuous-time systems \cite{lavalle2001randomized} and purely discrete-time systems \cite{branicky2003rrts}, fewer efforts have been devoted to applying RRT-type algorithms to solve motion planning problems for systems with combined continuous and discrete behavior. In \cite{branicky2003sampling}, a hybrid RRT algorithm is proposed for motion planning problems for a special class of hybrid systems, which follows the classical RRT scheme but does not establish key properties of the algorithm, such as probabilistic completeness.

This paper focuses on motion planning problems for hybrid systems modeled as hybrid equations \cite{goebel2009hybrid}. In this modeling framework, differential and difference equations with constraints are used to describe the continuous and discrete behavior of the hybrid system, respectively. This general hybrid system framework can capture most hybrid systems emerging in robotic applications, not only the class of hybrid systems considered in \cite{branicky2003sampling}, but also systems with memory states, timers, impulses, and constraints. 
For this broad class of hybrid systems, a motion planning algorithm is proposed in this paper. Following \cite{lavalle2001randomized}, the proposed algorithm, called HyRRT, incrementally constructs search trees, rooted in the initial state set and toward the random samples. At first, HyRRT draws samples from the state space. Then, it selects the vertex such that the state associated with this vertex has minimal distance to the sample. Next, HyRRT propagates the state trajectory from the state associated with the selected vertex. Following \cite{kleinbort2018probabilistic}, it is established that, under certain assumptions, HyRRT is probabilistically complete. To the authors' best knowledge, HyRRT is the first RRT-type algorithm for systems with hybrid dynamics that is probabilistically complete. The proposed algorithm is applied to \ifbool{conf}{a walking robot example}{an actuated bouncing ball system and a walking robot example} so as to assess its capabilities. \ifbool{conf}{}{}

The remainder of the paper is structured as follows. Section \ref{section:preliminary} presents notation and preliminaries. Section \ref{section:problemstatement} presents the problem statement and introduction of  \ifbool{conf}{application}{applications}. Section \ref{section:hybridRRT} presents the HyRRT algorithm. Section \ref{section:pc} presents the analysis of the probabilistic completeness of HyRRT algorithm. Section \ref{section:illustration} presents the illustration of HyRRT in \ifbool{conf}{the example}{examples}. \ifbool{conf}{Due to space constraints, proofs will be published elsewhere.}{}
\section{Notation and Preliminaries}
\label{section:preliminary}
\subsection{Notation}
The real numbers are denoted as $\mathbb{R}$ and its nonnegative subset is denoted as $\mathbb{R}_{\geq 0}$. The set of nonnegative integers is denoted as $\mathbb{N}$. The notation $\interior I$ denotes the interior of the interval $I$. The notation $\overline{S}$ denotes the closure of the set $S$. The notation $\partial S$ denotes the boundary of the set $S$. Given sets $P\subset\mathbb{R}^n$ and $Q\subset\mathbb{R}^n$, the Minkowski sum of $P$ and $Q$, denoted as $P + Q$, is the set $\{p + q: p\in P, q\in Q\}$. The notation $|\cdot|$ denotes the Euclidean norm. 
The notation $\rge f$ denotes the range of the function $f$.
Given a point $x\in \mathbb{R}^{n}$ and a subset $S\subset \mathbb{R}^{n}$, the distance between $x$ and $S$ is denoted $\text{dist}(x, S) := \inf_{s\in S} |x - s|$. The notation $\mathbb{B}$ denotes the closed unit ball of appropriate dimension in the Euclidean norm. 
\ifbool{conf}{}{The probability of the event $M$ is denoted as $Pr(M)$. Given a set $S$, the notation $\mu(S)$ denotes its Lebesgue measure.}
\ifbool{conf}{}{The Lebesgue measure of the $n$-th dimensional unit ball, denoted $\zeta_{n}$, is such that
\begin{equation}
	\label{equation:zetan}
		\zeta_{n} := \left\{\begin{aligned}
		&\frac{\pi^{k}}{k!} &\text{ if } n = 2k, k\in \mathbb{N}\\
		&\frac{2(k!)(4\pi)^{k}}{(2k + 1)!} &\text{ if } n = 2k + 1, k\in \mathbb{N}.\\
		\end{aligned}\right.
\end{equation}; see \cite{gipple2014volume}.}

\subsection{Preliminaries}
A hybrid system $\mathcal{H}$ with inputs is modeled as  \cite{goebel2009hybrid}
\begin{equation}
\mathcal{H}: \left\{              
\begin{aligned}               
\dot{x} & = f(x, u)     &(x, u)\in C\\                
x^{+} & =  g(x, u)      &(x, u)\in D\\                
\end{aligned}   \right. 
\label{model:generalhybridsystem}
\end{equation}
where $x\in \mathbb{R}^n$ is the state, $u\in \mathbb{R}^m$ is the input, $C\subset \mathbb{R}^{n}\times\mathbb{R}^{m}$ represents the flow set, $f: \mathbb{R}^{n}\times\mathbb{R}^{m} \to \mathbb{R}^{n}$ represents the flow map, $D\subset \mathbb{R}^{n}\times\mathbb{R}^{m}$ represents the jump set, and $g:\mathbb{R}^{n}\times\mathbb{R}^{m} \to \mathbb{R}^{n}$ represents the jump map, respectively. The continuous evolution of $x$ is captured by the flow map $f$. The discrete evolution of $x$ is captured by the jump map $g$. The flow set $C$ collects the points where the state can evolve continuously. The jump set $D$ collects the points where jumps can occur.

Given a flow set $C$, the set $U_{C} := \{u\in \mathbb{R}^{m}: \exists x\in \mathbb{R}^{n}\text{ such that } (x, u)\in C\}$ includes all possible input values that can be applied during flows. Similarly, given a jump set $D$, the set $U_{D} := \{u\in \mathbb{R}^ {m}: \exists x\in \mathbb{R}^{n}\text{ such that } (x, u)\in D\}$ includes all possible input values that can be applied at jumps. These sets satisfy $C\subset \mathbb{R}^{n}\times U_{C}$ and $D\subset \mathbb{R}^{n}\times U_{D}$. Given a set $K\subset \mathbb{R}^{n}\times U_{\star}$, where $\star$ is either $C$ or $D$, we define
	$
	\Pi_{\star}(K) := \{x: \exists u\in U_{\star} \text{ s.t. } (x, u)\in K\}
	$
	as the projection of $K$ onto $\mathbb{R}^{n}$, and define \ifbool{conf}{$C' := \Pi_{C}(C)$ and $D' := \Pi_{D}(D)$.}{
	\begin{equation}
	\label{equation:Cprime}
	C' := \Pi_{C}(C)
	\end{equation} and 
	\begin{equation}
	\label{equation:Dprime}
	D' := \Pi_{D}(D).
	\end{equation}}

In addition to ordinary time $t\in \mathbb{R}_{\geq 0}$, we employ $j\in \mathbb{N}$ to denote the number of jumps of the evolution of $x$ and $u$ for $\mathcal{H}$ in (\ref{model:generalhybridsystem}), leading to hybrid time $(t, j)$ for the parameterization of its solutions and inputs. The domain of a solution to $\mathcal{H}$ is given by a hybrid time domain. A hybrid time domain is defined as a subset $E$ of $\mathbb{R}_{\geq 0}\times \mathbb{N}$ that, for each $(T, J)\in E$, $E\cap ([0, T]\times \{0, 1,..., J\})$ can be written as $\cup_{j = 0}^{J}([t_{j}, t_{j+1}],j)$ for some finite sequence of times $0=t_{0}\leq t_{1}\leq t_{2}\leq ... \leq t_{J+1} = T$. A hybrid arc $\phi: \dom \phi \to \mathbb{R}^{n}$ is a function on a hybrid time domain that, for each $j\in \mathbb{N}$, $t\mapsto \phi(t,j)$ is locally absolutely continuous on each interval $I^{j}:=\{t:(t, j)\in \dom \phi\}$ with nonempty interior. The definition of solution pair to a hybrid system is given as follows. For more details, see \cite{goebel2009hybrid}.
\begin{definition} 
	\label{definition:solution}
	(Solution pair to a hybrid system) Given a pair of functions $\phi:\dom \phi\to \mathbb{R}^{n}$ and $u:\dom u\to \mathbb{R}^{m}$, $(\phi, u)$ is a solution pair to (\ref{model:generalhybridsystem}) if $\dom (\phi, u) := \dom \phi = \dom u$ is a hybrid time domain, $(\phi(0,0), u(0,0))\in \overline{C} \cup D$, and the following hold:
	\begin{enumerate}[label=\arabic*)]
		\item For all $j\in \mathbb{N}$ such that $I^{j}$ has nonempty interior, 
		\begin{enumerate}[label=\alph*)]
			\item the function $t\mapsto \phi(t, j)$ is locally absolutely continuous,
			\item $(\phi(t, j),u(t, j))\in C$ for all $t\in \interior I^j$,
			\item the function $t\mapsto u(t,j)$ is Lebesgue measurable and locally bounded,
			\item for almost all $t\in I^j$, \ifbool{conf}{$\dot{\phi}(t,j) = f(\phi(t,j), u(t,j))$.}{
			\begin{equation}
			\dot{\phi}(t,j) = f(\phi(t,j), u(t,j)).
			\end{equation}
		}
		\end{enumerate}
		\item For all $(t,j)\in \dom (\phi, u)$ such that $(t,j + 1)\in \dom (\phi, u)$, \ifbool{conf}{
			$$
			(\phi(t, j), u(t, j))\in D \quad
			\phi(t,j+ 1) = g(\phi(t,j), u(t, j)).
			$$}{
		\begin{equation}
		\begin{aligned}
		(\phi(t, j), u(t, j))&\in D\\
		\phi(t,j+ 1) &= g(\phi(t,j), u(t, j)).
		\end{aligned}
		\end{equation}
	}
	\end{enumerate}
\end{definition}

HyRRT requires concatenating solution pairs. The concatenation operation of solution pairs is defined next.
\begin{definition}
	\label{definition:concatenation}
	(Concatenation operation) Given two functions $\phi_{1}: \dom \phi_{1} \to \mathbb{R}^{n}$ and $\phi_{2}:\dom \phi_{2} \to \mathbb{R}^{n}$, where $\dom \phi_{1}$ and $\dom \phi_{2}$ are hybrid time domains, $\phi_{2}$ can be concatenated to $\phi_{1}$ if $ \phi_{1}$ is compact and $\phi: \dom \phi \to \mathbb{R}^n$ is the concatenation of $\phi_{2}$ to $\phi_{1}$, denoted $\phi = \phi_{1}|\phi_{2}$, namely,
	\begin{enumerate}[label=\arabic*)]
		\item $\dom \phi = \dom \phi_{1} \cup (\dom \phi_{2} + \{(T, J)\}) $, where $(T, J) = \max \dom \phi_{1}$ and the plus sign denotes Minkowski addition;
		\item $\phi(t, j) = \phi_{1}(t, j)$ for all $(t, j)\in \dom \phi_{1}\backslash \{(T, J)\}$ and $\phi(t, j) = \phi_{2}(t - T, j - J)$ for all $(t, j)\in \dom \phi_{2} + \{(T, J)\}$.
	\end{enumerate}
\end{definition}

\ifbool{conf}{}{
The proof in this paper requires truncating functions defined on hybrid time domains. The truncation operation is defined next.
\begin{definition}
	\label{definition: truncation}
	(Truncation and translation operation) Given a function $\phi: \dom \phi \to \mathbb{R}^{n}$, where $\dom \phi$ is hybrid time domain, and pairs of hybrid time $(T_{1}, J_{1})\in \dom \phi$ and $(T_{2}, J_{2})\in \dom \phi$ such that $T_{1}\leq T_{2}$ and $J_{1} \leq J_{2}$, the function $\widetilde\phi: \dom \widetilde\phi \to \mathbb{R}^{n}$ is the truncation of $\phi$ between $(T_{1}, J_{1})$ and $(T_{2}, J_{2})$ and translation by $(T_{1}, J_{1})$  if
	\begin{enumerate}
		\item $\dom \widetilde\phi =  \dom \phi \cap ([T_{1}, T_{2}]\times \{J_{1}, J_{1} + 1, ..., J_{2}\}) - \{(T_{1}, J_{1})\}$, where the minus sign denotes Minkowski difference;
		\item $\widetilde\phi(t, j) = \phi(t + T_{1}, j + J_{1})$ for all $(t, j)\in \dom \widetilde\phi$.
	\end{enumerate}
\end{definition}

}
In the main result of this paper, the following definition of closeness between hybrid arcs is used; see \cite{goebel2009hybrid}.
\begin{definition}\label{definition:closeness}
	($(\tau, \epsilon)$-closeness of hybrid arcs) Given $\tau, \epsilon>0$, two hybrid arcs $\phi_{1}$ and $\phi_{2}$ are $(\tau, \epsilon)$-close if
	\begin{enumerate}
		\item for all $(t, j)\in \dom \phi_{1}$ with $t + j \leq \tau$, there exists $s$ such that $(s, j)\in \dom \phi_{2}$, $|t - s|< \epsilon$, and $|\phi_{1}(t, j) - \phi_{2}(s, j)| < \epsilon$;
		\item for all $(t, j)\in \dom \phi_{2}$ with $t + j \leq \tau$, there exists $s$ such that $(s, j)\in \dom \phi_{1}$, $|t - s|< \epsilon$, and $|\phi_{2}(t, j) - \phi_{1}(s, j)| < \epsilon$.
	\end{enumerate}
\end{definition}

\section{Problem Statement and Applications}
\label{section:problemstatement}
The motion planning problem for hybrid systems studied in this paper is  formulated as follows.
\begin{problem}
	\label{problem:motionplanning}
	Given a hybrid system $\mathcal{H}$ with input $u\in \mathbb{R}^{m}$ and state $x\in \mathbb{R}^{n}$, the initial state set $X_{0}\subset\mathbb{R}^{n}$, the final state set $X_{f}\subset\mathbb{R}^{n}$, and the unsafe set $X_{u}\subset\mathbb{R}^{n}\times \mathbb{R}^{m}$, find a pair $(\phi, u): \dom (\phi, u)\to \mathbb{R}^{n}\times \mathbb{R}^{m}$, namely, \emph{a motion plan}, such that for some $(T, J)\in \dom (\phi, u)$, the following hold:
	\begin{enumerate}[label=\arabic*)]
		\item $\phi(0,0) \in X_{0}$, namely, the initial state of the solution belongs to the given initial state set $X_{0}$;
		\item $(\phi, u)$ is a solution pair to $\mathcal{H}$ as defined in Definition \ref{definition:solution};
		\item $(T,J)$ is such that $\phi(T,J)\in X_{f}$, namely, the solution belongs to the final state set at hybrid time $(T, J)$;
		\item $(\phi(t,j), u(t, j))\notin X_{u}$ for each $(t,j)\in \dom (\phi, u)$ such that $t + j \leq T+ J$, namely, the solution pair does not intersect with the unsafe set before its state trajectory reaches the final state set.
	\end{enumerate}
Therefore, given sets $X_{0}$, $X_{f}$ and $X_{u}$, and a hybrid system $\mathcal{H}$ with data $(C, f, D, g)$, a motion planning problem $\mathcal{P}$ is formulated as
$
	\mathcal{P} = (X_{0}, X_{f}, X_{u}, (C, f, D, g)).
$
\end{problem}

\ifbool{conf}{There are some interesting special  cases of Problem \ref{problem:motionplanning}. For example, when $D = \emptyset$ ($C = \emptyset$) and $C$ ($D$) is nonempty, then $\mathcal{P}$ denotes the motion planning problem for purely continuous-time (discrete-time, respectively) systems under constraints. Therefore, Problem \ref{problem:motionplanning} covers the motion planning problems for purely continuous-time and purely discrete-time system studied in \cite{lavalle2001randomized} and \cite{lavalle2006planning}. Moreover, note that the unsafe set $X_{u}$ can be used to constrain both states and inputs.}{There are some special cases of Problem \ref{problem:motionplanning}. For example, when $D = \emptyset$ and $C$ is nonempty, then $\mathcal{P}$ denotes the motion planning problem for purely continuous-time systems under constraints. In addition, when $C = \emptyset$ and D is nonempty, then $\mathcal{P}$ denotes the motion planning problem for purely discrete-time systems under constraints. Therefore, Problem \ref{problem:motionplanning} covers the motion planning problems for purely continuous-time and purely discrete-time system studied in \cite{lavalle2001randomized} and \cite{lavalle2006planning}. Moreover, the unsafe set $X_{u}$ can be used to constrain both states and inputs.}

Problem \ref{problem:motionplanning} is illustrated in the following \ifbool{conf}{example}{examples}.
\ifbool{conf}{
\begin{example}(Walking robot)\label{example:biped}
		The state $x$ of the compass model of a walking robot is composed of the angle vector $\theta$ and the velocity vector $\omega$ \cite{grizzle2001asymptotically}. 
		The angle vector $\theta$ contains the planted leg angle $\theta_{p}$, the swing leg angle $\theta_{s}$, and the torso angle $\theta_{t}$. The velocity vector $\omega$ contains the planted leg angular velocity $\omega_{p}$, the swing leg angular velocity $\omega_{s}$, and the torso angular velocity $\omega_{t}$. The input $u$ is the input torque, where $u_{p}$ is the torque applied on the planted leg from the ankle, $u_{s}$ is the torque applied on the swing leg from the hip, and $u_{t}$ is the torque applied on the torso from the hip.
		The continuous dynamics of $x = (\theta, \omega)$ comes from the Lagrangian method and is given by \ifbool{conf}{$\dot{\theta} = \omega, \dot{\omega} = D_{f}(\theta)^{-1}( - C_{f}(\theta, \omega)\omega - G_{f}(\theta) + Bu) = : \alpha(x, u)$}{
		\begin{equation}
		\label{equation:omegaflowmap1}
		\dot{\theta} = \omega\quad \dot{\omega} = D_{f}(\theta)^{-1}( - C_{f}(\theta, \omega)\omega - G_{f}(\theta) + Bu) = : \alpha(x, u)
		\end{equation}}
		where $D_{f}$ and $C_{f}$ are the inertial and Coriolis matrices, respectively, and $B$ is the actuator  relationship matrix.
		\ifbool{conf}{}{
			\begin{figure}[htbp] 
				\centering
	\def\svgwidth{0.35\columnwidth}
	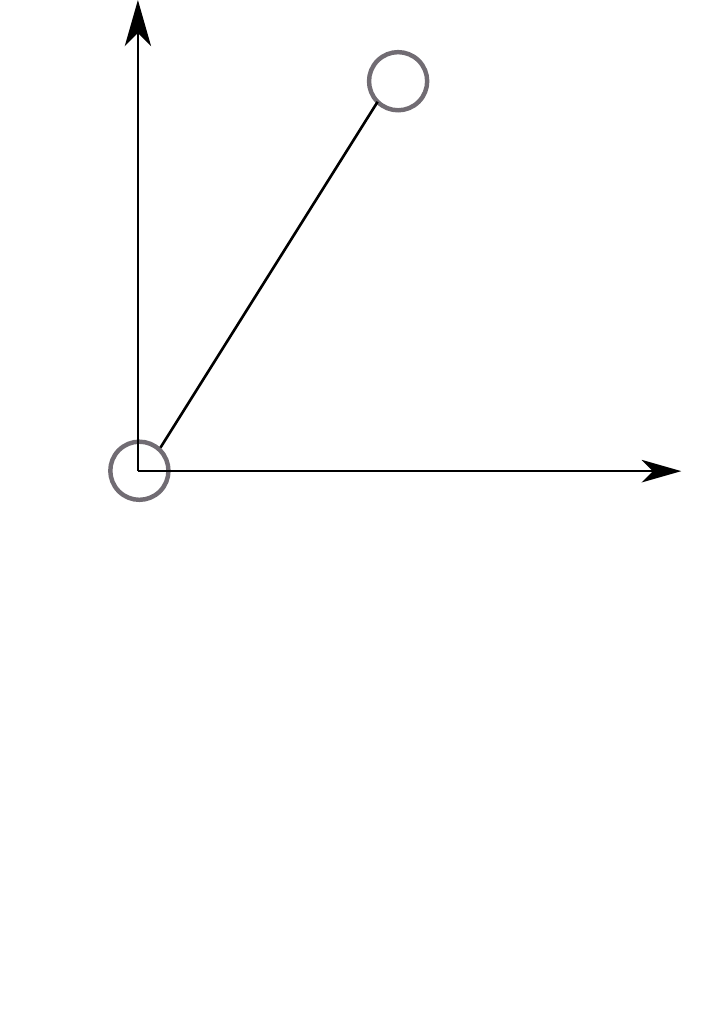

				\caption{The biped system in Example \ref{example:biped}. The angle vector $\theta$ contains the planted leg angle $\theta_{p}$, the swing leg angle $\theta_{s}$, and the torso angle $\theta_{t}$. The velocity vector $\omega$ contains the planted leg angular velocity $\omega_{p}$, the swing leg angular velocity $\omega_{s}$, and the torso angular velocity $\omega_{t}$. The input $u$ is the input torque, where $u_{p}$ is the torque applied on the planted leg from the ankle, $u_{s}$ is the torque applied on the swing leg from the hip, and $u_{t}$ is the torque applied on the torso from the hip.} \label{fig:bipedsystem} 
			\end{figure} 
		}
		In \cite{short2018hybrid}, the input torques that produce an acceleration  $a$ for a special state $x$ are determined by a function $\mu$, defined as \ifbool{conf}{$\mu(x, a) := B^{-1}(D_{f}(\theta)a + C_{f}(\theta, \omega)\omega + G_{f}(\theta)).$}{
		\begin{equation}
		\label{equation:bipedpreforwardcontrol}
		\mu(x, a) := B^{-1}(D_{f}(\theta)a + C_{f}(\theta, \omega)\omega + G_{f}(\theta)).
		\end{equation}}
		By applying $u = \mu(x, a)$ to $\dot{\omega} = \alpha(x, u)$, we obtain
		$
		\label{equation:omegaflowmap}
		\dot{\omega} = a.
		$
		Then, the flow map $f$ is defined as 
		$$
		\label{model:bipedflowmap}
		f(x, a) := \left[\begin{matrix}
		\omega\\
		a
		\end{matrix}\right] \quad \forall (x, a)\in C.
		$$
		
		Flow is allowed when only one leg is in contact with the ground. To determine if the biped has reached the end of a step, a function $h$ is defined as
		$
		h(x) := \phi_{s} - \theta_{p}$ for all $x\in \mathbb{R}^{6}
		$
		where $\phi_{s}$ denotes the step angle.
		The condition $h(x) \geq 0$ indicates that only one leg is in contact with the ground. Thus, the flow set is given as 
		$
		\label{model:bipedflowset}
		C:= \{(x, a)\in \mathbb{R}^{6}\times \mathbb{R}^{3}: h(x)\geq 0\}.
		$
		Furthermore, a step occurs when the change of $h$ is such that $\theta_{p}$  is approaching $\phi_{s}$, and $h$ equals zero. Thus, the jump set $D$ is defined as
		$
		\label{model:bipedjumpset}
		D := \{(x, a)\in \mathbb{R}^{6}\times \mathbb{R}^{3}: h(x) = 0, \omega_{p} \geq 0\}.
		$
		
		Following \cite{grizzle2001asymptotically}, when a step occurs, the swing leg becomes the planted leg, and the planted leg becomes the swing leg. The function $\Gamma$ is defined to swap angles and velocity variables as
		$
		\label{equation:thetajumpmap}
		\theta^{+} = \Gamma(\theta).
		$
		The angular velocities after a step are determined by a contact model denoted as
		\ifbool{conf}{$\Omega(x) := (\Omega_{p}(x), \Omega_{s}(x), \Omega_{t}(x))$}{
			$
			\label{equation:omegajumpmap}
			\Omega(x) := \left[\begin{matrix}
			\Omega_{p}(x)\\
			\Omega_{s}(x)\\
			\Omega_{t}(x)
			\end{matrix}\right]
			$}, where $\Omega_{p}$, $\Omega_{s}$, and $\Omega_{t}$ are the angular velocity of the planted leg, swing leg, and torso, respectively. Then, the jump map $g$ is defined as
		\begin{equation}
			\label{equation:bipedjumpmap}
			g(x, a) := \left[\begin{matrix}
			\Gamma(\theta)\\
			\Omega(x)
			\end{matrix}\right]\quad \forall (x, a)\in D.
		\end{equation}
		
		A particular motion planning problem for the walking robot is to generate a walking gait.  The final state set is defined as $X_{f} = \{(\phi_{s}, -\phi_{s}, 0, 0.1, 0.1, 0)\}$ so that after the impact, the walking robot starts the next walking cycle. The initial state set is chosen as $X_{0} = \{x_{0}\in \mathbb{R}^{6}: x_{0} = g(x_{f}, 0), x_{f}\in X_{f}\}$. In setting $X_{0}$, the input argument of $g$ can be set arbitrarily because input does not affect the value of $g$; see (\ref{equation:bipedjumpmap}). In practice, there are constraints on the acceleration of the planted leg, swinging leg, and the torso, respectively. To capture these, the unsafe set is defined as $X_{u} = \{(x, a)\in \mathbb{R}^{6}\times \mathbb{R}^{3}: a_{1} \notin [a_{1}^{\min}, a_{1}^{\max}]\text{ or }  a_{2} \notin [a_{2}^{\min}, a_{2}^{\max}]\text{ or } a_{3} \notin [a_{3}^{\min}, a_{3}^{\max}]\text{ or } (x, a)\in D  \}$, where $a_{1}^{\min}$, $a_{2}^{\min}$, and $a_{3}^{\min}$ are the lower bounds of $a_{1}$, $a_{2}$, and  $a_{3}$, respectively, and $a_{1}^{\max}$, $a_{2}^{\max}$, and $a_{3}^{\max}$ are the upper bounds of $a_{1}$, $a_{2}$, and  $a_{3}$, respectively.
		
		\ifbool{conf}{}{In this example, the step angle $\phi_{s}$ is set as $0.70$. The length of the torso $l_{t}$ and the legs $l_{l}$ are set as $1$. The leg mass, hip mass and torso mass are set as $1$. The walking velocity is set as $0.6$. $ a_{1}^{\min}$ and $a_{1}^{\max}$ are set as $-3$ and $3$, respectively. $ a_{2}^{\min}$ and $a_{2}^{\max}$ are set as $-3$ and $3$, respectively. $ a_{3}^{\min}$ and $a_{3}^{\max}$ are set as $-0.2$ and $0.2$, respectively.}
	\end{example}
	In the forthcoming \ifbool{conf}{Example \ref{example:bipedillustration},}{Examples \ref{example:illustrationbb} and \ref{example:bipedillustration},} we employ HyRRT to solve \ifbool{conf}{this motion planning problem formulated in Example \ref{example:biped}}{these motion planning problems}.
}{
\begin{example}\label{example:bouncingball}(Actuated bouncing ball system)
	Consider a ball bouncing on a fixed horizontal surface as is shown in Figure \ref{fig:bouncingball}. The surface is located at the origin and, through control actions, is capable of affecting the velocity of the ball after the impact.  The dynamics of the ball while in the air is given by
	\begin{equation}
	\label{model:bouncingballflow}
	\dot{x} = \left[ \begin{matrix}
	x_{2} \\
	-\gamma
	\end{matrix}\right] =: f(x, u)\qquad \forall(x, u)\in C
	\end{equation}
	where $x :=(x_{1}, x_{2})\in \mathbb{R}^2$. The height of the ball is denoted by $x_{1}$. The velocity of the ball is denoted by $x_{2}$. The gravity constant is denoted by $\gamma$. 
	\begin{figure}[htbp] 
		\centering
		\parbox[h]{\ifbool{conf}{0.27}{0.4}\textwidth}{
			\centering
			\subfigure[The actuated bouncing ball system\label{fig:bouncingball} ]{%
	\def\svgwidth{\ifbool{conf}{0.27}{0.3}\columnwidth}
\begingroup%
  \makeatletter%
  \providecommand\color[2][]{%
    \errmessage{(Inkscape) Color is used for the text in Inkscape, but the package 'color.sty' is not loaded}%
    \renewcommand\color[2][]{}%
  }%
  \providecommand\transparent[1]{%
    \errmessage{(Inkscape) Transparency is used (non-zero) for the text in Inkscape, but the package 'transparent.sty' is not loaded}%
    \renewcommand\transparent[1]{}%
  }%
  \providecommand\rotatebox[2]{#2}%
  \newcommand*\fsize{\dimexpr\f@size pt\relax}%
  \newcommand*\lineheight[1]{\fontsize{\fsize}{#1\fsize}\selectfont}%
  \ifx\svgwidth\undefined%
    \setlength{\unitlength}{155.66442535bp}%
    \ifx\svgscale\undefined%
      \relax%
    \else%
      \setlength{\unitlength}{\unitlength * \real{\svgscale}}%
    \fi%
  \else%
    \setlength{\unitlength}{\svgwidth}%
  \fi%
  \global\let\svgwidth\undefined%
  \global\let\svgscale\undefined%
  \makeatother%
  \begin{picture}(1,0.86048663)%
    \lineheight{1}%
    \setlength\tabcolsep{0pt}%
    \put(0,0){\includegraphics[width=\unitlength,page=1]{bouncingballfigure.pdf}}%
    \put(0.03673765,0.65008529){\rotatebox{90}{\makebox(0,0)[rt]{\lineheight{1.25}\smash{\begin{tabular}[t]{r}height ($x_{1}$)\end{tabular}}}}}%
    \put(0,0){\includegraphics[width=\unitlength,page=2]{bouncingballfigure.pdf}}%
    \put(0.55137666,0.30449301){\makebox(0,0)[lt]{\lineheight{1.25}\smash{\begin{tabular}[t]{l}control\\input $(u)$\end{tabular}}}}%
  \end{picture}%
\endgroup%

}
		}
		\parbox[h]{\ifbool{conf}{0.20}{0.4}\textwidth}{
			\centering
			\subfigure[A motion plan to the sample motion planning problem for actuated  bouncing ball system.\label{fig:samplesolutionbb}]{\includegraphics[width=\ifbool{conf}{0.20}{0.3}\textwidth]{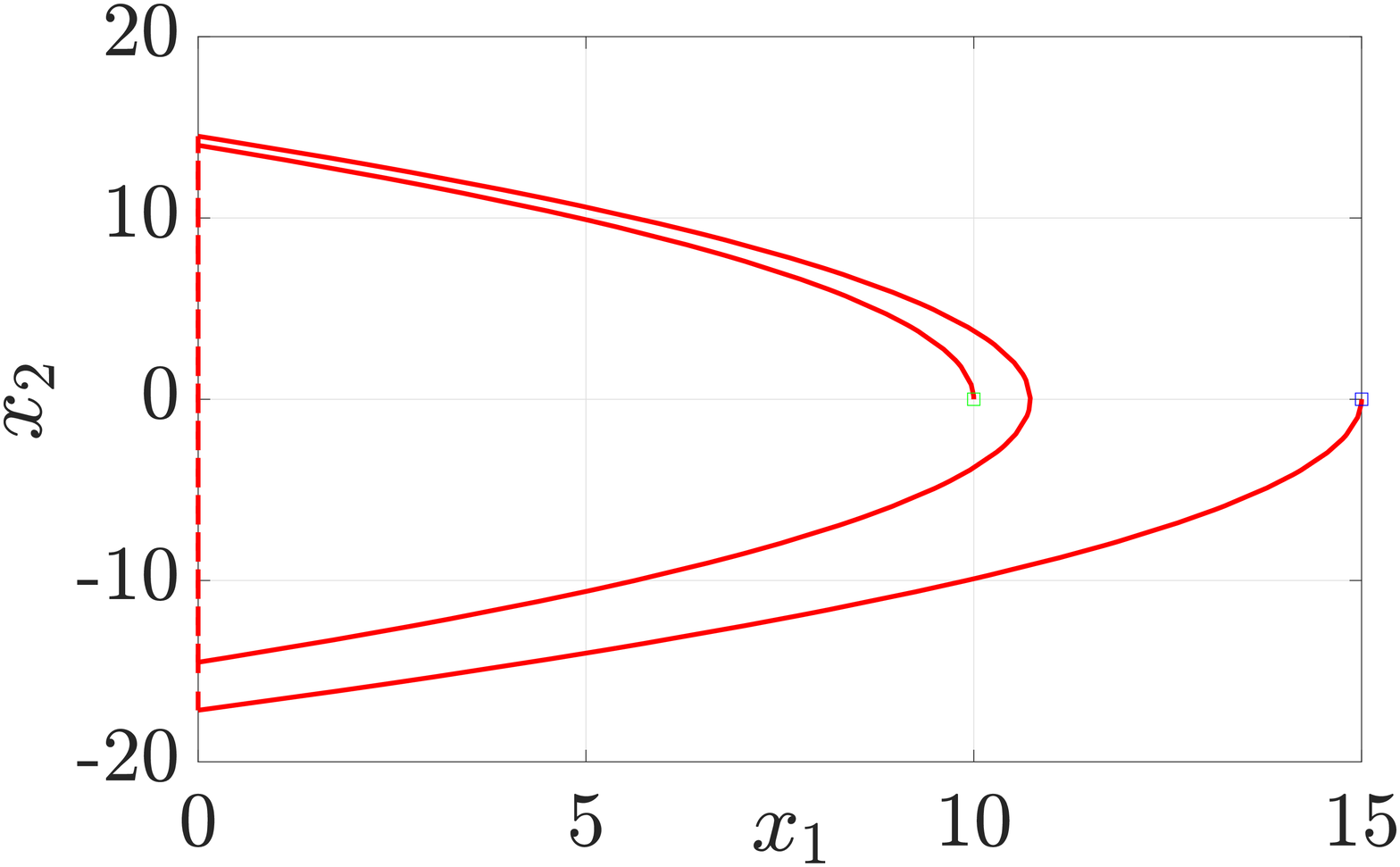}}
		}
		\caption{The actuated bouncing ball system in Example \ref{example:bouncingball}.}
	\end{figure} 
	The flow is allowed when the ball is above the surface. Hence, the flow set is
	\begin{equation}
	\label{model:bouncingballflowset}
		C := \{(x, u)\in \mathbb{R}^{2}\times \mathbb{R}: x_{1}\geq 0\}.
	\end{equation}
	At every impact, the velocity of the ball changes from pointing down to pointing up while the height remains the same. The dynamics at jumps of the actuated bouncing ball system is given as
	\begin{equation}
	x^{+} = \left[ \begin{matrix}
	x_{1} \\
	-\lambda x_{2}+u
	\end{matrix}\right] =: g(x, u)\qquad \forall (x, u)\in D
	\label{conservationofmomentum}
	\end{equation}
	where $u\geq 0$ is the input and $\lambda \in (0,1)$ is the coefficient of restitution. 
	The jump is allowed when the ball is on the surface with negative velocity. Hence, the jump set is 
	\begin{equation}
	\label{model:bouncingballjumpset}
		D:= \{(x, u)\in \mathbb{R}^{2}\times \mathbb{R}: x_{1} = 0, x_{2} \leq 0, u\geq 0\}.
	\end{equation}
	
	The hybrid model of the actuated bouncing ball system is given by  (\ref{model:generalhybridsystem}) where the flow map $f$ is given in (\ref{model:bouncingballflow}), the flow set $C$ is given in (\ref{model:bouncingballflowset}), the jump map $g$ is given in (\ref{conservationofmomentum}), and the jump set $D$ is given in (\ref{model:bouncingballjumpset}).

An instance of motion planning problem for the actuated bouncing ball system is given as follows.  The initial state set is $X_{0} = \{(15, 0)\}$. The final state set is $X_{f} = \{(10, 0)\}$. The unsafe set is $X_{u} =\{(x, u)\in \mathbb{R}^{2}\times \mathbb{R}: u\in[5, \infty)\}$.  The motion planning problem $\mathcal{P}$ is given as $\mathcal{P} = (X_{0}, X_{f}, X_{u}, (C, f, D, g))$. The state trajectory of this motion plan is shown in Figure \ref{fig:samplesolutionbb}. In the figure, the initial state set is denoted by a blue square. The final state set is denoted by a  green square. The red trajectory denotes the state trajectory of a sample motion plan.
\end{example}

}

\section{HyRRT: A Motion Planning Algorithm for Hybrid Systems}
\label{section:hybridRRT}
\ifbool{conf}{}{In this section, we introduce HyRRT, an algorithm to solve the motion planning problem $\mathcal{P} = (X_{0}, X_{f}, X_{u}, \mathcal{H})$ formulated in Problem \ref{problem:motionplanning}, where $\mathcal{H} = (C, f, D, g)$. }
	\subsection{Overview}
	\label{section:algorithmoverview}
	HyRRT searches for a motion plan by incrementally constructing a search tree. 
	The search tree is modeled by a directed tree. A directed tree $\mathcal{T}$ is a pair $\mathcal{T} = (V, E)$, where $V$ is a set whose elements are called vertices and $E$ is a set of paired vertices whose elements are called edges. The edges in the directed tree are directed, which means the pairs of vertices that represent edges are ordered. The set of edges $E$ is defined as
	$
	E \subseteq \{(v_{1}, v_{2}): v_{1}\in V, v_{2}\in V, v_{1}\neq v_{2}\}.
	$
	The edge $e = (v_{1}, v_{2})\in E$ represents an edge from $v_{1}$ to $v_{2}$.  
	A path in $\mathcal{T} = (V, E)$ is a sequence of vertices
	$
	p = (v_{1}, v_{2}, ..., v_{k})
	$ such that $(v_{i}, v_{i + 1})\in E$ for all $i= 1, 2,..., k - 1$.
	 
	\begin{figure}[htbp]
		\centering
				\ifbool{conf}{\vspace{-0.4cm}}{}
		\subfigure[States and solution pairs.\label{fig:searchgraph_statespace}]{%
	\def\svgwidth{0.7\columnwidth}
	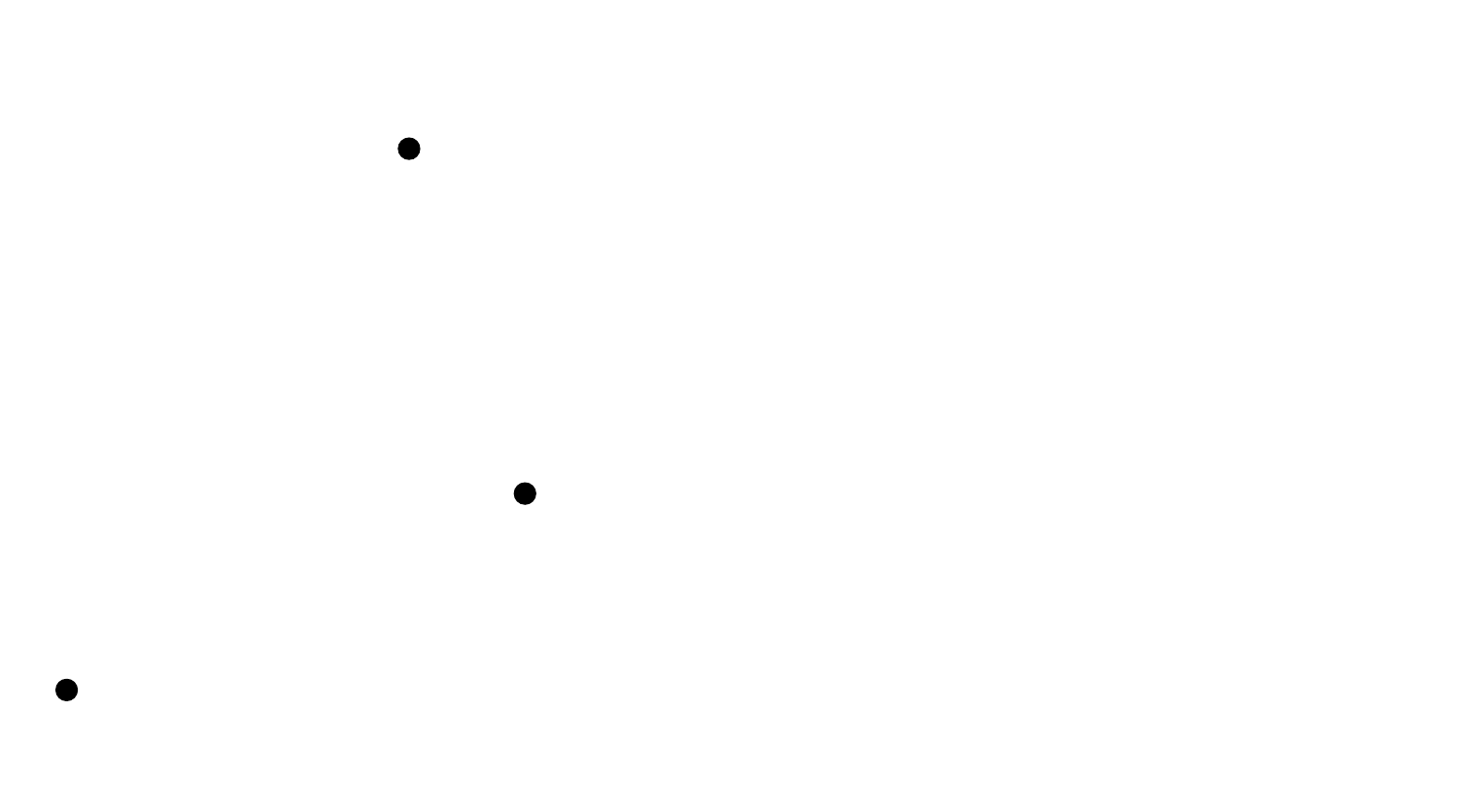
}
		\subfigure[Search tree associated with the states and solution pairs in Figure \ref{fig:searchgraph_statespace}.\label{fig:searchgraph_searchgraph}]{%
	\def\svgwidth{0.7\columnwidth}
	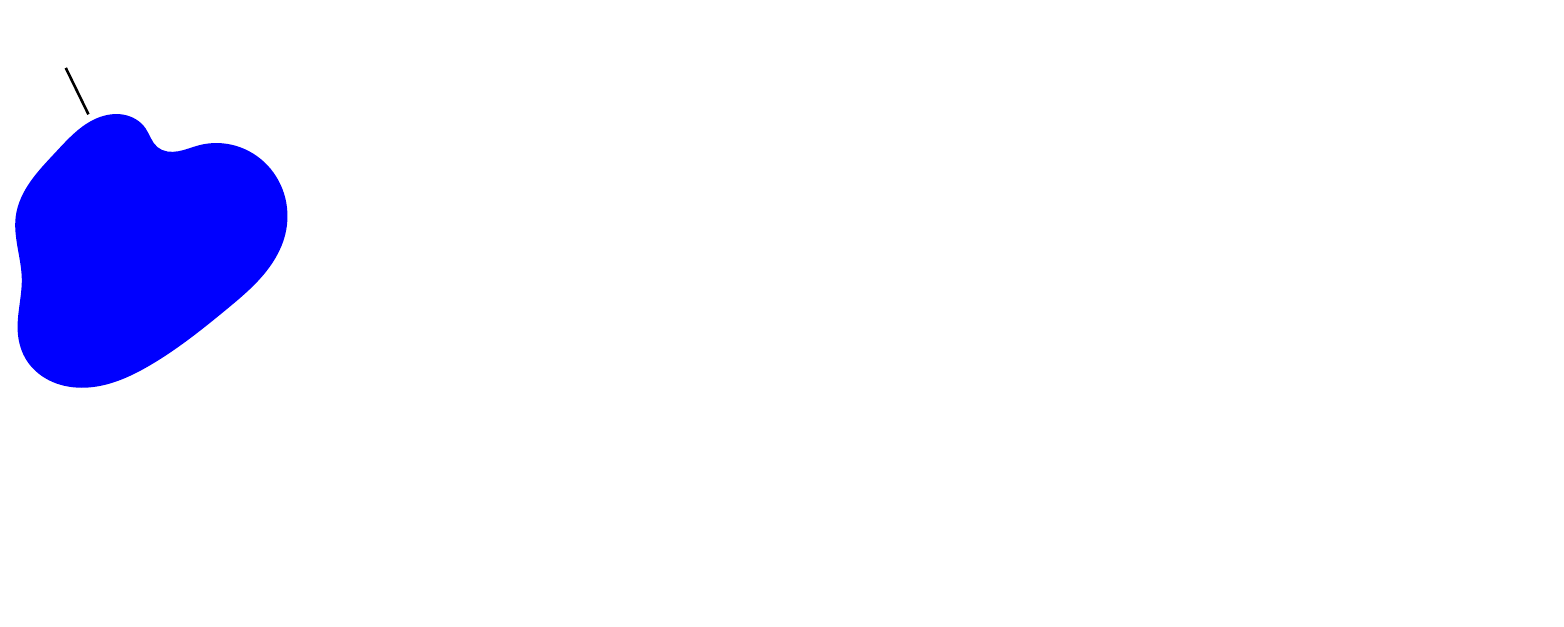
}
		\ifbool{conf}{\vspace{-0.2cm}}{}
		\caption{The association between states/solution pairs and the vertices/edges in the search tree. The blue region denotes $X_{0}$, the green region denotes $X_{f}$, and the black region denotes $X_{u}$. The dots and lines between dots in Figure \ref{fig:searchgraph_searchgraph} denote the vertices and edges associated with the states and solution pairs in Figure \ref{fig:searchgraph_statespace}. The path $p = (v_{1}, v_{2}, v_{3}, v_{cur}, v_{new})$ in the search graph in Figure \ref{fig:searchgraph_searchgraph} represents the solution pair $\tilde{\psi}_{p} = \overline{\psi}_{e_{1}}|\overline{\psi}_{e_{2}}|\overline{\psi}_{e_{3}}|\psi_{new}$ in Figure \ref{fig:searchgraph_statespace}.}
		\ifbool{conf}{\vspace{-0.3cm}}{}
		\label{fig:searchgraph}
	\end{figure}
	
	Each vertex in the search tree $\mathcal{T}$ is associated with a state value of $\mathcal{H}$. Each edge in the search tree is associated with a solution pair to $\mathcal{H}$ that connects the state values associated with their endpoint vertices. The state value associated with vertex $v\in V$ is denoted as $\overline{x}_{v}$ and the solution pair associated with edge $e\in E$ is denoted as $\overline{\psi}_{e}$, as shown in Figure \ref{fig:searchgraph}. The solution pair that the path $p = (v_{1}, v_{2}, ..., v_{k})$ represents is the concatenation of all those solutions associated with the edges therein, namely,
	\begin{equation}
	\label{equation:concatenationpath}
	\tilde{\psi}_{p} := \overline{\psi}_{(v_{1}, v_{2})}|\overline{\psi}_{(v_{2}, v_{3})}|\ ...\  |\overline{\psi}_{(v_{k-1}, v_{k})}	
	\end{equation}
	where $\tilde{\psi}_{p}$ denotes the solution pair associated with the path $p$. For the notion of concatenation, see Definition \ref{definition:concatenation}. An example of the path $p$ and its associated solution pair $\tilde{\psi}_{p}$ is shown in Figure \ref{fig:searchgraph}.
	
	\ifbool{conf}{The proposed HyRRT algorithm requires a library of possible inputs. The input library $(\mathcal{U}_{C}, \mathcal{U}_{D})$ includes the input signals that can be applied during flows (collected in $\mathcal{U}_{C}$) and the input values that can be applied at jumps (collected in $\mathcal{U}_{D}$). }{
	}

	Next, we introduce the main steps executed by HyRRT. Given the motion planning problem $\mathcal{P} = (X_{0}, X_{f}, X_{u}, (C, f, D, g))$ and the input library $(\mathcal{U}_{C}, \mathcal{U}_{D})$, HyRRT performs the following steps:
	\begin{steps}
		\item Sample a finite number of points from $X_{0}$ and initialize a search tree $\mathcal{T} = (V, E)$ by adding vertices associated with each sampling point. 
		\item Randomly select one regime among flow regime and jump regime for the evolution of $\mathcal{H}$.
		\item Randomly select a point $x_{rand}$ from $C'$ ($D'$) if the flow (jump, respectively) regime is selected in Step 2.
		\item Find the vertex associated with the state value that has minimal Euclidean distance to $x_{rand}$, denoted $v_{cur}$, as is shown in Figure \ref{fig:searchgraph_searchgraph}. 
		\item Randomly select an input signal (value) from $\mathcal{U}_{C}$ ($\mathcal{U}_{D}$) if the flow (jump, respectively) regime is selected. Then, compute a solution pair starting from $\overline{x}_{v_{cur}}$ with the selected input applied, denoted $\psi_{new} = (\phi_{new}, u_{new})$. Denote the final state of $\phi_{new}$ as $x_{new}$, as is shown in Figure \ref{fig:searchgraph_statespace}. If $\psi_{new}$ does not intersect with $X_{u}$, add a vertex $v_{new}$ associated with $x_{new}$ to $V$ and an edge $(v_{cur}, v_{new})$ associated with $\psi_{new}$ to $E$. Then, go to \textbf{Step 2}.
	\end{steps}

\ifbool{conf}{}{
	\subsection{Input library}
\label{section:inputlibrary}
HyRRT requires a library of inputs to simulate solution pairs. Note that inputs are constrained by the flow set $C$ and jump set $D$, respectively. Given the flow set $C$ and the jump set $D$ of the hybrid system $\mathcal{H}$, the input signal set during flows, denoted  $\mathcal{U}_{C}$, and the input values at jumps, denoted $\mathcal{U}_{D}$, are described as follows. 
\begin{enumerate}
	\item The input signal applied during flows is a continuous-time signal. Therefore, the input signal during flows, denoted $\tilde{u}$, is specified by a function from an interval of time of the form $[0, t^{*}]$ to $U_{C}$ for some $t^{*}\in \mathbb{R}_{\geq 0}$, namely,
	$$
	\tilde{u}: [0, t^{*}]\to U_{C}.
	$$
	Definition \ref{definition:solution} also requires that $\tilde{u}$ is Lebesgue measurable and locally bounded.
	Then, the input signal set during flows, denoted $\mathcal{U}_{C}$, is a pre-determined set that collects possible $\tilde{u}$ that can be applied during flows. 
	
	Given $\tilde{u}\in\mathcal{U}_{C}$, the functional $\overline{t}: \mathcal{U}_{C} \to [0, \infty)$ returns the time duration of $\tilde{u}$. Namely, given any $\tilde{u}: [0, t^{*}]\to U_{C}$, $\overline{t}(\tilde{u}) = t^{*}$.
	\item The input applied at a single jump can be specified by an input value in $U_{D}$. The input set at jumps, denoted $\mathcal{U}_{D}$, is a pre-determined set that collects possible values of inputs that can be applied at jumps, namely, 
	$$
		\mathcal{U}_{D}\subset U_{D}.
	$$
\end{enumerate}
The pair of sets $(\mathcal{U}_{C}, \mathcal{U}_{D})$ defines the input library, denoted $\mathcal{U}$, namely,
$$
	\mathcal{U} := (\mathcal{U}_{C}, \mathcal{U}_{D}).
$$

\label{section:samplingmotionprimitives}
The following introduces a method to construct $\mathcal{U}_{C}$ and $\mathcal{U}_{D}$, given sets $C$ and $D$ respectively. The  procedure to construct $\mathcal{U}_{C}$ is given as follows.
\begin{steps}
	\item Set $t^{*}$ to a positive constant. Choose a finite number of points from $U_C$ and denote it $U^s_C$.
	\item For each point $u^{s}\in U_{C}^{s}$, construct an input signal $[0, t^{*}]\to \{u^{s}\}$ and add it to $\mathcal{U}_{C}$. 
\end{steps}
Set $\mathcal{U}_{D}$ can be constructed  similarly by choosing a finite number of points from $U_{D}$. 

%
	\subsection{Simulator of continuous dynamics}
\label{section:flowsystemsimulator}
HyRRT requires a simulator to compute the solution pair starting from a given initial state $x_{0}\in C'$ with a given input signal $\tilde{u}\in \mathcal{U}_{C}$ applied, following continuous dynamics. The initial state $x_{0}$, the flow set $C$, and the flow map $f$ are used in the simulator. 

Note that when the simulated solution enters the intersection between the flow set $C$ and the jump set $D$, it can either keep flowing or stop to jump. In \cite{sanfelice2013toolbox}, the hybrid system simulator HyEQ uses a scalar priority option flag $rule$ to show whether the simulator gives priority to jumps ($rule= 1$), priority to flows ($rule= 2$), or no priority ($rule= 3$) when both $x\in C$ and $x\in D$ hold. When no priority is selected, then the simulator randomly selects to flow or jump.
	
	Therefore, the jump set $D$ and a priority option $rule\in \{flow, jump\}$  are input to the simulator. If $rule = flow$,  the simulation keeps flowing at $C\cap D$. If $rule = jump$, the simulation is terminated when entering $C\cap D$. In this paper, the option of random selection is not considered.  Therefore, the proposed simulator should be able to solve the following problem.
\begin{problem}
	\label{problem:flowsimulator}
	Given the flow set $C$, the flow map $f$, and the jump set $D$ of a hybrid system $\mathcal{H}$ with input $u\in \mathbb{R}^{m}$ and state $x\in \mathbb{R}^{n}$, a priority option flag $rule\in \{flow, jump\}$, an initial state $x_{0}\in \mathbb{R}^{n}$, and an input signal $\tilde{u}\in \mathcal{U}_{C}$ such that $(x_{0}, \tilde{u}(0))\in C$, find a pair $(\phi, u):  [0, t^{*}]\times \{0\} \to \mathbb{R}^{n} \times \mathbb{R}^{m}$, where $t^{*}\in [0, \overline{t}(\tilde{u})]$, such that the following hold:
	\begin{enumerate}
		\item $\phi(0, 0) = x_{0}$;
		\item For all $t\in [0, t^{*}]$, $u(t, 0) = \tilde{u}(t)$;
		\item If $[0, t^{*}]$ has nonempty interior,
		\begin{enumerate}
			\item the function $t\mapsto \phi(t, 0)$ is locally absolutely continuous,
			\item for all $t\in (0, t^{*})$,
				\begin{equation}
					\begin{aligned}
					(\phi(t, 0),u(t, 0))&\in C & \text{ if }rule = flow\\
					(\phi(t, 0),u(t, 0))&\in C\backslash D &\text{ if }rule = jump,\\
					\end{aligned}
				\end{equation}
			\item for almost all $t\in [0, t^{*}]$,
			\begin{equation}
			\label{equation:differentialequation}
			\dot{\phi}(t,0) = f(\phi(t,0), u(t,0)).
			\end{equation}
		\end{enumerate}
	\end{enumerate}
\end{problem}
\begin{remark}
	The solution to Problem \ref{problem:flowsimulator} is not unique. If there exists a solution $(\phi, u)$ to Problem \ref{problem:flowsimulator}, then for any hybrid time $(t_{1}, j_{1})\in \dom (\phi, u)$, the truncation of $(\phi, u)$ between $(0, 0)$ and  $(t_{1}, j_{1})$ is also a solution to Problem \ref{problem:flowsimulator}. The definition of truncation operation is provided in Definition \ref{definition: truncation}.
\end{remark}
In this paper, the simulator is designed to simulate the maximal solution to Problem \ref{problem:flowsimulator}. The definition of maximal solution is given as follows; see \cite{goebel2009hybrid}. 
\begin{definition}
	(Maximal solution) A solution $\psi$ to Problem \ref{problem:flowsimulator} is said to be maximal if there does not exist another solution $\psi'$ to Problem \ref{problem:flowsimulator} such that $\dom \psi$ is a proper subset of $\dom \psi'$ and $\psi(t, 0) = \psi'(t, 0)$ for all $t\in \dom_{t} \psi$.
\end{definition}

The module to simulate the maximal solution to Problem \ref{problem:flowsimulator} is called the simulator of continuous dynamics. As is shown in Figure \ref{fig:flowsimulator}, the input of this module is flow set $C$, flow map $f$, jump set $D$, priority option flag $rule$, initial state $x_{0}$, and input signal $\tilde{u}$. The output of this module is the maximal solution $(\phi, u)$ to Problem \ref{problem:flowsimulator} given the data above. Therefore, this module can be denoted as
\begin{equation}
\label{model:flowsystemsimulator}
(\phi, u) = continuous\_simulator(C, f, D, rule,  x_{0}, \tilde{u}).
\end{equation}

\begin{figure}[htbp] 
	\centering
	\def\svgwidth{0.8\columnwidth}
\begingroup%
  \makeatletter%
  \providecommand\color[2][]{%
    \errmessage{(Inkscape) Color is used for the text in Inkscape, but the package 'color.sty' is not loaded}%
    \renewcommand\color[2][]{}%
  }%
  \providecommand\transparent[1]{%
    \errmessage{(Inkscape) Transparency is used (non-zero) for the text in Inkscape, but the package 'transparent.sty' is not loaded}%
    \renewcommand\transparent[1]{}%
  }%
  \providecommand\rotatebox[2]{#2}%
  \newcommand*\fsize{\dimexpr\f@size pt\relax}%
  \newcommand*\lineheight[1]{\fontsize{\fsize}{#1\fsize}\selectfont}%
  \ifx\svgwidth\undefined%
    \setlength{\unitlength}{405.05496994bp}%
    \ifx\svgscale\undefined%
      \relax%
    \else%
      \setlength{\unitlength}{\unitlength * \real{\svgscale}}%
    \fi%
  \else%
    \setlength{\unitlength}{\svgwidth}%
  \fi%
  \global\let\svgwidth\undefined%
  \global\let\svgscale\undefined%
  \makeatother%
  \begin{picture}(1,0.22842442)%
    \lineheight{1}%
    \setlength\tabcolsep{0pt}%
    \put(0,0){\includegraphics[width=\unitlength,page=1]{flowsimulator.pdf}}%
    \put(0.08981929,0.18041256){\makebox(0,0)[lt]{\lineheight{1.25}\smash{\begin{tabular}[t]{l}$(C, f)$\end{tabular}}}}%
    \put(0.14883044,0.05395965){\makebox(0,0)[lt]{\lineheight{1.25}\smash{\begin{tabular}[t]{l}$x_{0}$\end{tabular}}}}%
    \put(0.14224217,0.00567052){\makebox(0,0)[lt]{\lineheight{1.25}\smash{\begin{tabular}[t]{l}$\tilde{u}$\end{tabular}}}}%
    \put(0.42866301,0.11832512){\makebox(0,0)[lt]{\lineheight{1.25}\smash{\begin{tabular}[t]{l}flow system\\simulator\end{tabular}}}}%
    \put(0.83862157,0.10008123){\makebox(0,0)[lt]{\lineheight{1.25}\smash{\begin{tabular}[t]{l}$(\phi, u)$\end{tabular}}}}%
    \put(0,0){\includegraphics[width=\unitlength,page=2]{flowsimulator.pdf}}%
    \put(0.12419364,0.10057964){\makebox(0,0)[lt]{\lineheight{1.25}\smash{\begin{tabular}[t]{l}$rule$\end{tabular}}}}%
    \put(0,0){\includegraphics[width=\unitlength,page=3]{flowsimulator.pdf}}%
    \put(0.14059433,0.14044955){\makebox(0,0)[lt]{\lineheight{1.25}\smash{\begin{tabular}[t]{l}$D$\end{tabular}}}}%
  \end{picture}%
\endgroup%

{6}
	\caption{The simulator module of continuous dynamics. }\label{fig:flowsimulator}   
\end{figure} 

The simulator module of continuous dynamics performs the following steps. 
\begin{steps}
	\item Construct a helper function $\hat{\phi}: [0, \overline{t}(\tilde{u})]\to \mathbb{R}^{n}$ by
	\begin{equation}
	\label{equation:flowequationintegration}
			\hat{\phi}(t) = x_{0} + \int_{0}^{t}f(\hat{\phi}(\tau),\tilde{u}(\tau))d\tau \quad t\in [0, \overline{t}(\tilde{u})].
	\end{equation} 
	\item Determine the time $\hat{t}\in [0, \overline{t}(\tilde{u})]$ by
	\begin{equation}
	\label{equation:integrationduration}
		\hat{t} = \left\{\begin{aligned}
			\begin{aligned}
			\max \{t&\in [0, \overline{t}(\tilde{u})]: \forall t'\in (0, t), (\hat{\phi}(t'), \tilde{u}(t'))\\ &\in C\}  \qquad \text{ if } rule = flow\\
			\max \{t&\in [0, \overline{t}(\tilde{u})]: \forall t'\in (0, t), (\hat{\phi}(t'), \tilde{u}(t'))\\ &\in C\backslash D\}  \qquad\text{ if } rule = jump.\\
			\end{aligned}
		\end{aligned}
		\right.
	\end{equation}
	\item Construct the solution function pair $(\phi, u): [0, \hat{t}]\times \{0\}\to \mathbb{R}^{n}\times \mathbb{R}^{m}$ by 
	\begin{equation}
	\label{equation:flowsegment}
		\begin{aligned}
			\phi(t, 0) &= \hat{\phi}(t)\quad t\in [0, \hat{t}]\\
			u(t, 0) &= \tilde{u}(t)\quad t\in [0, \hat{t}].\\
		\end{aligned}
	\end{equation}
\end{steps}
The solution function pair $(\phi, u)$ is the maximal solution to Problem \ref{problem:flowsimulator} and, hence, the output of module $continuous\_simulator$. Note that the challenge in implementing Step 1 is the integration in (\ref{equation:flowequationintegration}). It can be approximated by employing numerical integration methods. The challenge in implementing Step 2 is determining $\hat{t}$ in (\ref{equation:integrationduration}) during the integration. It can be approximated by employing zero-crossing detection algorithms. Next, we discuss how numerical integration scheme and zero-crossing detection algorithm are modeled, following a computational approach to implement the simulator of continuous dynamics.

\subsubsection{Numerical integration scheme model}
For most systems, the closed-form expression of the integration in (\ref{equation:flowequationintegration}) is not available and the integration has to be computed numerically.
The numerical integration scheme can be modeled as
\begin{equation}
\label{equation:numericalintegrationimplicit}
	F_{s}(x, x^{+}, \tilde{u}, f, t) = 0
\end{equation} where $s$ denotes the step size,  $x$ denotes the state at current step, $x^{+}$ denotes the approximation of the state at the next step, $\tilde{u}$ denotes the applied input signal, $f$ denotes the flow map, and $t$ denotes the time at current step. $x^{+}$ can be obtained by solving the equation in (\ref{equation:numericalintegrationimplicit}), given $x$, $\tilde{u}$, $f$ and $t$. Two examples on how (\ref{equation:numericalintegrationimplicit}) models explicit and implicit numerical integration scheme are given as follows.
\begin{example}
	(Forward Euler method) Given the differential equation
	\begin{equation}
	\label{equation:governequationexp}
		\dot{x} = f(x, u)
	\end{equation} with initial state $x_{0}$ and input signal $\tilde{u}\in \mathcal{U}_{C}$,
	the integration scheme using forward Euler method is
	\begin{equation}
	\label{equation:forwardeulerintegrationscheme}
		x^{+} = x + sf(x, \tilde{u}(t))		
	\end{equation}
	where $s$ denotes the step size,  $x$ denotes the state at current step, $x^{+}$ denotes the approximation of state at next step and $t$ denotes the time at current step. Following (\ref{equation:numericalintegrationimplicit}), (\ref{equation:forwardeulerintegrationscheme}) can be modeled as 
	\begin{equation}
		F_{s}(x, x^{+}, \tilde{u}, f, t) := x + sf(x, \tilde{u}(t)) - x^{+}
	\end{equation}
	and $x^{+}$ can be obtained by solving the equation 
	$$
	F_{s}(x, x^{+}, \tilde{u}, f, t)  = 0.
	$$
\end{example}

\begin{example}
	(Backward Euler method) Given the differential equation
	\begin{equation}
	\label{equation:governequationimp}
	\dot{x} = f(x, u)
	\end{equation} with initial state $x_{0}$  and input signal $\tilde{u}\in \mathcal{U}_{C}$,
	the integration scheme using backward Euler method is
	\begin{equation}
	\label{equation:backwardeulerintegrationscheme}
	x^{+} = x + sf(x^{+}, \tilde{u}(t + s))		
	\end{equation}
	where $s$ denotes the step size,  $x$ denotes the state at the current step, $x^{+}$ denotes the approximation of the state at the next step, $\tilde{u}$ denotes the applied input signal, and $t$ denotes the time at current step. Following (\ref{equation:numericalintegrationimplicit}), (\ref{equation:backwardeulerintegrationscheme}) can be modeled as 
	\begin{equation}
	F_{s}(x, x^{+}, \tilde{u}, f, t) := x + sf(x^{+}, \tilde{u}(t + s)) 
	- x^{+}
	\end{equation}
	and $x^{+}$ can be obtained by solving the equation 
	$$
	F_{s}(x, x^{+}, \tilde{u}, f, t)  = 0.
	$$
\end{example}

\subsubsection{Zero-crossing detection model to approximate $\hat{t}$}
The challenge in implementing Step 2 is to determine $\hat{t}$ in (\ref{equation:integrationduration}) during the integration. When the priority option $rule = flow$, $\hat{t}$ is the first time when $(\hat{\phi}, \tilde{u})$ steps outside $C$. When the priority option $rule = jump$, $\hat{t}$ is the first time when $(\hat{\phi}, \tilde{u})$ steps outside $C\backslash D$.

When $rule = flow$, if there exists a zero-crossing function $h_{f}: \mathbb{R}^{n}\times \mathbb{R}^{m}\to \mathbb{R}$ for the given flow set $C$ such that 
	\begin{enumerate}
		\item $h_{f}(x, u) > 0$ for all $(x, u)\in \interior C$,
		\item $h_{f}(x, u)< 0$ for all $(x, u)\in \interior (\mathbb{R}^{n}\times \mathbb{R}^{m}\backslash C)$,
		\item $h_{f}(x, u) = 0$ for all $(x, u)\in \partial  C$,
		\item $h_{f}$ is continuous over $\mathbb{R}^{n}\times \mathbb{R}^{m}$,
\end{enumerate}
then $\hat{t}$ is the first time when $h_{f}(\hat{\phi}, \tilde{u})$ crosses zeroes and can be approximated by zero-crossing detection algorithm. 

When $rule = jump$, if there exists a zero-crossing function $h_{g}: \mathbb{R}^{n}\times \mathbb{R}^{m}\to \mathbb{R}$ for the given flow set $C$ and the jump set $D$ such that 
	\begin{enumerate}
			\item $h_{g}(x, u) > 0$ for all $(x, u)\in \interior (C\backslash D)$,
		\item $h_{g}(x, u)< 0$ for all $(x, u)\in \interior (\mathbb{R}^{n}\times \mathbb{R}^{m}\backslash (C\backslash D))$,
		\item $h_{g}(x, u) = 0$ for all $(x, u)\in \partial  (C\backslash D)$,
		\item $h_{g}$ is continuous over $\mathbb{R}^{n}\times \mathbb{R}^{m}$,
	\end{enumerate}
	then $\hat{t}$ is the first time when $h_{g}(\hat{\phi}, \tilde{u})$ crosses zeroes and can be approximated by zero-crossing detection algorithm.
\begin{example}
	(Bouncing ball system in Example \ref{example:bouncingball}, revisited) In the bouncing ball system, the flow set is $C = \{(x, u)\in \mathbb{R}^{2}\times \mathbb{R}: x_{1}\geq 0\}$ and the jump set is $D = \{(x, u) \in \mathbb{R}^{2}\times \mathbb{R}: x_{1} = 0, x_{2} \leq 0, u\geq 0\}$. Then $h_{f}$ can be designed as:
	\begin{equation}
		h_{f}(x, u): = x_{1}\quad (x, u)\in \mathbb{R}^{n}\times \mathbb{R}^{m}.
	\end{equation} 
	Note that $D\subset \partial C$. Therefore, $\interior (C\backslash D) = \interior C$ and $\interior (\mathbb{R}^{n}\times \mathbb{R}^{m}\backslash (C\backslash D)) = \interior (\mathbb{R}^{n}\times \mathbb{R}^{m}\backslash C)$. Therefore, $h_{g}$ can be designed in the same way as $h_{f}$:
	\begin{equation}
		h_{g}(x, u): = x_{1}\quad (x, u)\in \mathbb{R}^{n}\times \mathbb{R}^{m}.
	\end{equation}
\end{example}

The zero-crossing detection algorithm that approximates $\hat{t}$ can be modeled as 
\begin{equation}
	\label{equation:zerocrossing}
	\hat{t} = t_{zcd} (x, x^{+}, \tilde{u}, t, rule, C, D)
\end{equation}
where $x$ and $x^{+}$ denote states at consecutive steps obtained by numerical integration, $\tilde{u}$ denotes the applied input signal, $t$ denotes the time at the current step, $rule$ denotes the priority option and $C$ and $D$ denotes the flow set and the jump set of the given hybrid system, respectively. The zero-crossing functions $h_{f}$ and $h_{g}$ are constructed based on sets $C$ and $D$. $\hat{t}$ is set as $-1$ when no zero-crossing is detected.

\subsubsection{A computational framework to approximate the simulator of continuous dynamics}
An overall computational approach to approximate the output of the simulator of continuous dynamics is introduced in this part. Notations used in the approach are provided as follows.
\begin{enumerate}
	\item The inputs to the simulator of continuous dynamics are the flow set $C$, flow map $f$ and jump set $D$ of the hybrid system $\mathcal{H}$, a priority option flag $rule\in \{flow, jump\}$, $x_{0}\in \overline{C'}$ and $\tilde{u}\in \mathcal{U}_{C}$ such that $(x_{0}, \tilde{u}(0))\in \overline{C}$.
	\item The integration scheme $F_{s}$ (modeled in (\ref{equation:numericalintegrationimplicit})) with its step size $s$ and the zero-crossing detection algorithm $t_{zcd}$ (modeled in (\ref{equation:zerocrossing})) with zero-crossing functions $h_{f}$ and $h_{g}$ are given. Zero-crossing functions $h_{f}$ and $h_{g}$ are constructed given the flow set $C$ and jump set $D$ in the inputs.
	\item The approximation of $\phi$ at $k$th step is denoted as $\phi_{s}(k, 0)$, where $s$ denotes the step size. The value of $u$ at $k$th step is denoted as $u_{s}(k, 0)$, where $s$ denotes the step size.
\end{enumerate} 
Then the computational approach to approximate the output of simulator of continuous dynamics is given in Algorithm \ref{algo:flowsystemsimulator}.
\begin{algorithm}[H]
	\caption{A computational approach to approximate the module $continuous\_simulator$}
	\begin{algorithmic}[1]
		\Function{$continuous\_simulator\_approximate$}{$C, f, D, rule,  x_{0}, \tilde{u}$}
		\State Set $t = 0$, $k = 0$; Set $\phi_{s}(0, 0) = x_{0}$, $u_{s}(0, 0) = \tilde{u}(0)$. 
		\While{$t\leq\overline{t}(\tilde{u})$}
			\State Compute $\phi_{s}(k + 1, 0)$ by solving $F_{s}(\phi_{s}(k, 0), \phi_{s}(k + 1, 0), \tilde{u}, f, t) = 0$.
			\State Set $u_{s}(k + 1, 0) = \tilde{u}(t + s)$.
			\State Set $\hat{t} = t_{zcd} (\phi_{s}(k, 0), \phi_{s}(k + 1, 0), \tilde{u}, t, rule, C, D)$.
			\If{$\hat{t}\neq-1$}
				\State Compute $\phi_{s}(k + 1, 0)$ by solving $F_{\hat{t} - t}(\phi_{s}(k, 0), \phi_{s}(k + 1, 0), \tilde{u}, f, t) = 0$.
				\State Set $u_{s}(k + 1, 0) = \tilde{u}(\hat{t}, 0)$.
				\State $t = \hat{t}$; $k = k + 1$.
				\State break.
			\EndIf
			\State $t = t + s$; $k = k + 1$.
		\EndWhile
		\State \Return $\phi_{s}$, $u_{s}$, $t$, $k$.
		\EndFunction
	\end{algorithmic}
\label{algo:flowsystemsimulator}
\end{algorithm}
The return $(\phi_{s}, u_{s}, t, k)$ of $continuous\_simulator\_approximate(C, f, D, rule,  x_{0}, \tilde{u})$ is a pointwise approximation of the output $(\phi, u)$ of $continuous\_simulator(C, f, D, rule,  x_{0}, \tilde{u})$. For $i = 0, 1, ... k - 1$, $\phi(i\times s, 0) \approx \phi_{s}(i, 0)$ and $u(i\times s, 0) = u_{s}(i, 0)$. Besides, the approximation of the last state and input are $\phi(t, 0) \approx \phi_{s}(k, 0)$ and $u(t, 0) = u_{s}(k, 0)$.
	\subsection{Simulator of discrete dynamics}\label{section:jumpsystemsimulator}
HyRRT algorithm requires a simulator to compute a purely discrete solution pair with a single jump, starting from an initial state $x_{0}\in D'$ with an input $u_{D}\in \mathcal{U}_{D}$ applied. Namely, the system simulator should be able to solve the following problem:
　\begin{problem}
	\label{problem:jumpsimulator}
	Given the jump set $D$ and jump map $g$ of hybrid system $\mathcal{H}$ with input $u\in \mathbb{R}^{m}$, state $x\in \mathbb{R}^{n}$, an initial state $x_{0}\in \mathbb{R}^{n}$ and an input value $u_{D}\in \mathcal{U}_{D}$ such that $(x_{0}, u_{D})\in D$, find a pair $(\phi, u):  \{0\}\times \{0, 1\} \to \mathbb{R}^{n} \times \mathbb{R}^{m}$ such that the following hold:
　	\begin{enumerate}
		\item $\phi(0, 0) = x_{0}$;
　		\item $u(0, 0) = u_{D}$;
　		\item $\phi(0, 1) = g(\phi(0, 0), u(0, 0))$.
	\end{enumerate}
\end{problem}
Problem \ref{problem:jumpsimulator} can be solved by constructing a function pair $(\phi, u):  \{0\}\times \{0, 1\} \to \mathbb{R}^{n} \times \mathbb{R}^{m}$ such that
\begin{equation}
\label{equation:jumpsegment}
\begin{aligned}
\phi(0, 0) &= x_{0}\\
\phi(0, 1) &= g( x_{0}, u_{D})\\
u(0, 0) & = u_{D}\\
u(0, 1)  & \in \mathbb{R}^{m}.
\end{aligned}
\end{equation}
In (\ref{equation:jumpsegment}), $u(0, 1)  \in  \mathbb{R}^{m}$ can be implemented by selecting a point in $ \mathbb{R}^{m}$ and assigning the point to $u(0, 1)$. The selection can be made by choosing a fixed value or simple random sampling in $\mathbb{R}^{m}$. The selection should avoid the input values such that $(\phi(0, 1), u(0, 1))$ falls into unsafe set. 

The function pair in (\ref{equation:jumpsegment}) can be constructed by a module called simulator of discrete dynamics. As is shown in Figure \ref{fig:jumpsimulator}, the inputs of this module are the jump set $D$, the jump map $g$, the initial state $x_{0}\in D'$, and the input $u_{D}\in \mathcal{U}_{D}$ such that $(x_{0}, u_{D})\in D$. The output of this module is the solution pair $(\phi, u)$ constructed in (\ref{equation:jumpsegment}). Therefore, the module can be denoted as
\begin{equation}
\label{model:jumpsystemsimulator}
(\phi, u) = discrete\_simulator(D, g, x_{0}, u_{D}).
\end{equation}

\begin{figure}[htbp] 
	\centering
	\def\svgwidth{1.2\columnwidth}
\begingroup%
  \makeatletter%
  \providecommand\color[2][]{%
    \errmessage{(Inkscape) Color is used for the text in Inkscape, but the package 'color.sty' is not loaded}%
    \renewcommand\color[2][]{}%
  }%
  \providecommand\transparent[1]{%
    \errmessage{(Inkscape) Transparency is used (non-zero) for the text in Inkscape, but the package 'transparent.sty' is not loaded}%
    \renewcommand\transparent[1]{}%
  }%
  \providecommand\rotatebox[2]{#2}%
  \newcommand*\fsize{\dimexpr\f@size pt\relax}%
  \newcommand*\lineheight[1]{\fontsize{\fsize}{#1\fsize}\selectfont}%
  \ifx\svgwidth\undefined%
    \setlength{\unitlength}{409.76148425bp}%
    \ifx\svgscale\undefined%
      \relax%
    \else%
      \setlength{\unitlength}{\unitlength * \real{\svgscale}}%
    \fi%
  \else%
    \setlength{\unitlength}{\svgwidth}%
  \fi%
  \global\let\svgwidth\undefined%
  \global\let\svgscale\undefined%
  \makeatother%
  \begin{picture}(1,0.16889752)%
    \lineheight{1}%
    \setlength\tabcolsep{0pt}%
    \put(0,0){\includegraphics[width=\unitlength,page=1]{jumpsimulator.pdf}}%
    \put(0.15,0.1245006){\makebox(0,0)[lt]{\lineheight{1.25}\smash{\begin{tabular}[t]{l}$(D, g)$\end{tabular}}}}%
    \put(0.2,0.0738031){\makebox(0,0)[lt]{\lineheight{1.25}\smash{\begin{tabular}[t]{l}$x_{0}$\end{tabular}}}}%
    \put(0.2,0.02606862){\makebox(0,0)[lt]{\lineheight{1.25}\smash{\begin{tabular}[t]{l}$u_{D}$\end{tabular}}}}%
    \put(0.46,0.08934816){\makebox(0,0)[lt]{\lineheight{1.25}\smash{\begin{tabular}[t]{l}jump system\\simulator\end{tabular}}}}%
    \put(0.84047516,0.07480128){\makebox(0,0)[lt]{\lineheight{1.25}\smash{\begin{tabular}[t]{l}$(\phi, u)$\end{tabular}}}}%
  \end{picture}%
\endgroup%

 {6}
	\caption{Simulator of discrete dynamics. }\label{fig:jumpsimulator} 
\end{figure} 
}
	\ifbool{conf}{\vspace{-0.3cm}}{}
\subsection{HyRRT Algorithm}
\label{section:hybridRRTframewrok}
Following the overview in Section \ref{section:algorithmoverview}, the proposed algorithm is given in Algorithm \ref{algo:hybridRRT}. The inputs of Algorithm \ref{algo:hybridRRT} are the problem $\mathcal{P} = (X_{0}, X_{f}, X_{u}, (C, f, D, g))$, the input library  $ (\mathcal{U}_{C}, \mathcal{U}_{D})$, a parameter $p_{n}\in (0, 1)$, which tunes the probability of proceeding with the flow regime or the jump regime, an upper bound $K\in \mathbb{N}_{>0}$ for the number of iterations to execute, and two tunable sets $X_{c}\supset \overline{C'}$ and $X_{d}\supset D'$, which act as constraints in finding a closest vertex to $x_{rand}$. 
Each function in Algorithm \ref{algo:hybridRRT} is defined next.
\subsubsection{$\mathcal{T}.init(X_{0})$} 
	The function call $\mathcal{T}.init$ is used to initialize a search tree $\mathcal{T} = (V, E)$.  It randomly selects a finite number of points from $X_{0}$. For each sampling point $x_{0}$, a vertex $v_{0}$ associated with $x_{0}$ is added to $V$. At this step, no edge is added to $E$.
	
	\subsubsection{$x_{rand}$$\leftarrow$$random\_state(S)$} 
	The function call $random\_state$ randomly selects a point from the set $S\subset \mathbb{R}^{n}$. It is designed to select from $\overline{C'}$ and $D'$ separately depending on the value of $r$ rather than to select from $\overline{C'}\cup D'$. The reason is that if $\overline{C'}$ ($D'$) has zero measure while $D'$ ($\overline{C'}$) does not, the probability that the point selected from $\overline{C'}\cup D'$ lies in $\overline{C'}$ ($D'$, respectively) is zero, which would prevent establishing probabilistic completeness.
	
	\subsubsection{$v_{cur}$$\leftarrow$$ nearest\_neighbor$$(x_{rand}, $$\mathcal{T}, $$\mathcal{H}, flag)$} 
	The function call $nearest\_neighbor$ searches for a vertex $v_{cur}$ in the search tree $\mathcal{T} = (V, E)$ such that its associated state value has minimal distance to $x_{rand}$. This function is implemented as follows.
	\begin{itemize}
		\item When $flag = flow$, the following optimization problem is solved over $X_{c}$.
		\begin{problem}
			\label{problem:nearestneighborflow}
			Given a hybrid system $\mathcal{H} = (C, f, D, g)$, $x_{rand}\in \overline{C'}$, and a search tree $\mathcal{T} = (V, E)$, solve
			$$
			\begin{aligned}
			\argmin_{v\in V}& \quad |\overline{x}_{v} -  x_{rand}|\\
			\textrm{s.t.}& \quad\overline{x}_{v} \in X_{c}.
			\end{aligned}
			$$
		\end{problem}
		\item When $flag = jump$, the following optimization problem is solved over $X_{d}$.
		\begin{problem}
			\label{problem:nearestneighborjump}
			Given a hybrid system $\mathcal{H} = (C, f, D, g)$, $x_{rand}\in D'$, and a search tree $\mathcal{T} = (V, E)$, solve
			$$
			\begin{aligned}
			\argmin_{v\in V}& \quad |\overline{x}_{v} - x_{rand}|\\
			\textrm{s.t.}& \quad \overline{x}_{v} \in X_{d}.
			\end{aligned}
			$$
		\end{problem}
	\end{itemize}
The data of Problem \ref{problem:nearestneighborflow} and Problem \ref{problem:nearestneighborjump} comes from the arguments of the $nearest\_neighbor$ function call. This optimization problem can be solved by traversing all the vertices in $\mathcal{T} = (V, E)$.

	\subsubsection{$return\leftarrow new\_state(x_{rand}, v_{cur}, (\mathcal{U}_{C}, \mathcal{U}_{D}) , \mathcal{H} , X_{u},$$\\x_{new},\psi_{new})$} 
	\label{section:newstate}
	If $\overline{x}_{v_{cur}}\in \overline{C'}\backslash D'$ ($\overline{x}_{v_{cur}}$$\in$$D'\backslash \overline{C'}$), the function call $new\_state$ generates a new solution pair $\psi_{new}$ to hybrid system $\mathcal{H}$ starting from $\overline{x}_{v_{cur}}$ by applying a input signal $\tilde{u}$ (an input value $u_{D}$) randomly selected from $\mathcal{U}_{C}$ ($\mathcal{U}_{D}$, respectively). 
	If $\overline{x}_{v_{cur}}$$\in\overline{C'}\cap D'$, then this function generates $\psi_{new}$ by randomly selecting flows or jump. The final state of $\psi_{new}$ is denoted as $x_{new}$. \ifbool{conf}{}{

If $\overline{x}_{v_{cur}}\in \overline{C'}\backslash D'$, the function call $new\_state$ randomly selects an input signal $\tilde{u}$ from $\mathcal{U}_{C}$ and feeds it to the $continuous\_simulator$. Therefore, 
	\begin{equation}
	\label{newstate:flow}
	\begin{aligned}
	\psi_{new} &= (\phi_{new}, u_{new}) \\
	&= continuous\_simulator(C, f, D, flow,  \overline{x}_{v_{cur}}, \tilde{u})
	\end{aligned}
	\end{equation}
	
	If $\overline{x}_{v_{cur}}$$\in$$D'\backslash \overline{C'}$, the function call $new\_state$ randomly selects an input value $u_{D}$ from $\mathcal{U}_{D}$ and feeds it to the $discrete\_simulator$. Therefore,
	\begin{equation}
	\label{newstate:jump}
	\begin{aligned}
	\psi_{new} &= (\phi_{new}, u_{new})\\
	&= discrete\_simulator(D, g, \overline{x}_{v_{cur}}, u_{D}),
	\end{aligned}
	\end{equation}

	If $\overline{x}_{v_{cur}}$$\in\overline{C'}\cap D'$, at first, the function call $new\_state$ randomly selects a real number $r_{fg}$ from the interval $[0, 1]$ and compares $r_{fg}$ with a tunable parameter $p_{fg}\in (0, 1)$. If $r_{fg} \leq p_{fg}$, then the function call $new\_state$ randomly selects an input signal $\tilde{u}$ from $\mathcal{U}_{C}$ and feeds it to the $continuous\_simulator$. If $r_{fg} > p_{fg}$, then the function call $new\_state$ randomly selects an input value $u_{D}$ from $\mathcal{U}_{D}$ and feeds it to the $discrete\_simulator$.
}

Note that the choices of inputs are random. Some RRT variants choose the optimal input that drives $x_{new}$ closest to $x_{rand}$. However, \cite{kunz2015kinodynamic} proves that such a choice makes the RRT algorithm probabilistically incomplete. After $\psi_{new}$ and $x_{new}$ are generated, the function $new\_state$ checks if there exists $(t, j)\in \dom \psi_{new}$ such that $\psi_{new}(t, j)\in X_{u}$. If so, then $\psi_{new}$ intersects with the unsafe set and $new\_state$ returns $false$. Otherwise, this function returns $true$. 

	\subsubsection{$v_{new}$$\leftarrow$$\mathcal{T}.add\_vertex(x_{new})$ and  $\mathcal{T}.add\_edge$$(v_{cur}$, $v_{new}$, $\psi_{new})$} 
	The function call $\mathcal{T}.add\_vertex(x_{new})$ adds a new vertex $v_{new}$ associated with $x_{new}$ to $\mathcal{T}$ and returns $v_{new}$. The function call $\mathcal{T}.add\_edge(v_{cur}, v_{new}, \psi_{new})$ adds a new edge $e_{new} = (v_{cur}, v_{new})$ associated with $\psi_{new}$  to $\mathcal{T}$. 
	\subsection{Solution Checking during HyRRT Construction}
	\label{section:checksolution}
	When the function call $extend$ returns $Reached$ or $Advanced$,  a solution checking function is employed to check if a path in $\mathcal{T}$ can be used to construct a motion plan to the given motion planning problem. If this function finds a path
	$
	p = ((v_{0}, v_{1}), (v_{1}, v_{2}), ..., (v_{n - 1}, v_{n})) = : (e_{0}, e_{1}, ..., e_{n - 1})
	$
	in $\mathcal{T}$ such that 
	\ifbool{conf}{ 1) $\overline{x}_{v_{0}} \in X_{0}$ and  2) $\overline{x}_{v_{n}} \in X_{f}$,}{
	\begin{enumerate}[label=\arabic*)]
		\item $\overline{x}_{v_{0}} \in X_{0}$,
		\item $\overline{x}_{v_{n}} \in X_{f}$,
		\item for each pair of adjacent edges $e_{i}$ and $e_{i + 1}$, where $i = 0, 1, ..., n - 2$, if $\overline{\psi}_{e_{i}}$ and $\overline{\psi}_{e_{i + 1}}$ are both purely continuous, then $\overline{\psi}_{e_{i + 1}}(0, 0)\in C$,
\end{enumerate}}
	then the solution pair $\tilde{\psi}_{p}$ is a motion plan to the given motion planning problem.
	\ifbool{conf}{
	In practice, item 2) is too restrictive. Given $\epsilon > 0$ representing the tolerance with this condition, we implement item 2) as
	$
	\text{dist}(\overline{x}_{v_{n}}, X_{f}) \leq \epsilon.
	$
}{While item 2) above requires that $\overline{x}_{v_{n}}$ belongs to $X_{f}$, in practice, given $\epsilon > 0$ representing the tolerance with this condition, we implement item 2) as
		\begin{equation}
		\label{equation:tolerance}
		\text{dist}(\overline{x}_{v_{n}}, X_{f}) \leq \epsilon.
		\end{equation}}
\begin{algorithm}[htbp]
	\caption{HyRRT algorithm}
	\label{algo:hybridRRT}
	\hspace*{\algorithmicindent} \textbf{Input: $X_{0}, X_{f}, X_{u}, \mathcal{H} = (C, f, D, g), (\mathcal{U}_{C}, \mathcal{U}_{D}), p_{n} \in (0, 1)$, $K\in \mathbb{N}_{>0}$}
	\begin{algorithmic}[1]
		\State $\mathcal{T}.init(X_{0})$;
		\For{$k = 1$ to $K$}
		\State randomly select a real number $r$ from $[0, 1]$;
		\If{$r\leq p_{n}$}
		\State $x_{rand}\leftarrow random\_state(\overline{C'})$;
		\State $extend(\mathcal{T}, x_{rand}, (\mathcal{U}_{C}, \mathcal{U}_{D}), \mathcal{H}, X_{u}, flow)$;
		\Else
		\State $x_{rand}\leftarrow random\_state(D')$;
		\State $extend(\mathcal{T}, x_{rand}, (\mathcal{U}_{C}, \mathcal{U}_{D}), \mathcal{H}, X_{u}, jump)$;
		\EndIf
		\EndFor
		\State \Return $\mathcal{T}$;
	\end{algorithmic}
	\begin{flushleft}
		$extend(\mathcal{T}, x, (\mathcal{U}_{C}, \mathcal{U}_{D}), \mathcal{H}, X_{u}, flag)$
	\end{flushleft}
	\begin{algorithmic}[1]
		\State $v_{cur}\leftarrow nearest\_neighbor(x, \mathcal{T}, \mathcal{H}, flag)$;
		\If {$new\_state(x, v_{cur}, (\mathcal{U}_{C}, \mathcal{U}_{D}), \mathcal{H}, X_{u}, x_{new}, \psi_{new})$}
		\State $v_{new} \leftarrow \mathcal{T}.add\_vertex(x_{new})$;
		\State $\mathcal{T}.add\_edge(v_{cur}, v_{new}, \psi_{new})$;
		\If{$x_{new} ==x $}
		\State \Return $Reached$;
		\Else
		\State \Return $Advanced$;
		\EndIf
		\EndIf
		\State \Return $Trapped$;
	\end{algorithmic}
\end{algorithm}

	\ifbool{conf}{\vspace{-0.4cm}}{}
\section{Probabilistic Completeness Analysis}\label{section:pc}
This section analyzes the probabilistic completeness property of HyRRT algorithm. Probabilistic completeness means that the probability that the planner fails to return a motion plan, if it exists, approaches zero as the number of samples approaches infinity. 
Section \ref{section:probabilisticcompleteness:subsection:preliminaries} presents the \ifbool{conf}{preliminaries}{premilinary results} to establish the probabilistic completeness. Section \ref{section:inflatedsystem} presents our main result showing that the HyRRT algorithm is probabilistically complete\ifbool{conf}{ under certain assumptions}{ with the clearance assumption relaxed}.
\ifbool{conf}{\subsection{Preliminaries about Probabilistic Completeness}}{\subsection{Probabilistic Completeness with Clearance}}
\label{section:probabilisticcompleteness:subsection:preliminaries}
The following defines the clearance of a motion plan. 
\begin{definition}
	\label{definition:clearance}
	(clearance of a solution pair) Given a motion plan $\psi = (\phi, u)$ to the motion planning problem $\mathcal{P} = (X_{0}, X_{f}, X_{u}, (C, f, D, g))$, the clearance of $\psi = (\phi, u)$ is equal to the maximal $\delta_{clear} > 0$ if the following hold:
	\begin{enumerate}[label=\arabic*)]
		\item For all $(t, j)\in \dom \psi $ such that $I^{j}$ has nonempty interior, $(\phi(t, j) + \delta_{clear}\mathbb{B}, u(t, j) + \delta_{clear}\mathbb{B}) \subset C$;
		\item For all $(t, j)\in \dom \psi $ such that $(t, j + 1)\in \dom \psi$, $(\phi(t, j) + \delta_{clear}\mathbb{B}, u(t, j) + \delta_{clear}\mathbb{B})\subset D$;
		\item  For all $(t, j)\in \dom \psi$, $(\phi(t, j) + \delta_{clear}\mathbb{B}, u(t, j) + \delta_{clear}\mathbb{B}) \cap X_{u} =\emptyset$.
	\end{enumerate} 
\end{definition}
\ifbool{conf}{}{
}

\ifbool{conf}{
}{
Following \cite{kleinbort2018probabilistic}, this paper studies the probabilistic completeness when the input function is piecewise constant in the following sense. 
\begin{definition}
	\label{definition:piecewiseconstantfun}
	(piecewise constant function) A function $\tilde{u}_{c}: [0, T]\to U_{C}$ is a piecewise constant function for probabilistic completeness if there exists $\Delta t \in \mathbb{R}_{>0}$, called resolution, such that 
	\begin{enumerate}[label=\arabic*)]
		\item $\frac{T}{\Delta t}\in \mathbb{N}$, with $\frac{T}{\Delta t}$ denoted $k$;
		\item $t \mapsto \tilde{u}_{c}(t)$ is constant over $[(i - 1)\Delta t, i\Delta t)$ for all $i\in \{1, 2,..., k\}$.
	\end{enumerate}
\end{definition}
\begin{remark}
	Figure \ref{figure:constantfunction} shows an example of the piecewise constant function. 
	In Definition \ref{definition:piecewiseconstantfun}, $\tilde{u}_{c}(T)$ is not required to equal other values of $\tilde{u}_{c}$. This is because a jump may occur and the input may switch to another value at this time. 
	\begin{figure}[htbp]
		\centering
	\def\svgwidth{0.7\columnwidth}
\begingroup%
  \makeatletter%
  \providecommand\color[2][]{%
    \errmessage{(Inkscape) Color is used for the text in Inkscape, but the package 'color.sty' is not loaded}%
    \renewcommand\color[2][]{}%
  }%
  \providecommand\transparent[1]{%
    \errmessage{(Inkscape) Transparency is used (non-zero) for the text in Inkscape, but the package 'transparent.sty' is not loaded}%
    \renewcommand\transparent[1]{}%
  }%
  \providecommand\rotatebox[2]{#2}%
  \newcommand*\fsize{\dimexpr\f@size pt\relax}%
  \newcommand*\lineheight[1]{\fontsize{\fsize}{#1\fsize}\selectfont}%
  \ifx\svgwidth\undefined%
    \setlength{\unitlength}{205.40932288bp}%
    \ifx\svgscale\undefined%
      \relax%
    \else%
      \setlength{\unitlength}{\unitlength * \real{\svgscale}}%
    \fi%
  \else%
    \setlength{\unitlength}{\svgwidth}%
  \fi%
  \global\let\svgwidth\undefined%
  \global\let\svgscale\undefined%
  \makeatother%
  \begin{picture}(1,0.54237366)%
    \lineheight{1}%
    \setlength\tabcolsep{0pt}%
    \put(0,0){\includegraphics[width=\unitlength,page=1]{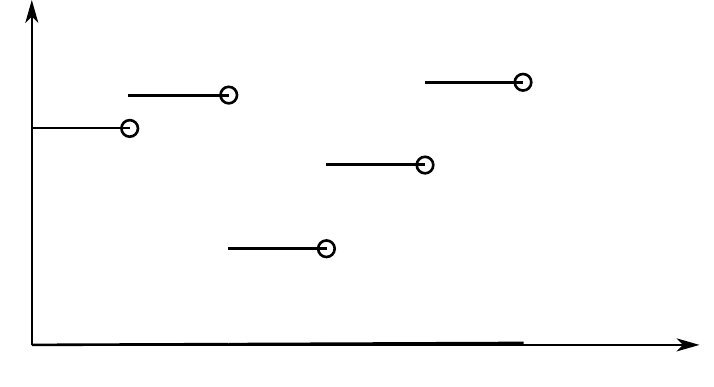}}%
    \put(0.00811071,0.46625139){\makebox(0,0)[lt]{\lineheight{1.25}\smash{\begin{tabular}[t]{l}$u$\end{tabular}}}}%
    \put(0.9066651,0.01384156){\makebox(0,0)[lt]{\lineheight{1.25}\smash{\begin{tabular}[t]{l}$t$\end{tabular}}}}%
    \put(-0.00442143,0.02010731){\makebox(0,0)[lt]{\lineheight{1.25}\smash{\begin{tabular}[t]{l}$0$\end{tabular}}}}%
    \put(0.14972391,0.0188542){\makebox(0,0)[lt]{\lineheight{1.25}\smash{\begin{tabular}[t]{l}$\Delta t$\end{tabular}}}}%
    \put(0.28348213,0.01815426){\makebox(0,0)[lt]{\lineheight{1.25}\smash{\begin{tabular}[t]{l}$2\Delta t$\end{tabular}}}}%
    \put(0.42900706,0.01280407){\makebox(0,0)[lt]{\lineheight{1.25}\smash{\begin{tabular}[t]{l}$3\Delta t$\end{tabular}}}}%
    \put(0.56686067,0.01405697){\makebox(0,0)[lt]{\lineheight{1.25}\smash{\begin{tabular}[t]{l}$4\Delta t$\end{tabular}}}}%
    \put(0.70220754,0.01029765){\makebox(0,0)[lt]{\lineheight{1.25}\smash{\begin{tabular}[t]{l}$T$\end{tabular}}}}%
    \put(0,0){\includegraphics[width=\unitlength,page=2]{constantfunction.pdf}}%
  \end{picture}%
\endgroup%

		\caption{Piecewise constant input function.\label{figure:constantfunction} }
	\end{figure}
\end{remark}

}

The following assumption is imposed on the input library.
\begin{assumption}
	\label{assumption:inputlibrary}
	The input library $(\mathcal{U}_{C}, \mathcal{U}_{D})$ is such that
	\begin{enumerate}[label=\arabic*)]
		\item Each input signal in $\mathcal{U}_{C}$ is constant and $\mathcal{U}_{C}$ includes all possible input signals such that their time domains are subsets of the interval $[0, T_{m}]$ for some $T_{m} > 0$ and their images belong to $U_{C}$. In other words, there exists $T_m > 0$ such that $\mathcal{U}_{C} = \{ \tilde{u} : \dom \tilde{u} = [0,T] \subset [0,T_m], \tilde{u} \text{ is } \text{constant} \text{ and } \rge \tilde{u} \in U_C\}$;
		\item $\mathcal{U}_D = U_D$.
	\end{enumerate}
\end{assumption}
\ifbool{conf}{}{
\begin{remark}
	Assumption \ref{assumption:inputlibrary} assumes that the input signals in $\mathcal{U}_{C}$ are all constant functions. Assuming that the input $u$ of the motion plan is piecewise constant, $u$ can be constructed by concatenating constant input signals in $\mathcal{U}_{C}$.  The values of those functions cover all possible input values in $U_{C}$. The domain of all those functions covers all possible duration in $[0, T_{m}]$. 

	In the function call $new\_state$, an input signal $\tilde{u}$ is randomly selected from $\mathcal{U}_{C}$. This selection can be implemented by randomly selecting $t_{m}$ from the interval $ [0, T_{m}]$ and $u_{C}$ from $U_{C}$, and then constructing the constant input signal $\tilde{u}:[0, t_{m}]\to u_{C}$.
\end{remark}
}

The following assumption is imposed on the random selection in HyRRT.
\begin{assumption}
		\label{assumption:uniformsample}
The probability distributions of the random selection in the function calls $T.init$, $random\_state$, and  $new\_state$ are the uniform distribution.
	\end{assumption}

\ifbool{conf}{}{
\begin{remark}
		By Assumption \ref{assumption:uniformsample}, the computation of the probability of randomly selecting a point that lies in a given set is simplified. When randomly selecting a point $s$ from the set $S$, the probability that $s$ belongs to a subset $R\subset S$ is
		\begin{equation}
		\label{equation:probabilitylebesgue}
		Pr(s\in R) = \int_{s\in R} \frac{1}{\mu(S)}\text{d}s = \frac{\mu(R)}{\mu(S)}
		\end{equation}
		where $\mu(R)$ denotes the Lebesgue measure of the set $R$ and $\mu(S)$ denotes the Lebesgue measure of the set $S$.
\end{remark}
}

The following assumptions are imposed on the flow map $f$ and the jump map $g$ of the hybrid system $\mathcal{H}$ in (\ref{model:generalhybridsystem}).
\begin{assumption}
	\label{assumption:flowlipschitz}
	The flow map $f$ is Lipschitz continuous. In particular, there exist $K^{f}_{x}, K^{f}_{u}\in \mathbb{R}_{>0}$ such that, for all $(x_{0}, x_{1}, u_{0}, u_{1})$ such that $(x_{0}, u_{0}) \in C$, $(x_{0}, u_{1}) \in C$, and $(x_{1}, u_{0}) \in C$,
	$$
	\begin{aligned}
	|f(x_{0}, u_{0}) - f(x_{1}, u_{0})|&\leq K^{f}_{x}|x_{0} - x_{1}|\\
	|f(x_{0}, u_{0}) - f(x_{0}, u_{1})|&\leq K^{f}_{u}|u_{0} - u_{1}|.
	\end{aligned}
$$
\end{assumption}
\ifbool{conf}{}{
\begin{remark}
	This assumption guarantees that the flow map is Lipschitz continuous for both state and input arguments. This assumption provides an explicit upper bound for the distance between the motion plan and the simulated solution pair in the forthcoming Lemma \ref{lemma:pccontinuouslowerbound}.
\end{remark}
}

\ifbool{conf}{}
{
	The forthcoming Lemma \ref{lemma:pccontinuouslowerbound} characterizes the probability that the simulated solution pair computed by the function call $new\_state$ in the flow regime is close to the motion plan. 
	\begin{lemma}
	\label{lemma:pccontinuouslowerbound}
	Given a hybrid system $\mathcal{H}$ that satisfies Assumption \ref{assumption:flowlipschitz} and an input library that satisfies item 1) of Assumption \ref{assumption:inputlibrary}, let $\psi = (\phi, u)$ be a purely continuous solution pair to $\mathcal{H}$ with clearance $\delta > 0$, $(\tau, 0) = \max \dom \psi$, and constant input function $u$. Suppose that $\tau \leq T_{m}$, where $T_{m}$ comes from item 1) in Assumption \ref{assumption:inputlibrary}. Suppose Assumption \ref{assumption:uniformsample} is satisfied.
	Let $\psi_{new} = (\phi_{new}, u_{new})$ be the purely continuous solution pair generated by the function $new\_state$ in Algorithm \ref{algo:hybridRRT} with $flag = flow$ and initial state $\overline{x}_{v_{cur}} = \phi_{new}(0, 0) \in \phi(0, 0) + \kappa_{1}\delta\mathbb{B}$ for some $\kappa_{1}\in (0, 1/2)$.
	Then, for each $\kappa_{2}\in (2\kappa_{1}, 1)$ and each $\epsilon\in (0, \frac{\kappa_{2}\delta}{2})$, there exists $p_{t}\in (0, 1]$ such that 
	\begin{equation}
	\label{equation:lemmaflow}
	\begin{aligned}
		&Pr(E_{1} \& E_{2})\geq \\
		&p_{t}\frac{\zeta_{n} \left(\max \left\{\min \left\{\frac{\frac{\kappa_{2}\delta}{2} - \epsilon -\exp(K_{x}^{f}\tau)\kappa_{1}\delta}{K^{f}_{u}\tau \exp(K_{x}^{f}\tau)}, \delta\right\}, 0\right\}\right)^{m}}{\mu(U_{C})}.
	\end{aligned}
	\end{equation}
	where 
	\begin{enumerate}
		\item $E_{1}$ denotes the event that $\phi$ and $\phi_{new}$ are $(\overline{\tau}, \kappa_{2}\delta)$-close where $(\tau', 0) = \max \dom \phi_{new}$ and $\overline{\tau} = \max (\tau, \tau')$;
		\item $E_{2}$ denotes the event that $x_{new} = \phi_{new}(\tau', 0)\in \phi(\tau, 0) + \kappa_{2}\delta \mathbb{B}$ where $x_{new}$ stores the final state of $\phi_{new}$ in the function call $new\_state$ as is introduced in Section \ref{section:newstate},
	\end{enumerate} and $\zeta_{n}$ is given in (\ref{equation:zetan}), $\mu(U_{C})$ denotes the Lebesgue measure of $U_{C}$, and $K_{x}^{f}$ and $K_{u}^{f}$ come from Assumption \ref{assumption:flowlipschitz}.  
\end{lemma}

	\begin{proof}
Let $r' = \frac{\kappa_{2} \delta}{2} - \epsilon$. Note that $\epsilon \in (0, \frac{\kappa_{2}\delta}{2})$. Therefore, $r' > 0$. Construct a sequence of balls centered at $c(t) = \phi(t, 0)$ for all 
\begin{equation}\label{equation:Tk}
	t\in T_{k} := \{t_{l}\in [\max\{\tau - \kappa_{2}\delta, 0\}, \tau]: \forall t'\in [t_{l}, \tau], 
	\phi(t', 0) + r'\mathbb{B} \subset \phi(\tau, 0) + \frac{\kappa_{2}\delta}{2}\mathbb{B}\}.
\end{equation}
See the red balls in Figure \ref{figure:flowlemma} as an illustration of this sequence of balls. 

The construction of the balls centered at $c(t)$ with radius $r'$ as above is feasible. Note that $\phi$ is purely continuous. Therefore, for arbitrary small $\epsilon > 0$, we can find a lower bound $t_{l} \leq \tau$ such that for all $t'\in [t_{l}, \tau]$, $|\phi(t', 0) - \phi(\tau, 0)|< \epsilon$. For all $t'\in [t_{l}, \tau]$ and any point $x_{p}\in \phi(t', 0) + r'\mathbb{B}$, then
$$
\begin{aligned}
|x_{p}- \phi(\tau, 0)| &= |x_{p} - \phi(t', 0) +  \phi(t', 0) - \phi(\tau, 0)|\\
&\leq  |x_{p} - \phi(t', 0)| +  |\phi(t', 0) - \phi(\tau, 0)|\\
&\leq r' + \epsilon = \frac{\kappa_{2}\delta}{2}.
\end{aligned}
$$ Therefore, for all $t'\in [t_{l}, \tau]$, $x_{p}\in \phi(t', 0) + r'\mathbb{B}$ implies $x_{p}\in \phi(\tau, 0) + \frac{\kappa_{2}\delta}{2}\mathbb{B}$. This implication leads to $\phi(t', 0) + r'\mathbb{B}\subset \phi(\tau, 0) + \frac{\kappa_{2}\delta}{2}\mathbb{B}$ for all $t'\in [t_{l}, \tau]$. Therefore, $T_{k}$ is not empty set and, hence, the construction of balls centered at $c(t')$ with radius $r'$ is feasible.

\begin{figure}[htbp]
	\centering
	\def\svgwidth{0.9\columnwidth}
	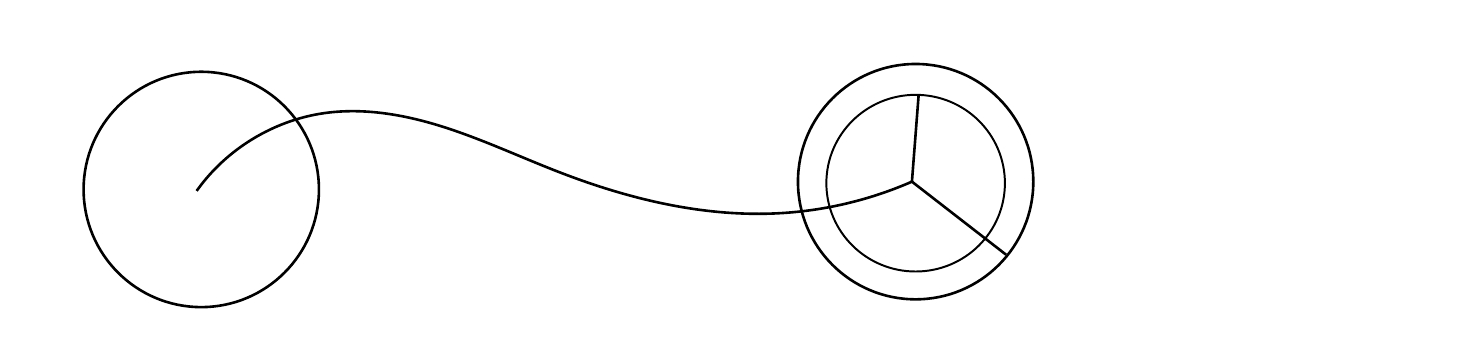

	\caption{Illustration of Lemma \ref{lemma:pccontinuouslowerbound}.\label{figure:flowlemma} }
\end{figure}

As is assumed, the function $new\_state$ proceeds with $flag = flow$. Item 1) in Assumption \ref{assumption:inputlibrary} assumes that the input signals in $\mathcal{U}_{C}$ are constant, denoted $\tilde{u}': [0, t'_{m}]\to u'_{C}\in U_{C}$ where $t'_{m}$ denotes the time duration of $\tilde{u}'$. Besides, since Assumption \ref{assumption:uniformsample} is assumed, the function call $new\_state$ randomly selects $\tilde{u}'$ from $\mathcal{U}_{C}$, which is equivalent with randomly selecting $t'_{m}$ from $[0, T_{m}]$ and $u'_{C}$ from $U_{C}$, with uniform distribution. Since the input function $u$ is also assumed to be constant, denote $u: [0, \tau]\times\{0\}\to u_{C}\in U_{C}$.

 Since Assumption \ref{assumption:flowlipschitz} is assumed, from Lemma 2 in \cite{kleinbort2018probabilistic} (included in Appendix), we have  
\begin{equation}
\label{equation:bound}
	|\phi(t, 0) - \phi_{new}(t, 0)| \leq \exp(K_{x}^{f}t)\kappa_{1}\delta + K^{f}_{u}t\exp(K_{x}^{f}t)\Delta u
\end{equation}
for all $(t, 0)\in \dom \phi \cap \dom \phi_{new} = [0, \min \{\tau, \tau'\}]\times \{0\}$ where $\Delta u = |u_{C} - u'_{C}| $. 
Note that if the time duration $t'_{m}$ of the input signal $\tilde{u}'$ is selected such that 
\begin{equation}
\label{equation:tmtau}
t'_{m} \leq \tau,
\end{equation} then 
$$
\tau' \leq t'_{m} \leq \tau
$$ and  
$$\dom \phi \cap \dom \phi_{new} = [0, \min \{\tau, \tau'\}]\times \{0\} = [0,  \tau']\times \{0\}$$
because the time duration $\tau'$ of $\psi_{new}$ cannot exceed the time duration $t'_{m}$ of the applied input signal $\tilde{u}'$; see Problem \ref{problem:flowsimulator}.

If $\Delta u$ is such that
\begin{equation}
\label{equation:inequalityflowproof1}
\Delta u < \frac{\frac{\kappa_{2} \delta}{2} - \epsilon -\exp(K_{x}^{f}\tau)\kappa_{1}\delta}{K^{f}_{u}\tau \exp(K_{x}^{f}\tau)},
\end{equation}
then according to (\ref{equation:bound}), 
\begin{equation}\label{equation:flowlemmainequality}
	\begin{aligned}
	|\phi(t, 0) - \phi_{new}(t, 0)| &\leq \exp(K_{x}^{f}t)\kappa_{1}\delta + K^{f}_{u}t\exp(K_{x}^{f}t)\Delta u \\
	&\leq \exp(K_{x}^{f}\tau)\kappa_{1}\delta + K^{f}_{u}\tau\exp(K_{x}^{f}\tau)\Delta u\\
	&\leq \frac{\kappa_{2}\delta}{2} - \epsilon.
	\end{aligned}
\end{equation}
for all $t\in [0, \tau']$, which implies $\phi_{new}(t, 0) \in \phi(t, 0) +(\kappa_{2} \delta/2 - \epsilon)\mathbb{B}$ for all $t\in [0, \tau']$.

From Definition \ref{definition:clearance}, the following is also required such that the safety requirement is met:
\begin{equation}
\label{equation:inequalityflowproof2} 
\Delta u \leq \delta.
\end{equation}

Hence, to satisfy both (\ref{equation:inequalityflowproof1}) and (\ref{equation:inequalityflowproof2}), we have 
\begin{equation}
\label{equation:safety}
	\Delta u < \min \left\{\frac{\frac{\kappa_{2} \delta}{2} - \epsilon -\exp(K_{x}^{f}\tau)\kappa_{1}\delta}{K^{f}_{u}\tau \exp(K_{x}^{f}\tau)}, \delta\right\}.
\end{equation}

Then, we are ready to compute the probability of the occurrence of both $E_{1}$ and $E_{2}$. The probability that both events $E_{1}$ and $E_{2}$ occur should be at least the product of
\begin{enumerate}
	\item The probability $p_{t}$ of selecting $t'_{m}$ from $[0, T_{m}]$ such that 
	\begin{equation}
		\label{equation:tm}
	\begin{aligned}
		t'_{m} \in T_{k} = \{t_{l}\in [\max\{\tau - \kappa_{2}\delta, 0\}, \tau]: \forall t'\in [t_{l}, \tau], 
		\phi(t', 0) + r'\mathbb{B} \subset \phi(\tau, 0) + \frac{\kappa_{2}\delta}{2}\mathbb{B}\};
	\end{aligned}
	\end{equation}
	\item The probability $p_{u}$ of selecting $u'_{C}$ from $U_{C}$ such that 
	\begin{equation}
			\label{equation:inequalityflowproof3}
			\Delta u< \min \left\{\frac{\frac{\kappa_{2} \delta}{2} - \epsilon -\exp(K_{x}^{f}\tau)\kappa_{1}\delta}{K^{f}_{u}\tau \exp(K_{x}^{f}\tau)}, \delta\right\}.
	\end{equation}
\end{enumerate}

The probability $p_{t}$ in the first item above is
\begin{equation}
p_{t} = \frac{\mu(T_{k})}{T_{m}} \in (0, 1].
\end{equation}

Note that the choice of $u'_{C}$ that satisfies (\ref{equation:inequalityflowproof3}) is a ball in $\mathbb{R}^{m}$ centered at $u_{C}$ with radius $\min \left\{\frac{\frac{\kappa_{2}\delta}{2} - \epsilon -\exp(K_{x}^{f}\tau)\kappa_{1}\delta}{K^{f}_{u}\tau \exp(K_{x}^{f}\tau)}, \delta\right\}$. Therefore, the Lebesgue measure of this ball is $\zeta_{n}\left(\min \left\{\frac{\frac{\kappa_{2}\delta}{2} - \epsilon -\exp(K_{x}^{f}\tau)\kappa_{1}\delta}{K^{f}_{u}\tau \exp(K_{x}^{f}\tau)}, \delta\right\}\right)^{m}$ where $\zeta_{n}$ denotes the Lebesgue measure of the unit ball in $\mathbb{R}^{m}$.

Hence, according to (\ref{equation:probabilitylebesgue}), the probability $p_{u}$ in the second item above is
\begin{equation}
p_{u} = \frac{\zeta_{n} \left(\min \left\{\frac{\frac{\kappa_{2}\delta}{2} - \epsilon -\exp(K_{x}^{f}\tau)\kappa_{1}\delta}{K^{f}_{u}\tau \exp(K_{x}^{f}\tau)}, \delta\right\}\right)^{m}}{\mu(U_{C})}.
\end{equation}
where $\mu(U_{C})$ denotes the Lebesgue measure of $U_{C}$. Note that $p_{u}$ is probability and cannot be negative. Therefore, $p_{u}$ is rewritten as follows to rule out the negative values:
\begin{equation}
p_{u} = \frac{\zeta_{n} \left(\max \left\{\min \left\{\frac{\frac{\kappa_{2}\delta}{2} - \epsilon -\exp(K_{x}^{f}\tau)\kappa_{1}\delta}{K^{f}_{u}\tau \exp(K_{x}^{f}\tau)}, \delta\right\}, 0\right\}\right)^{m}}{\mu(U_{C})}.
\end{equation}
Hence, 
$$
\begin{aligned}
&Pr(E_{1}\&E_{2})\geq p_{t} p_{u} \\
&= p_{t}\frac{\zeta_{n} \left(\max \left\{\min \left\{\frac{\frac{\kappa_{2}\delta}{2} - \epsilon -\exp(K_{x}^{f}\tau)\kappa_{1}\delta}{K^{f}_{u}\tau \exp(K_{x}^{f}\tau)}, \delta\right\}, 0\right\}\right)^{m}}{\mu(U_{C})}.
\end{aligned}
$$

Next, we prove that if the random selection of $\tilde{u}'$ satisfies (\ref{equation:tm}) and (\ref{equation:inequalityflowproof3}), then both $E_{1}$ and $E_{2}$ occur.
We start with proving that $E_{1}$ occurs when (\ref{equation:tm}) and (\ref{equation:inequalityflowproof3}) hold. Note that if (\ref{equation:tm}) is satisfied, then (\ref{equation:tmtau}) holds and hence, $\tau' \leq \tau$ and $\overline{\tau} = \max (\tau, \tau') = \tau$ in $E_{1}$. The following shows that $\psi$ and $\psi_{new}$ satisfy each item in Definition \ref{definition:closeness}.
\begin{enumerate}
	\item This item proves that $\psi$ and $\psi_{new}$ satisfy the first item in Definition \ref{definition:closeness}.
	Because $\tau' \leq \tau$, according to (\ref{equation:flowlemmainequality}), for all $(t, 0)\in \dom \phi_{new}$ with $t + 0 \leq \tau' + 0 \leq \tau + 0 =\tau = \overline{\tau}$, there exists $s = t$ such that $(s, 0)\in \dom \phi$, $|t - s|= 0 < \kappa_{2}\delta$ and 
	$$
	|\phi(t, 0) - \phi_{new}(s, 0)| \leq \frac{\kappa_{2} \delta}{2} - \epsilon < \kappa_{2} \delta - \epsilon < \kappa_{2} \delta.
	$$
	Hence, item 1) in Definition \ref{definition:closeness} is proved. 
	\item To prove item 2) in Definition \ref{definition:closeness}, we consider the following two cases.
	\begin{enumerate}
		\item For all $(t, 0)\in \dom \phi$ with $0 \leq t + 0 \leq \tau' + 0 = \tau'$, there exists $s = t$ such that $(s, 0)\in \dom \phi_{new}$, $|t - s| = 0 < \kappa_{2} \delta$ and 
		$$
		|\phi(t, 0) - \phi_{new}(s, 0)| \leq \frac{\kappa_{2} \delta}{2} - \epsilon < \kappa_{2} \delta - \epsilon < \kappa_{2} \delta
		$$ because of (\ref{equation:flowlemmainequality}).
		\item This item considers the case of $(t, 0)\in \dom \phi$ with $\tau'\leq t + 0 \leq \tau + 0 = \tau$. We first shows $\tau' = t'_{m}$. 
		
		When both (\ref{equation:tm}) and (\ref{equation:inequalityflowproof3}) hold, the helper function $\hat{\phi}$ in (\ref{equation:flowequationintegration}) constructed by the simulator of  continuous dynamics is such that 
		$$
		\begin{aligned}
		|\hat{\phi}(t) - \phi(t, 0)| &\leq \exp(K_{x}^{f}t)\kappa_{1}\delta + K^{f}_{u}t\exp(K_{x}^{f}t)\Delta u\\
		&\leq \frac{\kappa_{2}\delta}{2} - \epsilon \\
		&< \delta
		\end{aligned}
		$$ for all $t\in [0, t'_{m}]$ because of Lemma 2 in \cite{kleinbort2018probabilistic}. Hence, for all $t\in [0, t'_{m}]$,
		$$
			\hat{\phi}(t)\in \phi(t, 0) + \delta\mathbb{B}.
		$$
		Since (\ref{equation:inequalityflowproof2}) holds and both $\tilde{u}'$ and $u$ are constant, therefore,
		$$
			\tilde{u}'(t) = u'_{C}\in u_{C} + \delta\mathbb{B} = u(t, 0) + \delta\mathbb{B}.
		$$ for all $t\in [0, t'_{m}]$.
		Since $\psi$ is assumed to have clearance $\delta$, according to item 1) of Definition \ref{definition:clearance}, then $(\hat{\phi}(t), \tilde{u}'(t))\in C$ for all $t\in [0, t'_{m}]$. Therefore, according to (\ref{equation:integrationduration}),
		$$
		\begin{aligned}
		\tau' &= \max \{t\in [0, t'_{m}]: \forall t'\in (0, t), (\hat{\phi}(t'), \tilde{u}'(t')) \in C\} \\
		&= t'_{m}.
		\end{aligned}
		$$ 
		
		Then we are ready to prove the second case. For all $(t, 0)\in \dom \phi$ with $\tau'\leq t + 0 \leq \tau + 0 = \tau$, let $s = \tau'$. Because of (\ref{equation:tm}), then $s = \tau' = t'_{m}\in [\max\{\tau - \kappa_{2}\delta, 0\}, \tau]$. Since $s \in [\max\{\tau - \kappa_{2}\delta, 0\}, \tau]$ and $t\in [\tau', \tau]$, therefore, we have 
		$$
		|t - s| \leq \kappa_{2}\delta.
		$$ Also, because of (\ref{equation:tm}), for all $(t, 0)\in \dom \phi$ with $\tau' =t'_{m}\leq t \leq \tau$, we have
		\begin{equation}
		\phi(t, 0) \in \phi(t, 0) + r'\mathbb{B} \subset \phi(\tau, 0) + \frac{\kappa_{2}\delta}{2}\mathbb{B}
		\end{equation}
		and 
		\begin{equation}
		\phi_{new}(s, 0) \in \phi(s, 0) + r'\mathbb{B} = \phi(t'_{m}, 0) + r'\mathbb{B} \subset \phi(\tau, 0) + \frac{\kappa_{2}\delta}{2}\mathbb{B}
		\end{equation}
		because of (\ref{equation:flowlemmainequality}) and (\ref{equation:tm}).
		Therefore, for all $(t, 0)\in \dom \phi$ with $\tau'\leq t \leq \tau$, $\phi(t, 0)$ and $\phi_{new}(s, 0)$ are both in a ball centered at $\phi(\tau, 0)$ with radius $\frac{\kappa_{2}\delta}{2}$. Note that the maximum distance between two points within a circle is its diameter. 
		
		Hence, for all $(t, 0)\in \dom \phi$ with $\tau'\leq t + 0 \leq \tau + 0 = \tau$, there exists $s = \tau'$ such that $(s, 0)\in \dom \phi_{new}$, $|t - s| < \kappa_{2} \delta$, and
		$$|\phi(t, 0) - \phi_{new}(s, 0)| < \kappa_{2} \delta.$$ 
	\end{enumerate}
\end{enumerate}
Therefore, when both (\ref{equation:tm}) and (\ref{equation:inequalityflowproof3}) hold, $\phi$ and $\phi_{new}$ are $(\overline{\tau}, \kappa_{2}\delta)$-close, and, hence, $E_{1}$ is guaranteed to occur. 

Next, we prove that $E_{2}$ occurs when both (\ref{equation:tm}) and (\ref{equation:inequalityflowproof3}) hold. Since $t'_{m} = \tau'$, then
$$
\begin{aligned}
\phi_{new}(\tau', 0) &= \phi_{new}(t'_{m}, 0)\\
&\in \phi(t'_{m}, 0) + r'\mathbb{B}.\\
\end{aligned} 
$$ because of (\ref{equation:flowlemmainequality}).
Since (\ref{equation:tm}) holds, then
$$
\begin{aligned}
 \phi(t'_{m}, 0) + r'\mathbb{B}&\subset \phi(\tau, 0) + \frac{\kappa_{2}\delta}{2}\mathbb{B}\\
&\subset \phi(\tau, 0) + \kappa_{2}\delta\mathbb{B}.
\end{aligned}
$$
Therefore, $\phi_{new}(\tau', 0)\in \phi(\tau, 0) + \kappa_{2}\delta\mathbb{B}$ and, hence, $E_{2}$ is guaranteed to occur.
\end{proof}
}

\begin{assumption}
	\label{assumption:pcjumpmap}
	The jump map $g$ is such that there exist $K^{g}_{x}\in \mathbb{R}_{>0}$ and $K^{g}_{u}\in \mathbb{R}_{>0}$ such that, for all $(x_{0}, u_{0}) \in D$ and $(x_{1}, u_{1}) \in D$,
	$$
	|g(x_{0}, u_{0}) - g(x_{1}, u_{1})|\leq K^{g}_{x}|x_{0} - x_{1}| + K^{g}_{u}|u_{0} - u_{1}|.
	$$
\end{assumption}
\ifbool{conf}{}{
	\begin{remark}
	This assumption relates the upper bound over the distance between $g(x_{0}, u_{0})$ and $g(x_{1}, u_{1})$ with the distance between the initial states $x_{0}$ and  $x_{1}$ and the distance between the inputs $u_{0}$ and $u_{1}$ which contributes to the forthcoming Lemma \ref{lemma:pcdiscretelowerbound}.
\end{remark}
}

\ifbool{conf}{}{
	The forthcoming Lemma \ref{lemma:pcdiscretelowerbound} characterizes the probability that the simulated solution pair computed by the function call $new\_state$ in the jump regime is close to the motion plan. 
	\begin{lemma}
	\label{lemma:pcdiscretelowerbound}
	Given a hybrid system $\mathcal{H}$ that satisfies Assumption \ref{assumption:pcjumpmap} and an input library that satisfies item 2) of Assumption \ref{assumption:inputlibrary}, let $\psi = (\phi, u)$ be a purely discrete solution pair to $\mathcal{H}$ with a single jump, i.e., $\max \dom \psi = (0, 1)$ and clearance $\delta > 0$. 
	Let $\psi_{new} = (\phi_{new}, u_{new})$ be the purely discrete solution pair generated by the function $new\_state$ in Algorithm \ref{algo:hybridRRT} with $flag = jump$ and initial state $\overline{x}_{v_{cur}} = \phi_{new}(0, 0)\in \phi(0, 0) + \kappa_{1}\delta\mathbb{B}$ for some positive $\kappa_{1}\in (0, 1]$.
	Then, for any $\kappa_{2}\in (0, 1]$,  we have that
	\begin{equation}
	Pr(E)\geq \frac{\zeta_{n} \left(\max \{\min \{\frac{(\kappa_{2} - K^{g}_{x}\kappa_{1})\delta}{K^{g}_{u}}, \delta\}, 0\}\right)^{m}}{\mu(U_{D})}
	\end{equation}
	where $E$ denotes the event that $x_{new} = \phi_{new}(0, 1) \in \phi(0, 1) + \kappa_{2} \delta\mathbb{B}$, $x_{new}$ stores the final state of $\phi_{new}$ in the function call $new\_state$ as is introduced in Section \ref{section:newstate}, $\zeta_{n}$ is given in Section \ref{section:preliminary}, $\mu(U_{D})$ denotes the Lebesgue measure of $U_{D}$, and $K^{g}_{x}$ and $K^{g}_{u}$ come from Assumption \ref{assumption:pcjumpmap}.  
%
%
\end{lemma}
\begin{proof}
	Since Assumption \ref{assumption:pcjumpmap} is satisfied, we have
	\begin{equation}
	\begin{aligned}
		|\phi(0, 1) - \phi_{new}(0, 1)| &\leq K^{g}_{x}|\phi(0, 0) - x'_{0}| + K^{g}_{u}|u(0, 0) - u'(0, 0)|\\
		&\leq K^{g}_{x}\kappa_{1}\delta + K^{g}_{u}\Delta u
	\end{aligned}
	\end{equation}
	where $\Delta u = |u(0, 0) - u'(0, 0)|$. 
	
	If $\Delta u$ is such that 
	\begin{equation}
	\label{equation:inputclose}
	\Delta u \leq \frac{(\kappa_{2} - K^{g}_{x}\kappa_{1})\delta}{K^{g}_{u}},
	\end{equation}
	then, 
	$$
		|\phi(0, 1) - \phi_{new}(0, 1)| \leq K^{g}_{x}\kappa_{1}\delta + K^{g}_{u}\Delta u \leq \kappa_{2}\delta.
	$$ and, hence, $\phi'(0, 1) \in \phi(0, 1) + \kappa_{2}\delta\mathbb{B}$.

	From Definition \ref{definition:clearance}, the following is also required such that the safety requirement is met:
	\begin{equation}
	\label{equation:jumpsafety}
		\Delta u \leq \delta.
	\end{equation}
	Therefore, to make both (\ref{equation:inputclose}) and (\ref{equation:jumpsafety}) hold, we require that
	\begin{equation}
	\label{equation:lemmajump}
		\Delta u \leq \min \{\frac{(\kappa_{2} - K^{g}_{x}\kappa_{1})\delta}{K^{g}_{u}}, \delta\}.
	\end{equation}
	Note that Assumption \ref{assumption:uniformsample} is satisfied. Therefore, the probability to choose the input that satisfy (\ref{equation:lemmajump}) is
	\begin{equation}
	\label{equation:jumpsamplelowerbound}
		p_{u} = \frac{\zeta_{n} (\max \{\min \{\frac{(\kappa_{2} - K^{g}_{x}\kappa_{1})\delta}{K^{g}_{u}}, \delta\}, 0\})^{m}}{\mu(U_{D})}
	\end{equation}
	where $\zeta_{n}$ is the Lebesgue measure of the unit ball in $\mathbb{R}^{m}$ and $\mu(U_{D})$ denotes the Lebesgue measure of $U_{D}$.
	Therefore, 
	\begin{equation}
	\label{equation:lemmadiscrete}
		Pr(E) \geq p_{u} = \frac{\zeta_{n} \left(\max \{\min \{\frac{(\kappa_{2} - K^{g}_{x}\kappa_{1})\delta}{K^{g}_{u}}, \delta\}, 0\}\right)^{m}}{\mu(U_{D})}.
	\end{equation}
\end{proof}
\begin{remark}
	From Equation \ref{equation:jumpsamplelowerbound}, it is obvious that the lower bound $p_{u}$ of $E$ is positive when 
	\begin{equation}
	\label{equation:jumppossibilityinequality}
		\kappa_{2} > K^{g}_{x}\kappa_{1}.
	\end{equation}
	Note that when $\phi_{new}(0, 0)\in \phi(0, 0) + \kappa_{1}\delta\mathbb{B}$ and $\phi_{new}(0, 1)\in \phi(0, 1) + \kappa_{2}\delta\mathbb{B}$ are satisfied, $\phi$ and $\phi_{new}$ are $(1, \max (\kappa_{1}\delta, \kappa_{2}\delta))$-close.
\end{remark}
}

\ifbool{conf}{}{
	Next, we show that the concatenated operation keeps the closeness between the hybrid arcs.
	\begin{proposition}
	\label{proposition:concatenateclose}
	Given compact hybrid arcs $\phi_{1}$,  $\phi_{2}$,  $\phi'_{1}$, and $\phi'_{2}$ such that $\phi_{1}$ and $\phi'_{1}$ are $(\tau_{1}, \epsilon_{1})$-close and $\phi_{2}$ and $\phi'_{2}$ are $(\tau_{2}, \epsilon_{2})$-close, where $(T_{1}, J_{1}) = \max \dom \phi_{1}$, $(T'_{1}, J'_{1}) = \max \dom \phi'_{1}$ and $\tau_{1} = \max \{T_{1} + J_{1}, T'_{1} + J'_{1}\}$, then $\phi_{1}|\phi_{2}$ and $\phi'_{1}|\phi'_{2}$ are $(\tau_{1}+ \tau_{2}, \epsilon_{1} + \epsilon_{2})$-close.
\end{proposition}
\begin{proof}
	Since $\phi_{1}$ and $\phi'_{1}$ are $(\tau_{1}, \epsilon_{1})$-close, then 
	\begin{enumerate}
		\item for all $(t, j)\in \dom \phi_{1}$ with $t + j \leq \tau_{1}$, there exists $s$ such that $(s, j)\in \dom \phi'_{1}$, $|t - s|< \epsilon_{1}$, and $|\phi_{1}(t, j) - \phi'_{1}(s, j)| < \epsilon_{1}$;
		\item for all $(t, j)\in \dom \phi'_{1}$ with $t + j \leq \tau_{1}$, there exists $s$ such that $(s, j)\in \dom \phi_{1}$, $|t - s|< \epsilon_{1}$, and $|\phi'_{1}(t, j) - \phi_{1}(s, j)| < \epsilon_{1}$.
	\end{enumerate}

	Since $\phi_{2}$ and $\phi'_{2}$ are $(\tau_{2}, \epsilon_{2})$-close, then 
	\begin{enumerate}
		\item for all $(t, j)\in \dom \phi_{2}$ with $t + j \leq \tau_{2}$, there exists $s$ such that $(s, j)\in \dom \phi'_{2}$, $|t - s|< \epsilon_{2}$, and $|\phi_{2}(t, j) - \phi'_{2}(s, j)| < \epsilon_{2}$;
		\item for all $(t, j)\in \dom \phi'_{2}$ with $t + j \leq \tau_{2}$, there exists $s$ such that $(s, j)\in \dom \phi_{2}$, $|t - s|< \epsilon_{2}$, and $|\phi'_{2}(t, j) - \phi_{2}(s, j)| < \epsilon_{2}$.
	\end{enumerate}

	We want to show that $\phi = \phi_{1}|\phi_{2}$ and $\phi' = \phi'_{1}|\phi'_{2}$ are $(\tau_{1}+ \tau_{2}, \epsilon_{1} + \epsilon_{2})$-close.
	\begin{enumerate}
		\item for all $(t, j)\in \dom \phi$ with $t + j \leq \tau_{1}+ \tau_{2}$, there exists $s$ such that $(s, j)\in \dom \phi'$, $|t - s|< \epsilon_{1} + \epsilon_{2}$, and $|\phi(t, j) - \phi'(s, j)| < \epsilon_{1} + \epsilon_{2}$;
		\item for all $(t, j)\in \dom \phi'$ with $t + j \leq \tau_{1}+ \tau_{2}$, there exists $s$ such that $(s, j)\in \dom \phi$, $|t - s|< \epsilon_{1} + \epsilon_{2}$, and $|\phi'(t, j) - \phi(s, j)| < \epsilon_{1} + \epsilon_{2}$.
	\end{enumerate}
	We start with the first item above. From the definition of concatenation operation; see Definition \ref{definition:concatenation}, we have $\dom \phi = \dom \phi_{1} \cup (\dom \phi_{2} + \{(T_{1}, J_{1})\} )$, where $(T_{1}, J_{1}) = \max \dom \phi_{1}$. For all $(t, j)\in \dom \phi$ with $t + j \leq \tau_{1}+ \tau_{2}$, we consider the following two cases.
	\begin{enumerate}
		\item Consider the case when $(t, j)\in \dom \phi_{1}$. Since $\phi_{1}$ and $\phi'_{1}$ are $(\tau_{1}, \epsilon_{1})$-close, then we can find $s$ such that $(s, j)\in \dom \phi'_{1}\subset \dom \phi'$, $|t - s|< \epsilon_{1} < \epsilon_{1} + \epsilon_{2}$, and $|\phi(t, j) - \phi'(s, j)| = |\phi_{1}(t, j) - \phi'_{1}(s, j)| < \epsilon_{1} < \epsilon_{1} + \epsilon_{2}$; see item 1) of $\phi_{1}$ and $\phi'_{1}$ being $(\tau_{1}, \epsilon_{1})$-close above.
		\item Consider the case when $(t, j)\in (\dom \phi_{2} + \{(T_{1}, J_{1})\} )$. Since $\phi_{2}$ and $\phi'_{2}$ are $(\tau_{2}, \epsilon_{2})$-close, then for all $(t - T_{1}, j - J_{1})\in \dom \phi_{2}$ with $t - T_{1} +  j - J_{1} \leq \tau_{2}$, we can find $s$ such that $(s, j - J_{1})\in \dom \phi'_{2}$, $|t - T_{1} - s|< \epsilon_{2}$, and $|\phi_{2}(t - T_{1}, j - J_{1}) - \phi'_{2}(s, j - J_{1})| < \epsilon_{2}$; see item 1) of $\phi_{2}$ and $\phi'_{2}$ being $(\tau_{2}, \epsilon_{2})$-close above.
		
		Because of Definition \ref{definition:closeness} and $(s, j - J_{1})\in \dom \phi'_{2}$, we have that $s$ is such that $(s + T'_{1}, j)\in \dom \phi'$. Furthermore,  because of $|t - T_{1} - s|< \epsilon_{2}$, we have 
		$$
		|t - (s + T'_{1})|\leq |t - T_{1} - s| + |T_{1} - T'_{1}| < \epsilon_{2} + \epsilon_{1}.
		$$ Because of $|\phi_{2}(t - T_{1}, j - J_{1}) - \phi'_{2}(s, j - J_{1})| < \epsilon_{2}$, we have $$
		|\phi(t, j) - \phi'(s + T'_{1}, j)| = |\phi_{2}(t - T_{1}, j - J_{1}) - \phi'_{2}(s, j - J_{1})| <  \epsilon_{2} < \epsilon_{1} + \epsilon_{2}.
		$$ Therefore, we can find the $s' = s + T'_{1}$ such that $(s', j)\in \dom \phi'$, $|t - s'|< \epsilon_{1} + \epsilon_{2}$, and $|\phi(t, j) - \phi'(s', j)| < \epsilon_{1} + \epsilon_{2}$.
	\end{enumerate}
	The second item in Definition \ref{definition:closeness} can be proved in a similar way. Therefore, $\phi_{1}|\phi_{2}$ and $\phi'_{1}|\phi'_{2}$ are $(\tau_{1}+ \tau_{2}, \epsilon_{1} + \epsilon_{2})$-close.
\end{proof}
}

\ifbool{conf}{}{
	Next, we characterize the probability that a vertex in the search tree that is close to an existing motion plan is selected as $v_{cur}$ by the function call $nearest\_neighbor$ in Algorithm \ref{algo:hybridRRT}. 
	\begin{lemma}
	\label{lemma:nearestvertex}
	Suppose Assumption \ref{assumption:uniformsample} is satisfied. Given a hybrid system $\mathcal{H} = (C, f, D, g)$ with state $x\in \mathbb{R}^{n}$, let $x_{c}\in \mathbb{R}^{n}$ be such that $x_{c} + \delta\mathbb{B}\subset S$, where $S$ is either $C'$ or $D'$. Suppose that there exists a vertex $v$ in the search graph $\mathcal{T} = (V, E)$ such that $\overline{x}_{v}\in x_{c} + 2\delta/5\mathbb{B}$.  Denote $v_{cur}$ the return of the function call $nearest\_neighbor$ in Algorithm \ref{algo:hybridRRT}. The probability that $x_{v_{cur}}\in x_{c} + \delta\mathbb{B}$ is at least $\zeta_{n}(\delta/5)^{n}/\mu(S)$, where $\zeta_{n}$ is given in (\ref{equation:zetan}). 
\end{lemma}
\begin{proof}
	The proof is similar to the proof for Lemma 4 in \cite{kleinbort2018probabilistic}. Suppose there exists a vertex $z$ in the search graph $\mathcal{T} = (V, E)$ such that $\overline{x}_{z} \in S$ and $\overline{x}_{z}\notin x_{c} +\delta\mathbb{B}$, as otherwise it is immediate that $x_{v_{cur}}\in x_{c} + \delta\mathbb{B}$. We show that if the sampling point $x_{rand}$ returned by the function call $random\_state$ is such that $x_{rand}\in x_{c} + \delta/5\mathbb{B}$, then $x_{v_{cur}}\in x_{c} + \delta \mathbb{B}$.
	
	Since $x_{rand}\in x_{c} + \delta/5\mathbb{B}$, then 
	\begin{equation}
		|x_{rand} - x_{c}| \leq \delta/5.
	\end{equation}
	Since $\overline{x}_{v}\in x_{c}+ 2\delta/5\mathbb{B}$, then
	\begin{equation}
	|\overline{x}_{v} - x_{c}| \leq 2\delta/5.
	\end{equation}
	Therefore, 
	\begin{equation}
		|x_{rand} - \overline{x}_{v}| \leq |x_{rand} - x_{c}| + 	|x_{c} - \overline{x}_{v} |\leq \delta/5 + 2\delta/5 = 3\delta/5.
	\end{equation}
	 
	 Since $\overline{x}_{z}\notin x_{c} +\delta\mathbb{B}$, then 
	 \begin{equation}
	 |\overline{x}_{z} - x_{c}| > \delta.
	 \end{equation}
	 Note that
	 \begin{equation}
	 |x_{rand} - x_{c}| \leq \delta/5.
	 \end{equation}
	 Therefore, 
	 \begin{equation}
	 \delta < |\overline{x}_{z} - x_{c}| \leq |\overline{x}_{z} - x_{rand}| + | x_{rand} - x_{c}| \leq  |\overline{x}_{z} - x_{rand}| + \delta/5.
	 \end{equation}
	 Then we have
	 \begin{equation}
	 	|\overline{x}_{z} - x_{rand}| > 4\delta/5.
	 \end{equation}
	 
	 Since $|\overline{x}_{v} - x_{rand}| \leq 3\delta/5$ and $|\overline{x}_{z} - x_{rand}| > 4\delta/5$, $v$ is closer to $x_{rand}$ than $z$. It implies that $z$ will not be reported as $v_{cur}$ by the function call $nearest\_neighbor$. 
	 
	 If $v$ is reported as $v_{cur}$, then $x_{v_{cur}} = \overline{x}_{v} \in x_{c}+ 2\delta/5\mathbb{B} \subset x_{c} + \delta\mathbb{B}$.
	 If $v$ is not reported as $v_{cur}$, then there must exists another vertex $y$ such that $|\overline{x}_{y} - x_{rand}|\leq 3\delta/5$. Then $|\overline{x}_{y} - x_{c}|\leq |\overline{x}_{y} - x_{rand}| + |x_{rand} - x|\leq  4\delta/5$, which implies that $\overline{x}_{y} \in x_{c} + \delta \mathbb{B}$. This implies that if the sampling point $x_{rand}$ is such that $x_{rand}\in x_{c} + \delta/5\mathbb{B}$, no matter $v$ or $y$ is reported as $x_{v_{cur}}$, we have $x_{v_{cur}}\in x_{c} + \delta\mathbb{B}$.
	 
	 Note that $x_{rand}\in x_{c} + \delta/5\mathbb{B}$ implies $x_{rand}$ is sampled from the ball centered at $x_{c}$ with radius $\delta/5$. Therefore, since Assumption \ref{assumption:uniformsample} is assumed, the probability of $x_{rand}\in x_{c} + \delta/5\mathbb{B}$ is $\zeta_{n}(\delta/5)^{n}/\mu(S)$ where $\zeta_{n}$ denotes the Lebesgue measure of the unit ball in $\mathbb{R}^{n}$.
\end{proof}
\begin{remark}
	Lemma \ref{lemma:nearestvertex} shows that given $x_{c}\in \mathbb{R}^{n}$, when there exists a vertex $v$  such that $\overline{x}_{v}\in x_{c} + 2\delta/5\mathbb{B}$, then the probability that the function call $nearest\_neighbor$ selects a vertex that is close to $x_{c}$ is bounded from below by a positive constant. This lemma is used to provide a positive lower bound over the probability that a vertex that is close enough to the motion plan is returned by the function $nearest\_neighbor$ in Algorithm \ref{algo:hybridRRT}.
\end{remark}
}

The following assumption assumes that the existing motion plan is away from the boundary of initial state set, final state set, and unsafe set, and uses a piecewise-constant input during flows.
	\begin{assumption}\label{assumption:mpproblem}
	Given a motion planning problem $\mathcal{P} = (X_{0}, X_{f}, X_{u}, (C, f, D, g))$, there exists a motion plan $\psi = (\phi, u)$ to $\mathcal{P}$ such that for some $\delta' > 0$
	\begin{enumerate}
		\item  $\phi(0, 0) + \delta'\mathbb{B}\subset X_{0}$;
		\item  $\phi(T, J) + \delta'\mathbb{B}\subset X_{f}$, where $(T, J) = \max \dom \psi$;
		\item  for all $(t, j)\in \dom \psi$, $(\phi(t, j) + \delta'\mathbb{B}, u(t, j) + \delta'\mathbb{B}) \cap X_{u} =\emptyset$;
		\item for all $j\in \mathbb{N}$ such that $I^{j}$ has nonempty interior, $t\mapsto u(t, j)$ is piecewise constant with resolution $\Delta t$.
	\end{enumerate}
\end{assumption}

\ifbool{conf}{
}{
Note that Lemmas \ref{lemma:pccontinuouslowerbound} and \ref{lemma:pcdiscretelowerbound} provide the lower bound over the probability that starting from a state that is close to the motion plan, the function call $new\_state$ propagates a new state that is within the clearance of the motion plan. Lemma \ref{lemma:nearestvertex} provides the lower bound over the probability that the function call $nearest\_neighbor$ selects a vertex in the search graph such that its associated state is close enough to the motion plan. Now we are ready to prove our main theorem on probabilistic completeness when there is positive clearance.
}

\ifbool{conf}{}{
	\begin{theorem}
	\label{theorem:pc}
	Given a motion planning problem $\mathcal{P} = (X_{0}, X_{f}, X_{u}, (C, f, D, g))$, suppose that Assumptions \ref{assumption:inputlibrary}, \ref{assumption:uniformsample}, \ref{assumption:flowlipschitz}, and \ref{assumption:pcjumpmap} are satisfied and there exists a motion plan $(\phi, u)$ to $\mathcal{P}$  with clearance $\delta_{clear}$ satisfying Assumption \ref{assumption:mpproblem} for some $\delta' > 0$. Then, the probability that HyRRT fails to find a motion plan $\psi' = (\phi', u')$ such that $\phi'$ is $(\tilde{\tau}, \delta)$-close to $\phi$ after $k$ iterations is at most $a e^{-bk}$, for some constant $a, b\in \mathbb{R}_{> 0}$, where $(T, J) = \max \dom \psi$, $(T', J') = \max \dom \psi'$, $\tilde{\tau} = \max(T+ J, T' + J')$, and $\delta = \min \{\delta_{clear}, \delta'\}$.
\end{theorem}
\begin{proof}
	First, we construct a sequence of hybrid time instances $Q_{h} := \{(t_{i}, j_{i})\in \dom \psi\}_{i  = 1, 2,..., m}$ on the hybrid time domain of $\psi$ where $m$ denotes the total number of the hybrid time instances in this sequence such that 
		\begin{enumerate}
			\item $(t_{1}, j_{1}) = (0, 0)$,
			\item $(t_{m}, j_{m}) = (T, J)$ where $(T, J) = \max \dom \psi$,
			\item for all $i \in \{1, 2, ..., m - 1\}$, either of the following holds:
			\begin{enumerate}
				\item for some constant $\Delta t' \in (0, T_{m}]$ where $T_{m}$ is from Assumption \ref{assumption:inputlibrary},
				\begin{equation}
					\label{equation:deltatprime}
					t_{i + 1} = t_{i} + \Delta t'
				\end{equation}
				\item $j_{i + 1} = j_{i} + 1$;
			\end{enumerate}
		\end{enumerate}
	The construction of $Q_{h}$ is feasible because it truncates the hybrid time domain of the motion plan into purely continuous pieces with constant length, and purely discrete pieces with a single jump. 
	
	Then, for each $i\in \{1, 2,.., m\}$, we construct the radius $r_{i}\in \mathbb{R}_{>0}$ of the circles centered at $\phi(t_{i},  j_{i})$ for each $(t_{i},  j_{i})\in Q_{h}$ such that
	\begin{enumerate}
		\item $r_{m} < \delta$,
		\item for all $i\in \{1, 2,.., m - 1\}$ and a proper $\xi > 1$,
		\begin{equation}
		\label{equation:xi}
			r_{i + 1} > \xi r_{i}
		\end{equation}
	\end{enumerate}
	and collect them in $Q_{r} := \{r_{i}\in \mathbb{R}_{>0}\}_{i = 1, 2,..., m}$. 
	
	Since both $Q_{h}$ and $Q_{r}$ are constructed, for each $i\in \{1, 2,.., m\}$, a ball centered at the $c_{i} := \phi(t_{i}, j_{i})|_{(t_{i}, j_{i})\in Q_{h}}$ with radius $r_{i}\in Q_{r}$ is constructed. 
	
	After this sequence of balls along the motion plan is constructed, we can show that
	\begin{enumerate}
		\item the probability that the event $E_{init}$ occurs is positive, where $E_{init}$ denotes the event that the function call $\mathcal{T}.init$ adds a vertex $v_{0}$ to the search tree $\mathcal{T} = (V, E)$ such that $\overline{x}_{v_{0}}\in c_{1} + \frac{2r_{1}}{5}\mathbb{B}$, namely, one vertex is initialized to the search tree that is within the first ball in the sequence;
		\item  for each $i\in \{1, 2,.., m - 1\}$, the probability that each of events $E_{regime}$, $E_{near}$, $E_{fg}$, and $E_{new}$ occurs is positive, where
		\begin{enumerate}
			\item $E_{regime}$ denotes the event that 1) $r\leq p_{n}$ if $t_{i + 1} = t_{i} + \Delta t'$ and 2) $r > p_{n}$ if $j_{i + 1} = j_{i} + 1$ in the random selection in Line 3 of Algorithm \ref{algo:hybridRRT}, namely, the same regime as the continuous/discrete piece between the hybrid time instances $(t_{i}, j_{i})$ and $(t_{i + 1}, j_{i + 1})$ is selected;
			\item $E_{near}$ denotes the event that $v_{cur}$ returned by the function call $nearest\_neighbor$ is such that $\overline{x}_{v_{cur}}\in c_{i} + r_{i}\mathbb{B}$, namely, the nearest vertex selected by the function call $nearest\_neighbor$ is within the $i$-th circle.
			\item $E_{fg}$ denotes the event that 1) the simulator of continuous dynamics as (\ref{newstate:flow}) is chosen to compute $\psi_{new}$ and $x_{new}$ by the function call $new\_state$ if  $t_{i + 1} = t_{i} + \Delta t'$ and 2) the simulator of discrete dynamics as (\ref{newstate:jump}) is chosen to compute $\psi_{new}$ and $x_{new}$ by the function call $new_state$ if  $j_{i + 1} = j_{i} + 1$.
			\item $E_{new}$ denotes the event that $(\phi_{new}, u_{new})$ computed by the function call $new\_state$ is such that 
			\begin{enumerate}[label= $C{\arabic*}$), leftmargin=*]
				\item $x_{new} \in c_{i + 1}  +\frac{2r_{i + 1}}{5}\mathbb{B}$,
				\item $\phi_{new}$ and $\phi_{i}$ are $(\tilde{\tau}_{i}, r_{i + 1})$-close, where $\tilde{\tau}_{i} = \max (T_{i} + J_{i}, T_{new} + J_{new})$, $(T_{new}, J_{new}) = \max \dom \psi_{new}$, and $(T_{i}, J_{i}) = \max \dom \psi_{i}$,
			\end{enumerate} namely, the function call $new\_state$ extends the search tree to the $i+1$-th circle and the generated solution pair is close to the motion plan.
		\end{enumerate}
	\end{enumerate}
	
	Next, we show that the probabilities of all the events listed above are positive. Since Assumption \ref{assumption:uniformsample} is assumed, the probability that $E_{init}$ occurs for one random selection made by the function call $\mathcal{T}.init$, denoted $p_{0}$, is 
	$$
	p_{0} = \zeta_{n}\frac{(\frac{2r_{1}}{5})^{n}}{\mu(X_{0})} > 0.
	$$
	Suppose the random selections are made for $l$ times in the function call $\mathcal{T}.init$, then the probability that $E_{init}$ occurs at least once in $\mathcal{T}.init$ is
	\begin{equation}
		Pr(E_{init}) = 1 - (1 - p_{0})^{l} > 0.
	\end{equation}
	
	Next, we show that $Pr(E_{regime}) > 0$. Since Assumption \ref{assumption:uniformsample} is assumed, it is clear that
	\begin{equation}
			Pr(E_{regime}) = \left\{\begin{aligned}
				&p_{n}\quad &\text{ if } t_{i + 1} = t_{i} + \Delta t'\\
				&1 - p_{n} \quad &\text{ if } j_{i + 1} = j_{i} + 1
			\end{aligned}\right.
	\end{equation}
	Since $p_{n} \in (0, 1)$, then we have $Pr(E_{regime}) > 0$.
	
	The probability that $E_{near}$ occurs is proved to be positive by Lemma \ref{lemma:nearestvertex}. 
	
	The probability that $E_{fg}$ occurs is characterized by $p_{fg}$ in the function call $new\_state$ such that 
	\begin{equation}
		Pr(E_{fg}) = \left\{
		\begin{aligned}
			&p_{fg} &\text{ if } t_{i + 1} = t_{i} + \Delta t'\\
			&1 - p_{fg} \quad &\text{ if } j_{i + 1} = j_{i} + 1
		\end{aligned}
		\right.
	\end{equation}
	Since $p_{fg} \in (0, 1)$, then we have $Pr(E_{fg}) \geq \max \{p_{fg}, 1 - p_{fg}\} > 0$.
	
	The probability that $E_{new}$ occurs is positive because of
	Lemma \ref{lemma:pccontinuouslowerbound} and Lemma  \ref{lemma:pcdiscretelowerbound} under the condition that $Q_{h}$ and $Q_{r}$ are properly constructed, which is discussed in details later.
	
	 This repeated process can be viewed as $k$ Bernoulli trials with success probability, denoted $p$, that all of the events $E_{regime}$, $E_{near}$, $E_{fg}$, and $E_{new}$ occur where 
	 \begin{equation}
	 \label{equation:p}
	 	p = Pr(E_{regime}) \times Pr(E_{near})\times Pr(E_{fg})\times Pr(E_{new})>0.
	 \end{equation} The planning problem can be solved if 
	\begin{enumerate}
		\item $E_{init}$ occurs;
		\item the Bernoulli trials have $m - 1$ success outcomes where a success outcome means that all of $E_{regime}$, $E_{near}$, $E_{fg}$, and $E_{new}$ occur in the same iteration.
	\end{enumerate}  
The key of this proof is to show that it is feasible to properly construct the sequences $Q_{h}$ and $Q_{r}$ such that Lemma \ref{lemma:nearestvertex}, Lemma \ref{lemma:pccontinuouslowerbound}, and Lemma \ref{lemma:pcdiscretelowerbound} can be utilized to provide positive lower bounds over $Pr(E_{near})$ and $Pr(E_{new})$, which is discussed next.

	
	The lower bound provided by Lemma \ref{lemma:pccontinuouslowerbound} in (\ref{equation:lemmaflow}), denoted $p_{c}$, is
	\begin{equation}
	\label{equation:pc}
		p_{c} := p_{t}\frac{\zeta_{n} \left(\max \left\{\min \left\{\frac{\frac{\kappa_{2}\delta}{2} - \epsilon -\exp(K_{x}^{f}\tau)\kappa_{1}\delta}{K^{f}_{u}\tau \exp(K_{x}^{f}\tau)}, \delta\right\}, 0\right\}\right)^{m}}{\mu(U_{C})}.
	\end{equation}
	Note that $p_{c}$ is positive if
	\begin{equation}
	\label{equation:pccontinuousinequality}
		\frac{\kappa_{2}\delta}{2} - \epsilon -\exp(K_{x}^{f}\tau)\kappa_{1}\delta = (\frac{\kappa_{2}}{2} - \exp(K_{x}^{f}\tau)\kappa_{1})\delta - \epsilon > 0
	\end{equation} 
	where $\tau > 0$ and $\epsilon > 0$ come from Lemma \ref{lemma:pccontinuouslowerbound}.
	If 
	\begin{equation}
	\label{equation:kappa1}
	\kappa_{2} > 2\kappa_{1}.
	\end{equation}
	then, we can set the values for $\tau$ and $\epsilon$ in the following steps such that (\ref{equation:pccontinuousinequality}) holds.
	\begin{enumerate}[label= \textbf{Step} \arabic*)]
		\item Solve the inequations
		$$
		\begin{aligned}
		&\frac{\kappa_{2}}{2} - \exp(K_{x}^{f}\tau)\kappa_{1} > 0\\
		&\tau > 0
		\end{aligned}
		$$ for $\tau$. The analytical solution is
		\begin{equation}
		\label{equation:tauinequation}
			\begin{aligned}
			\tau &< \frac{\ln(\frac{\kappa_{2}}{2\kappa_{1}})}{K_{x}^{f}}\\
			\tau &> 0.
			\end{aligned}
		\end{equation} Note that the upper bound for $\tau$ is determined by the ratio $\frac{\kappa_{2}}{\kappa_{1}}$.
		\item Once the value for $\tau$ in \textbf{Step} 1) is obtained, solve the inequations
		$$
		\begin{aligned}
		&\epsilon < (\frac{\kappa_{2}}{2} - \exp(K_{x}^{f}\tau)\kappa_{1})\delta\\
		&\epsilon > 0
		\end{aligned}
		$$ for $\epsilon$. 
	\end{enumerate}
Note that the positive $\tau$ can be set arbitrarily close to $0$ in \textbf{Step} 1. So we can also impose the following conditions on $\tau$: 
\begin{enumerate}
	\item $\tau\leq T_{m}$ where $T_{m}$ is from Assumption \ref{assumption:inputlibrary};
	\item there exists $l\in \mathbb{N}_{>0}$ such that $l\cdot \tau = \Delta t$, where $\Delta t$ is the resolution of the piece-wise constant in Assumption \ref{assumption:mpproblem}.
\end{enumerate}

Therefore, if (\ref{equation:kappa1}) is satisfied, we can find a pair of $\tau > 0$ and $\epsilon > 0$ such that 
	\begin{enumerate}
		\item $\tau\leq T_{m}$;
		\item there exists $l\in \mathbb{N}_{>0}$ such that $l\cdot \tau = \Delta t$;
		\item (\ref{equation:pccontinuousinequality}) holds.
	\end{enumerate}
	
	The lower bound provided by Lemma \ref{lemma:pcdiscretelowerbound} in (\ref{equation:lemmadiscrete}), denoted $p_{d}$, is
	\begin{equation}
	\label{equation:pd}
		p_{d} :=  \frac{\zeta_{n} \left(\max \{\min \{\frac{(\kappa_{2} - K^{g}_{x}\kappa_{1})\delta}{K^{g}_{u}}, \delta\}, 0\}\right)^{m}}{\mu(U_{D})}.
	\end{equation}
	Note that $p_{d}$ is positive if
	\begin{equation}
	\label{equation:kappa2}
		\kappa_{2} > K_{x}^{g}\kappa_{1}.
	\end{equation}
	
	To satisfy both (\ref{equation:kappa1}) and (\ref{equation:kappa2}), the following is required:
	\begin{equation}
	\label{equation:kappa}
		\kappa_{2} >\xi\kappa_{1}
	\end{equation} where 
	\begin{equation}
	\label{equation:greater2}
		\xi = \max(K_{x}^{g}, 2) \geq 2.
	\end{equation} Hence, the value of $\xi$ in (\ref{equation:xi}) is determined.
	
	Next, we are ready to construct $Q_{h}$ and $Q_{r}$. For any $i\in \{1, 2,..., m -1\}$, set $r_{i} = \kappa_{1}\delta$ and $\frac{2r_{i + 1}}{5} = \kappa_{2}\delta$ in Lemma \ref{lemma:pccontinuouslowerbound} and Lemma \ref{lemma:pcdiscretelowerbound}. 
	Therefore, (\ref{equation:kappa}) implies that
	\begin{equation}
	\label{equation:r1r2}
		\frac{2}{5\xi}r_{i + 1} >  r_{i}.
	\end{equation}
	Arbitrarily pick a $\overline{\xi} \in (0, \frac{2}{5\xi})$. Because $\xi \geq 2$ in (\ref{equation:greater2}), we have  $\overline{\xi} < \frac{1}{5}$. If
	\begin{equation}
	\label{equation:k1k2ratio}
		\overline{\xi}r_{i + 1} = r_{i}
	\end{equation} holds
	for all $i\in \{1, 2, ..., m - 1\}$, both (\ref{equation:r1r2}) and (\ref{equation:kappa}) are satisfied. The relationship (\ref{equation:k1k2ratio}) implies a geometric progression of $Q_{r}$ with common ratio as $\frac{1}{\overline{\xi}}$.
	
	Next, we are going to determine $m$. Since $\overline{\xi}$ is determined, if (\ref{equation:k1k2ratio}) holds, we have 
	$$
		\frac{\kappa_{2}}{\kappa_{1}} = \frac{\frac{r_{i + 1}}{\delta}}{\frac{r_{i}}{\delta}} = \frac{r_{i + 1}}{r_{i}} = \frac{1}{\overline{\xi}}.
	$$ Hence, we have the upper bound for $\tau$ in (\ref{equation:tauinequation}) as
	$$
	\frac{\ln(\frac{\kappa_{2}}{2\kappa_{1}})}{K_{x}^{f}} = \frac{\ln(\frac{1}{2\overline{\xi}})}{K_{x}^{f}}.
	$$ After the upper bound of $\tau$ is determined, the value of $\tau$ can be from the interval $(0,  \frac{\ln(\frac{1}{2\overline{\xi}})}{K_{x}^{f}})$.
	
	Then this fixed $\tau$ is used in the construction of $Q_{h}$ as the value of $\Delta t'$ in (\ref{equation:deltatprime}). As a result, the value of $m$, which denotes the number of hybrid time instances in $Q_{h}$, is such that 
	$m = \frac{T}{\tau} + J + 1$ where $(T, J) = \max \dom \psi$.
	
	
	
	Since we determine $m$ in the construction of $Q_{h}$, we are ready to construct the sequence of radius $Q_{r}$. For all $i\in \{1, 2, ...,m\}$, we set
	$$r_{i} = \overline{\xi}^{m - i + 1}\delta.$$
	In this way, the requirement in (\ref{equation:k1k2ratio}) is met.
	
	Until this point, we show that constructing $Q_{h}$ and $Q_{r}$ is feasible. As a result, we show the existence the sequence of ball along the motion plan such that the probability that the events $E_{regime}$, $E_{near}$, and $E_{new}$ occur is positive.

Next, we provide an upper bound in the form of $a e^{-bk}$ for the probability that HyRRT fails to solve the motion planning problem. 
Note that solving the motion planning problem can be viewed as $m - 1$ successful outcomes in the Bernoulli process with success probability $p > 0$; see (\ref{equation:p}). Let $X_{k}$ denote the number of successes in $k$ trials. Similar with the proof for Theorem 1 in \cite{kleinbort2018probabilistic}, we have
\begin{equation}
\begin{aligned}
Pr(X_{k} < m - 1) &\leq \sum_{i = 0}^{m - 2} \left(\begin{matrix}
k\\
i
\end{matrix}\right)p^{i}(1- p)^{k - i}\\
&\leq  \sum_{i = 0}^{m - 2} \left(\begin{matrix}
k\\
m - 2
\end{matrix}\right)p^{i}(1- p)^{k - i}\\
&\leq  \left(\begin{matrix}
k\\
m - 2
\end{matrix}\right)\sum_{i = 0}^{m - 2} (1- p)^{k}\\
&\leq  \left(\begin{matrix}
k\\
m - 2
\end{matrix}\right)\sum_{i = 0}^{m - 2} (e^{-p})^{k} 
=  \left(\begin{matrix}
k\\
m - 2
\end{matrix}\right) (m - 1)e^{-pk}\\
&= \frac{\prod_{i = k - m}^{k}i}{(k - 1)!}(m - 1)e^{-pk}
\leq \frac{1}{(m - 2)!}e^{-pk}\\
\end{aligned}
\end{equation}
where $m< k$, $p< 1/2$ and $(1 - p)\leq e^{-p}$.
As $p$ and $m$ are fixed and independent of $k$, then $ \frac{1}{(m - 2)!}e^{-pk}$ decays to $0$ exponentially with $k$. 

Next, we prove that $\phi$ and $\phi'$ are $(\tilde{\tau}, \delta)$-close. We know that $\phi'$ is the concatenation of $\psi_{new}$'s computed by the function call $new\_state$ in $m - 1$ successful trials. For $i\in \{1, 2, ..., m -1\}$, let $\psi'_{i} = (\phi'_{i}, u'_{i})$ denote the $\psi_{new} = (\phi_{new}, u_{new})$ computed in $i$th successful trial. Then, $\phi'$ is constructed as follows:
$$
	\phi' = \phi'_{1}|\phi'_{2}|...|\phi'_{m -1}.
$$

For $i\in \{1, 2, ..., m - 1\}$, denote the truncation of the motion plan $\psi$ between $(t_{i}, j_{i})$ and $(t_{i + 1}, j_{i + 1})$ following a translation by $(t_{i}, j_{i})$ as $\psi_{i} = (\phi_{i}, u_{i})$. Note that for all $i \in \{1, 2, ..., m - 1\}$, $\psi_{i}$ is either purely continuous or purely discrete with a single jump because of the requirements 3a and 3b of $Q_{h}$.
 From C2) of the event $E_{new}$, when $E_{new}$ occurs, for all $i\in \{1, 2,..., m - 1\}$, $\phi_{i}$ and $\phi'_{i}$ are $(\tilde{\tau}_{i}, r_{i + 1})$-close where $\tilde{\tau}_{i} = \max\{T_{i} + J_{i}, T'_{i} + J'_{i}\}$, $(T_{i}, J_{i}) = \max \dom \psi_{i}$ and $(T'_{i}, J'_{i}) = \max \dom \psi'_{i}$. Because of Proposition \ref{proposition:concatenateclose}, we have that the concatenation result $\phi'$ and $\phi$ are $(\sum_{i = 1}^{m-1}\tilde{\tau}_{i}, \sum_{i = 2}^{m}r_{i})$-close.

Note that 
$$
\begin{aligned}
\sum_{i = 1}^{m-1}\tilde{\tau}_{i} &= \sum_{i = 1}^{m - 1} \max(T_{i} + J_{i}, T'_{i} + J'_{i})\\
&\geq \max (\sum_{i = 1}^{m - 1} T_{i} + J_{i} , \sum_{i = 1}^{m - 1} T'_{i} + J'_{i} ) \\
&= \max (\sum_{i = 1}^{m - 1} T_{i} + \sum_{i = 1}^{m - 1}J_{i} , \sum_{i = 1}^{m - 1} T'_{i} + \sum_{i = 1}^{m - 1}J'_{i} ) \\
&= \max(T+ J, T' + J') =	\tilde{\tau}
\end{aligned}
$$ and
$$
	\begin{aligned}
	 \sum_{i = 2}^{m}r_{i} = \sum_{i = 2}^{m} \overline{\xi}^{m - i + 1}\delta = \frac{(1 - \overline{\xi}^{m - 1})\overline{\xi}}{1 - \overline{\xi}}\delta.
	\end{aligned}
$$
Note that $\overline{\xi} \in (0, \frac{1}{5})$, then $\frac{(1 - \overline{\xi}^{m - 1})\overline{\xi}}{1 - \overline{\xi}} < 1$ and therefore, $\frac{(1 - \overline{\xi}^{m - 1})\overline{\xi}}{1 - \overline{\xi}}\delta < \delta$.

Therefore, the fact that $\phi'$  and $\phi$ are $(\sum_{i = 1}^{m-1}\tilde{\tau}_{i}, \sum_{i = 1}^{m-1}r_{i})$-close implies that $\phi'$  and $\phi$ are $(\tilde{\tau}, \delta)$-close. 
\end{proof}
}
\subsection{Inflated Hybrid System and Main Result}\label{section:inflatedsystem}
In the probabilistic completeness result in \cite[Theorem 2]{kleinbort2018probabilistic}, a motion plan with positive clearance is assumed to exist. However, such assumption is restrictive for hybrid systems. Indeed, if the motion plan reaches the boundary of the flow set or of the jump set, then the motion plan has no clearance\ifbool{conf}{}{; see Definition \ref{definition:clearance}}.  To overcome this issue and to assure that HyRRT is probabilistically complete, the hybrid system $\mathcal{H} = (C, f, D, g)$ is modified as follows. \ifbool{conf}{}{Figure \ref{fig:motionplannoclearance} shows a motion plan to the sample motion planning problem for the actuated bouncing ball system without clearance.
	\begin{figure}[htbp]
		\centering
		\subfigure[\label{fig:motionplannoclearance}]{\includegraphics[width = 0.23\textwidth]{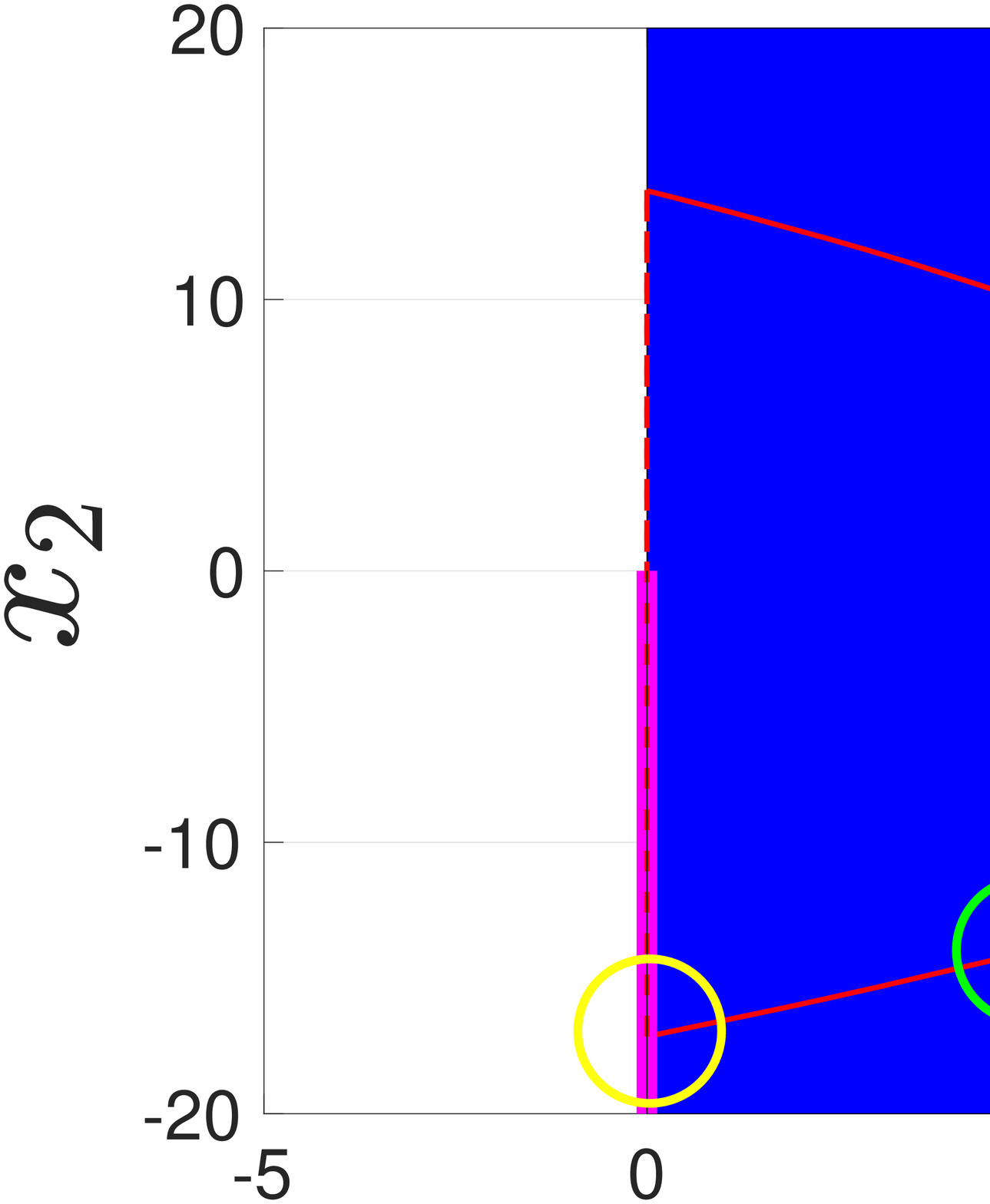}}
		\subfigure[\label{fig:motionplannoclearance_inflated}]{\includegraphics[width = 0.23\textwidth]{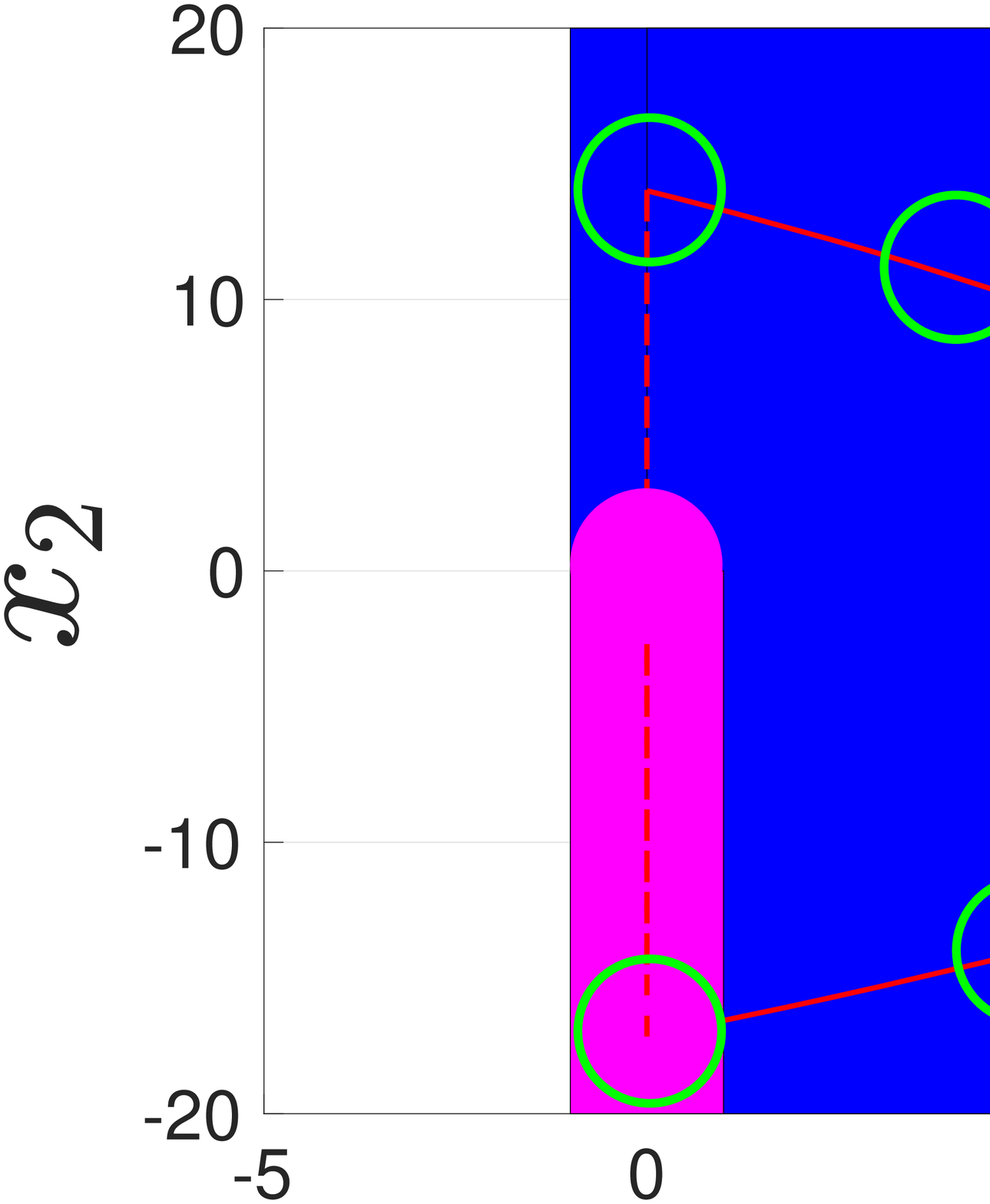}}
		\caption{Figure \ref{fig:motionplannoclearance} shows an sample motion plan for the bouncing ball system without clearance. The red trajectory shows the motion plan. The blue region denotes the projection of flow set on to the state space. The circles denote the boundaries of the balls along the motion plan. Note that the ball surrounded by the yellow circle is not within the  flow set. In this case, no clearance $\delta$ exists. After the inflation, the existence of the green circles along the motion plan with radius $\delta$, as shown in Figure \ref{fig:motionplannoclearance_inflated}, implies the clearance $\delta$ of the motion plan.}
\end{figure}}
\begin{definition}($\delta$-inflation of hybrid system)
	\label{definition:inflation}
	Given a hybrid system $\mathcal{H} = (C, f, D, g)$ and $\delta > 0$, the $\delta$-inflation of the hybrid system $\mathcal{H}$, denoted $\mathcal{H}_{\delta}$, is given by
	\begin{equation}
		\mathcal{H}_{\delta}: \left\{              
		\begin{aligned}               
		\dot{x} & = f_{\delta}(x, u)     &(x, u)\in C_{\delta}\\                
		x^{+} & =  g_{\delta}(x, u)      &(x, u)\in D_{\delta}\\                
		\end{aligned}   \right. 
		\label{model:inflatedhybridsystem}
		\end{equation}
	 where
	\begin{enumerate}[label=\arabic*)]
		\item $C_{\delta} := \{(x, u)\in \mathbb{R}^{n}\times \mathbb{R}^{m}: \exists (y,  v)\in C \text{ such that } x\in y + \delta \mathbb{B}, u\in v+ \delta \mathbb{B}\}$,
		\item $f_{\delta}(x, u) := f(x, u)\quad \forall (x, u)\in C_{\delta}$,
		\item $D_{\delta} := \{(x, u)\in \mathbb{R}^{n}\times \mathbb{R}^{m}: \exists (y,  v)\in D \text{ such that } x\in y + \delta \mathbb{B}, u\in v+ \delta \mathbb{B}\}$,
		\item $g_{\delta}(x, u) := g(x, u)\quad \forall (x, u)\in D_{\delta}$.
	\end{enumerate}
\end{definition}

\ifbool{conf}{}{The following gives an example of the $\delta$-inflation of the actuatedd bouncing ball system.
\begin{example}
	(The actuated bouncing ball system in Example \ref{example:bouncingball}, revisited) The $\delta$-inflation of the actuated bouncing ball system is constructed as follows. 
	\begin{enumerate}[label=\arabic*)]
		\item From (\ref{model:bouncingballflowset}), the flow set $C_{\delta}$ is given by
		$
		\label{model:bouncingballflowsetinflated}
			C_{\delta} := \{(x, u)\in \mathbb{R}^{2}\times \mathbb{R}: x_{1} \geq -\delta\}.
		$
		\item From (\ref{model:bouncingballflow}), the flow map $f_{\delta}$ is given by
		$
		\label{model:bouncingballflowmapinflated}
			f_{\delta}(x, u) := f(x, u) =  \left[ \begin{matrix}
			x_{2} \\
			-\gamma
			\end{matrix}\right], \forall (x, u)\in C_{\delta}.
		$
		\item From (\ref{model:bouncingballjumpset}), the jump set $D_{\delta}$ is given by
		$
		\label{model:bouncingballjumpsetinflated}
		D_{\delta} := \{(x, u)\in \mathbb{R}^{2}\times \mathbb{R}: -\delta \leq x_{1}\leq \delta, x_{2} \leq 0, u\geq -\delta  \}\cup \{(x, u)\in \mathbb{R}^{2}\times \mathbb{R}: x_{1}^{2} + x_{2}^2 \leq \delta^{2}, u\geq -\delta  \}.
		$
		\item From (\ref{conservationofmomentum}), the jump map $g_{\delta}$ is given by
		$
		\label{model:bouncingballjumpmapinflated}
			g_{\delta}(x, u) := g(x, u) = \left[ \begin{matrix}
			x_{1} \\
			-\lambda x_{2}+u
			\end{matrix}\right], \forall (x, u)\in D_{\delta}
			$
	\end{enumerate}
 Figure \ref{fig:motionplannoclearance_inflated} shows that after  inflation, the clearance of the motion plan for the bouncing ball is equal to $\delta$.
\end{example}}
\ifbool{conf}{}{
	Next we show that a motion plan to the original motion planning problem is also a motion plan to the motion planning problem for its $\delta$-inflation.
	\begin{proposition}
	\label{proposition:motioninflated}
	Given a motion planning problem $\mathcal{P} = (X_{0}, X_{f}, X_{u}, (C, f, D, g))$ in Problem \ref{problem:motionplanning}, if $\psi$ is a motion plan to $\mathcal{P}$, then $\psi$ is also a motion plan to the motion planning problem $\mathcal{P}_{\delta} = (X_{0}, X_{f}, X_{u},  (C_{\delta}, f_{\delta}, D_{\delta}, g_{\delta}))$, where $(C_{\delta}, f_{\delta}, D_{\delta}, g_{\delta})$ is the $\delta$-inflation of hybrid system of $(C, f, D, g)$ for any $\delta > 0$.
\end{proposition}
\begin{proof}
	Note that the initial state set  $X_{0}$, final state set $X_{f}$, and unsafe set $X_{u}$ are the same in both $\mathcal{P}$ and $\mathcal{P}_{\delta}$. Therefore, it free to have $\psi$ satisfying items 1), 3), and 4) in Problem \ref{problem:motionplanning} for $\mathcal{P}_{\delta}$. We only need to show that $\psi$ is a solution pair to $(C_{\delta}, f_{\delta}, D_{\delta}, g_{\delta})$, as is defined in Definition \ref{definition:solution}. 
	
	Since $\psi = (\phi, u)$ is a motion plan to $\mathcal{P}$, then $\psi$ is a solution pair to $(C, f, D, g)$. Therefore, items 1a and 1c in Definition \ref{definition:solution} are satisfied for free for $\mathcal{P}_{\delta}$ because these items are not related to $(C, f, D, g)$. 
	
	Next, we show that $\psi = (\phi, u)$ satisfies the rest items in Definition \ref{definition:solution}. Note that $C\subset C_{\delta}$, $D\subset D_{\delta}$. Therefore, 
	\begin{equation}
		(\phi(0,0), u(0,0))\in \overline{C}\cup D \subset \overline{C_{\delta}}\cup D_{\delta}.
	\end{equation}
	
	For all $j\in \mathbb{N}$ such that $I^{j}= \{t: (t, j)\in \dom (\phi, u)\}$ has nonempty interior, 
	\begin{enumerate}
		\item $(\phi(t, j),u(t, j))\in C \subset C_{\delta}$ for all $t\in \interior I^j$. Then item 1b in Definition \ref{definition:solution} is satisfied.
		\item for almost all $t\in I^j$,
		\begin{equation}
		\dot{\phi}(t,j) = f(\phi(t,j), u(t,j)) = f_{\delta}(\phi(t,j), u(t,j)).
		\end{equation}
		Then item 1d in Definition \ref{definition:solution} is satisfied.
	\end{enumerate}
	
	For all $(t,j)\in \dom (\phi, u)$ such that $(t,j + 1)\in \dom (\phi, u)$,
	\begin{equation}
	\begin{aligned}
	(\phi(t, j), u(t, j))&\in D\subset D_{\delta}\\
	\phi(t,j+ 1) &= g(\phi(t,j), u(t, j)) = g_{\delta}(\phi(t,j), u(t, j))
	\end{aligned}
	\end{equation}
	Then item 2) in Definition \ref{definition:solution} is satisfied. 
	
	Since $\psi$ satisfies all the items in Definition \ref{definition:solution} for data $(C_{\delta}, f_{\delta}, D_{\delta}, g_{\delta})$, then $\psi$ is a solution pair to $(C_{\delta}, f_{\delta}, D_{\delta}, g_{\delta})$ and therefore, is a motion plan to $\mathcal{P}_{\delta}$.
\end{proof}

}\ifbool{conf}{}{
	Next we show that the existing motion plan has positive clearance for the motion planning problem for the $\delta$-inflation of the original hybrid system.
	\begin{lemma}
	\label{lemma:motionplanintlated}
	Let $\psi$ be a motion plan to the motion planning problem $\mathcal{P} = (X_{0}, X_{f}, X_{u}, (C, f, D, g))$ formulated as Problem \ref{problem:motionplanning}. Suppose that there exists $\delta_{free} > 0$ such that for all $(t, j)\in \dom \psi$, $(\phi(t, j) + \delta_{free}\mathbb{B}, u(t, j) + \delta_{free}\mathbb{B}) \cap X_{u} =\emptyset$. Then for arbitrary $\delta > 0$, $\psi$ is a motion plan to the motion planning problem $\mathcal{P}_{\delta} = (X_{0}, X_{f}, X_{u}, (C_{\delta}, f_{\delta}, D_{\delta}, g_{\delta}))$ with clearance, denoted $\delta_{clear}$, such that $\delta_{clear} \geq \min (\delta_{free}, \delta)$, where $\mathcal{H}_{\delta} = (C_{\delta}, f_{\delta}, D_{\delta}, g_{\delta})$ is the $\delta$-inflation of $\mathcal{H} = (C, f, D, g)$.
\end{lemma}
\begin{proof}
	Proposition \ref{proposition:motioninflated} has shown that $\psi$ is a motion plan to the problem $\mathcal{P}_{\delta_{inflat}}$. We need to show that $\psi$ has clearance that is at least $\min (\delta_{free}, \delta)$.
	
	Since $\psi$ is a solution pair to $\mathcal{H} = (C, f, D, g)$, then for all $j\in \mathbb{N}$ such that $I^{j}= \{t: (t, j)\in \dom (\phi, u)\}$ has nonempty interior, 
	$$
	(\phi(t, j),u(t, j))\in C
	$$ for all $t\in \interior I^j.$
 	Because of item 1 of Definition \ref{definition:inflation}, then 
 	$$
 	(\phi(t, j) +\delta \mathbb{B}, u(t, j) + \delta \mathbb{B} )\subset C_{\delta}.
 	$$ Therefore, for all $\delta_{u}\in [0, \delta]$, 
 	$$(\phi(t, j) +\delta_{u} \mathbb{B}, u(t, j) + \delta_{u} \mathbb{B} )\subset (\phi(t, j) +\delta \mathbb{B}, u(t, j) + \delta \mathbb{B} )\subset C_{\delta}.
 	$$ Hence, item 1) in Definition \ref{definition:clearance} is satisfied for all $\delta_{u}\in [0, \delta]$. 
	
	Similarly, since $\psi$ is a solution pair to $\mathcal{H} = (C, f, D, g)$, for all $(t, j)\in \dom \psi $ such that $(t, j + 1)\in \dom \psi$, 
	$$(\phi(t, j), u(t, j))\in D.$$ 
	Because of item 3 of Definition \ref{definition:inflation}, then 
	$$
	(\phi(t, j) +\delta \mathbb{B}, u(t, j) + \delta \mathbb{B} )\subset D_{\delta}.
	$$ Therefore, for all $\delta_{u}\in [0, \delta]$, 
	$$
	(\phi(t, j) +\delta_{u} \mathbb{B}, u(t, j) + \delta_{u} \mathbb{B} )\subset (\phi(t, j) +\delta \mathbb{B}, u(t, j) + \delta \mathbb{B} )\subset D_{\delta}.$$ Hence, item 2) in Definition \ref{definition:clearance} is also satisfied for all $\delta_{u}\in [0, \delta]$. 
	
	Note that it is assumed that for all $(t, j)\in \dom \psi$, $(\phi(t, j) + \delta_{free}\mathbb{B}, u(t, j) + \delta_{free}\mathbb{B})\cap X_{u} =\emptyset$. Then, for all $\delta_{u} \in (0, \delta_{free}]$, since $$
	(\phi(t, j) + \delta_{u}\mathbb{B}, u(t, j) + \delta_{u}\mathbb{B})\subset (\phi(t, j) + \delta_{free}\mathbb{B}, u(t, j) + \delta_{free}\mathbb{B}),
	$$ then, $$
	(\phi(t, j) + \delta_{u}\mathbb{B}, u(t, j) + \delta_{u}\mathbb{B}) \cap X_{u} =\emptyset.
	$$
	Hence, item 3) in Definition \ref{definition:clearance} is satisfied for all $\delta_{u} \in [0, \delta_{free}]$. 
	
	In conclusion, for all $\delta_{u} \in [0, \delta]\cap [0, \delta_{free}] = [0, \min(\delta, \delta_{free})]$, all the items in Definition \ref{definition:clearance} are satisfied. Therefore, $\min(\delta, \delta_{free})$ is the lower bound over $\delta_{clear}$ and $\psi$ is a motion plan to $\mathcal{P}_{\delta}$ with clearance $\delta_{clear} \geq \min (\delta_{free}, \delta)$. 
\end{proof}
\begin{remark}
	Lemma \ref{lemma:motionplanintlated} shows that if there is a motion plan to the original motion planning problem for the hybrid systems, then it is guaranteed to be a motion plan with positive clearance for the motion planning problem for the inflation of the hybrid systems.
\end{remark}
}

Note that any solution to $\mathcal{H}$ in (\ref{model:generalhybridsystem}) is a solution to its inflation in (\ref{model:inflatedhybridsystem}). The clearance property in Definition \ref{definition:clearance} is satisfied for free since items 1) and 2) therein are satisfied by constructing $C_{\delta}$ and $D_{\delta}$, and item 3) therein is satisfied by item 3) in Assumption \ref{assumption:mpproblem}. Next, we state our main result.     
\begin{theorem}
	\label{theorem:inflatedpc}
	Given a motion planning problem $\mathcal{P} = (X_{0}, X_{f}, X_{u}, (C, f, D, g))$, suppose that Assumptions \ref{assumption:inputlibrary}, \ref{assumption:uniformsample}, \ref{assumption:flowlipschitz}, and \ref{assumption:pcjumpmap} are satisfied and that there exists a motion plan $(\phi, u)$ to $\mathcal{P}$ satisfying Assumption \ref{assumption:mpproblem} for some $\delta' > 0$. When HyRRT is used to solve the problem $\mathcal{P}_{\delta} = (X_{0}, X_{f}, X_{u},(C_{\delta}, f_{\delta}, D_{\delta}, g_{\delta}))$, where, for some $\delta > 0$, $(C_{\delta}, f_{\delta}, D_{\delta}, g_{\delta})$ denotes the $\delta$-inflation of $(C, f, D, g)$ in (\ref{model:inflatedhybridsystem}), the probability that HyRRT fails to find a motion plan $\psi' = (\phi', u')$ to $\mathcal{P}_{\delta}$ such that $\phi'$ is $(\tilde{\tau}, \tilde{\delta})$-close to $\phi$ after $k$ iterations is at most $a\exp(-bk)$, for some constant $a, b\in \mathbb{R}_{> 0}$, where $(T, J) = \max \dom \psi$, $(T', J') = \max \dom \psi'$, $\tilde{\tau} = \max\{T+ J, T' + J'\}$, and $\tilde{\delta} = \min \{\delta, \delta'\}$.
\end{theorem}
\ifbool{conf}{}{
\begin{proof}
	Since $\psi$ is a motion plan to $\mathcal{P}$ and there exists $\delta' > 0$ such that for all $(t, j)\in \dom \psi$, $(\phi(t, j) + \delta' \mathbb{B}, u(t, j) + \delta' \mathbb{B}) \cap X_{u} =\emptyset$, according to Lemma \ref{lemma:motionplanintlated}, $\psi$ is a motion plan to $\mathcal{P}_{\delta}$ with clearance $\delta_{clear} \geq \min(\delta', \delta) >0$. Then, according to Theorem \ref{theorem:pc}, the probability that HyRRT fails to find $\psi'$ is at most $ae^{-bk}$ and the generated $\phi'$ is $(\tilde{\tau}, \tilde{\delta})$-close to $\phi$ where $\tilde{\tau} = \max(T+ J, T' + J')$ and $\tilde{\delta} = \min (\delta_{clear}, \delta')  =  \min (\min(\delta', \delta), \delta') = \min(\delta, \delta')$.
\end{proof}

}
\ifbool{conf}{
}{
\begin{remark}
	Theorem \ref{theorem:inflatedpc} renders HyRRT algorithm probabilistically complete for a motion plan without positive clearance. Instead of the original motion planning problem, the motion planning problem for its inflation of hybrid system is fed to HyRRT algorithm and hence, the existence of the clearance is guaranteed by Lemma \ref{lemma:motionplanintlated}. With all other conditions in Theorem \ref{theorem:pc} satisfied, HyRRT algorithm is probabilistically complete for the motion planning problem for the inflated hybrid system and hence, is guaranteed to output a solution that is close to the motion plan to the original motion planning problem. 
\end{remark}
}

\section{HyRRT Software Tool for Motion Planning for Hybrid Systems and Examples}
\label{section:illustration}
Algorithm \ref{algo:hybridRRT} leads to a software tool\footnote{Code at \href{https://github.com/HybridSystemsLab/hybridRRT}{https://github.com/HybridSystemsLab/hybridRRT}.} to solve the motion planning problems for hybrid systems. This software only requires the motion planning problem data $(X_{0}, X_{f}, X_{u}, (C, f, D, g))$, an input library $(\mathcal{U}_{C}, \mathcal{U}_{D})$, a tunable parameter $p_{n}\in (0, 1)$, an upper bound $K$ over the iteration number and two constraint sets $X_{c}$ and $X_{d}$. The tool is illustrated in \ifbool{conf}{Example \ref{example:biped}. We have successfully applied HyRRT to other hybrid systems, including the actuated bouncing ball and a point-mass robotics manipulator.}{Examples \ref{example:bouncingball} and \ref{example:biped}.}
\ifbool{conf}{
	\begin{example}(Walking robot system in Example \ref{example:biped}, revisited) \label{example:bipedillustration}
	\ifbool{conf}{}{The settings for HyRRT planner are given as follows.
	\begin{enumerate}[label=\arabic*)]
		\item The input library $(\mathcal{U}_{C}, \mathcal{U}_{D})$ is constructed as follows. In this illustration, 
		$U^{s}_{C}$ is constructed as 
		$
		U_{C}^{s} = \{-2.0, -1.0, 0.0 , 1.0, 2.0\}\times \{-2.0, -1.0, 0.0 , 1.0, 2.0\}$$\times$$
		\{-0.15, $$-0.0875, $$-0.0250,$ $ $$0.0375, $$0.10 \}.
		$ There are $125$ elements in $U_{C}^{s}$. For each $u^{s}\in U^{s}_{C}$, an input signal $[0, 0.2]\to \{u^{s}\}$ is constructed and added to $\mathcal{U}_{C}$. 
		
		In the biped system, since input has no effect on the jump, then $\mathcal{U}_{D}$ is constructed as $\{(0, 0, 0)\}$.
		\item The tolerance $\epsilon$ in (\ref{equation:tolerance}) is set to $0.3$ and $K$ is set as infinity. The parameter $p_{n}$ is set as $0.9$ to encourage the flow regime. 
	\end{enumerate} }\ifbool{conf}{The simulation result in Figure \ref{fig:illustrationbiped} with tolerance $\epsilon$ set to $0.3$ shows that HyRRT is able to solve the instance of motion planning problem for the walking robot. In this simulation, the constraint set $X_{c}$ is chosen as $\{(x, a)\in \mathbb{R}^{6}\times \mathbb{R}^{3}: h(x)\geq -s\}$ and $X_{d}$ as $\{(x, a)\in \mathbb{R}^{6}\times \mathbb{R}^{3}: h(x) = 0, \omega_{p} \geq -s\}$ with a tunable parameter $s$ set to $0$, $0.3$, $0.5$, $1$, and $2$, such that
	$C = X_{c}|_{s = 0} \subsetneq X_{c}|_{s = 0.3}\subsetneq X_{c}|_{s = 0.5}\subsetneq X_{c}|_{s = 1}\subsetneq X_{c}|_{s = 2}$ and $D = X_{d}|_{s = 0} \subsetneq X_{d}|_{s = 0.3}\subsetneq X_{d}|_{s = 0.5}\subsetneq X_{d}|_{s = 1}\subsetneq X_{d}|_{s = 2}$.}{}
	
	The simulation is implemented in MATLAB and processed by a $3.5$ GHz Intel Core i5 processor. The simulation takes $71.5/85.3/99.4/167.7/242.8$ seconds with $s$ set to $0/0.3/0.5/1.0/2.0$, respectively. The simulation takes at least $71.5$ seconds to finish. Compared with the forward/backward propagation algorithm based on breadth-first search which takes $1608.2$ seconds to solve the same problem, the improvement provided by the rapid exploration is significant: $95.5\%$ computation time improvement. It is also observed that as the sets $X_{c}$ and $X_{d}$ grow, HyRRT considers more vertices in solving Problems \ref{problem:nearestneighborflow} and \ref{problem:nearestneighborjump} leading to higher computation time.
\end{example}
	\ifbool{conf}{\vspace{-0.6cm}}{}
\ifbool{conf}{\begin{figure}[H]
		\centering
		\subfigure{\includegraphics[width=0.5\textwidth]{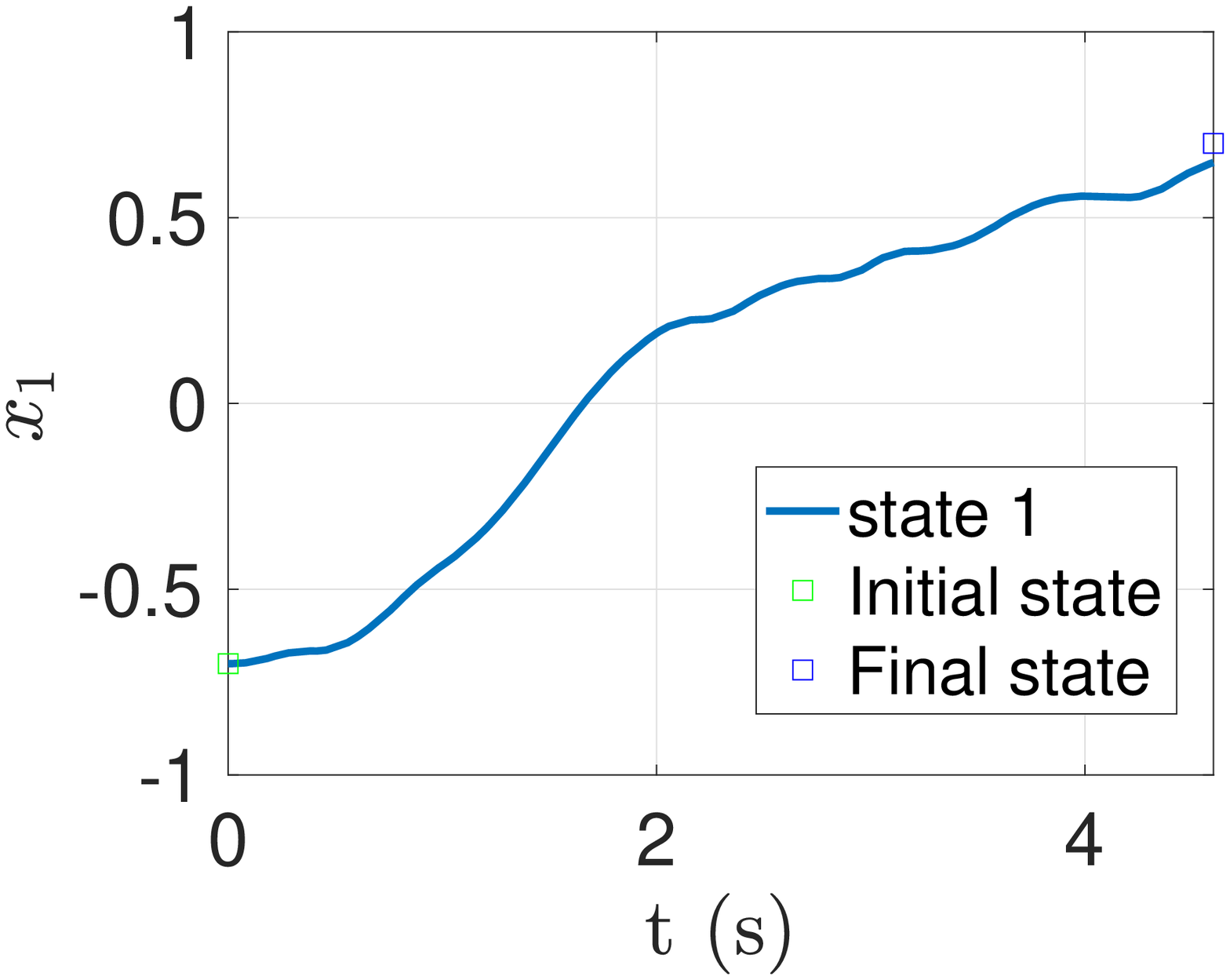}}\hspace{-0.36cm}
		\subfigure{\includegraphics[width=0.5\textwidth]{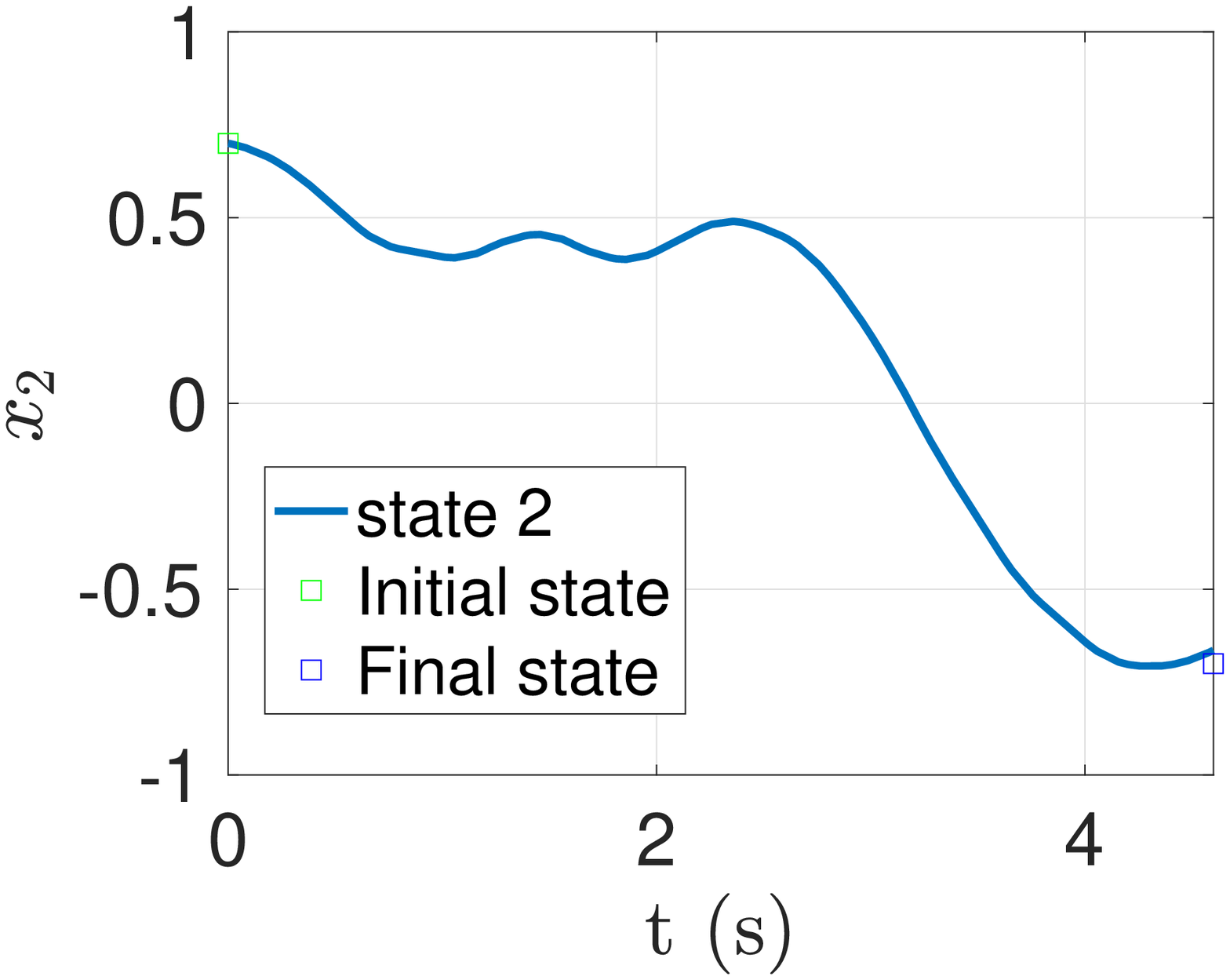}}\hspace{-0.36cm}
		\subfigure{\includegraphics[width=0.5\textwidth]{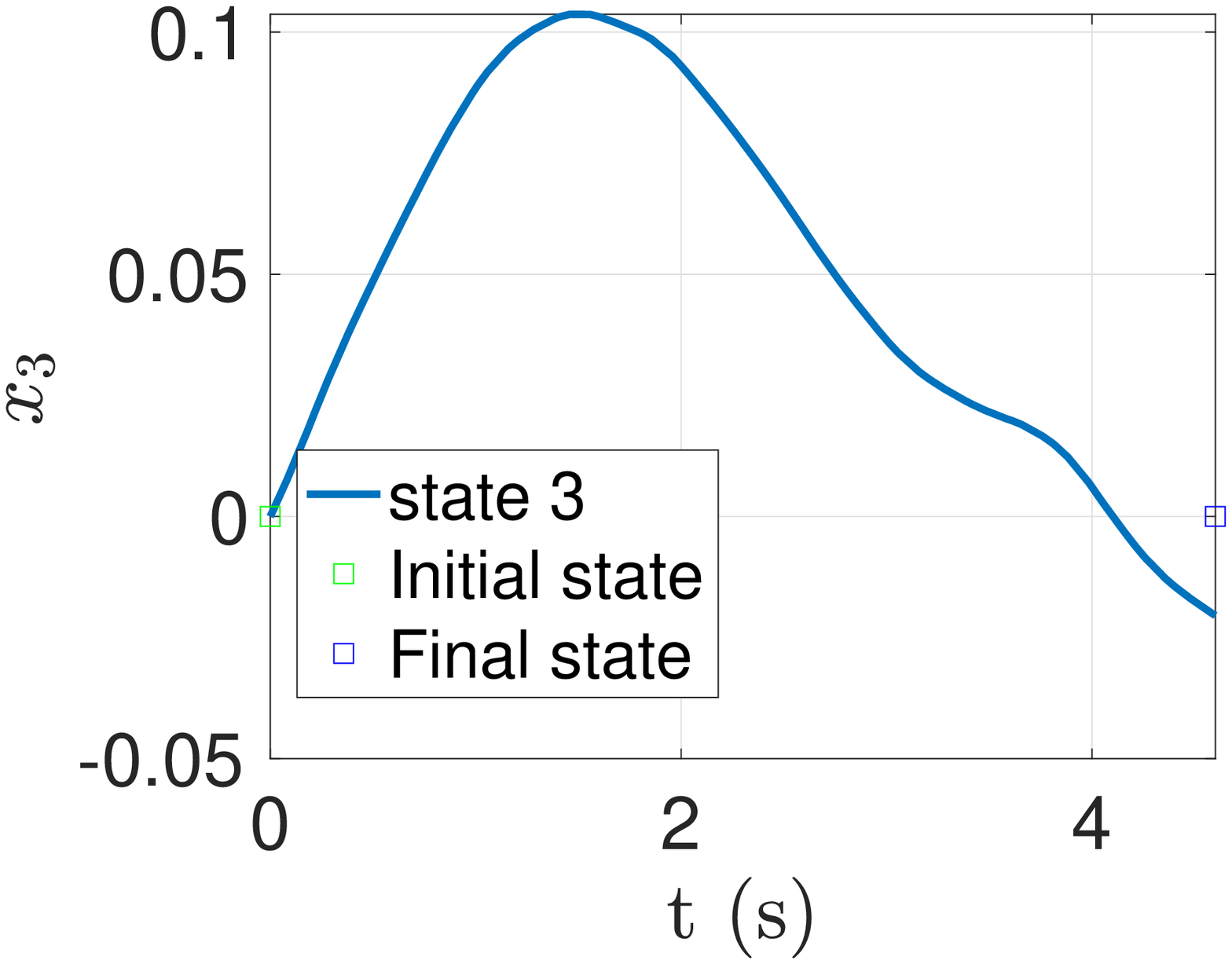}}
		\ifbool{conf}{\vspace{-0.5cm}}{}
		\caption{Selected state trajectories of the generated motion plan for the walking robot system. In each figure above, the green and blue squares denote the corresponding initial and final state components, respectively.}
		\label{fig:illustrationbiped}
\end{figure}}{
\begin{figure}[htbp]
	\centering
	\subfigure[The trajectory of $x_{1}$ component of the generated motion plan.]{\includegraphics[width=0.15\textwidth]{figures/bipedsystem/rrt/state1.eps}}
	\subfigure[The trajectory of $x_{2}$ component of the generated motion plan.]{\includegraphics[width=0.15\textwidth]{figures/bipedsystem/rrt/state2.eps}}
	\subfigure[The trajectory of $x_{3}$ component of the generated motion plan.]{\includegraphics[width=0.15\textwidth]{figures/bipedsystem/rrt/state3.eps}}
	\subfigure[The trajectory of $x_{4}$ component of the generated motion plan.]{\includegraphics[width=0.15\textwidth]{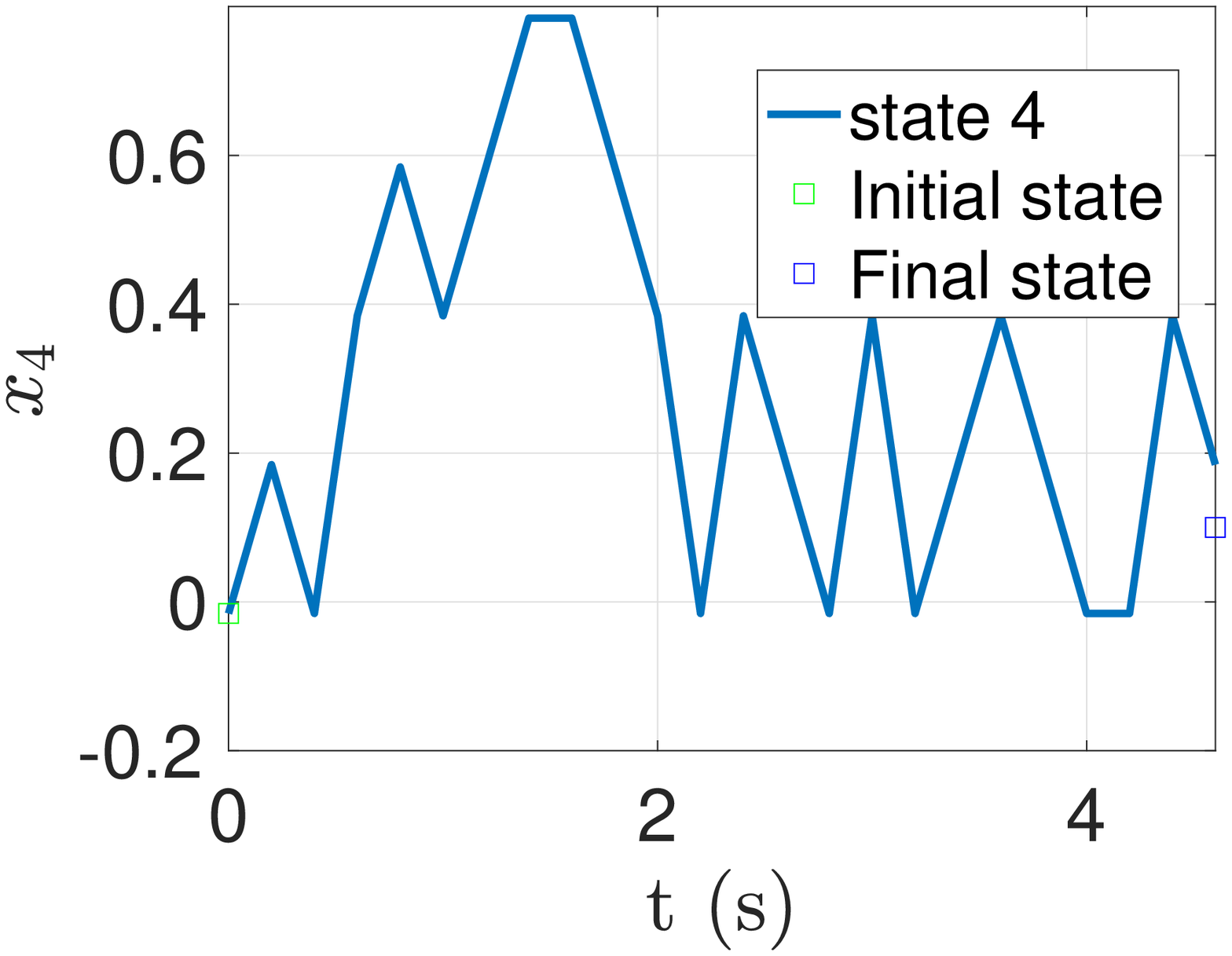}}
	\subfigure[The trajectory of $x_{5}$ component of the generated motion plan.]{\includegraphics[width=0.15\textwidth]{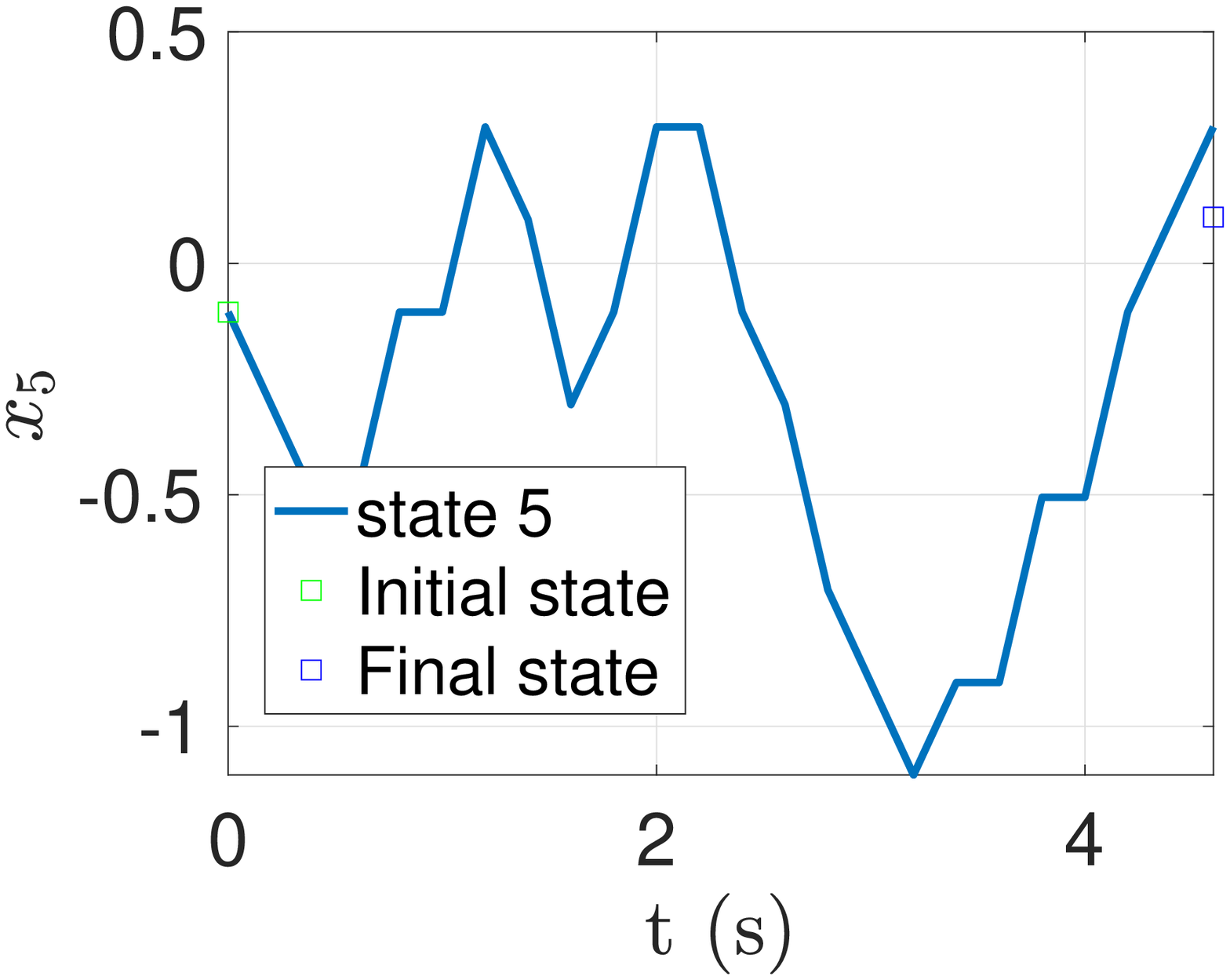}}
	\subfigure[The trajectory of $x_{6}$ component of the generated motion plan.]{\includegraphics[width=0.15\textwidth]{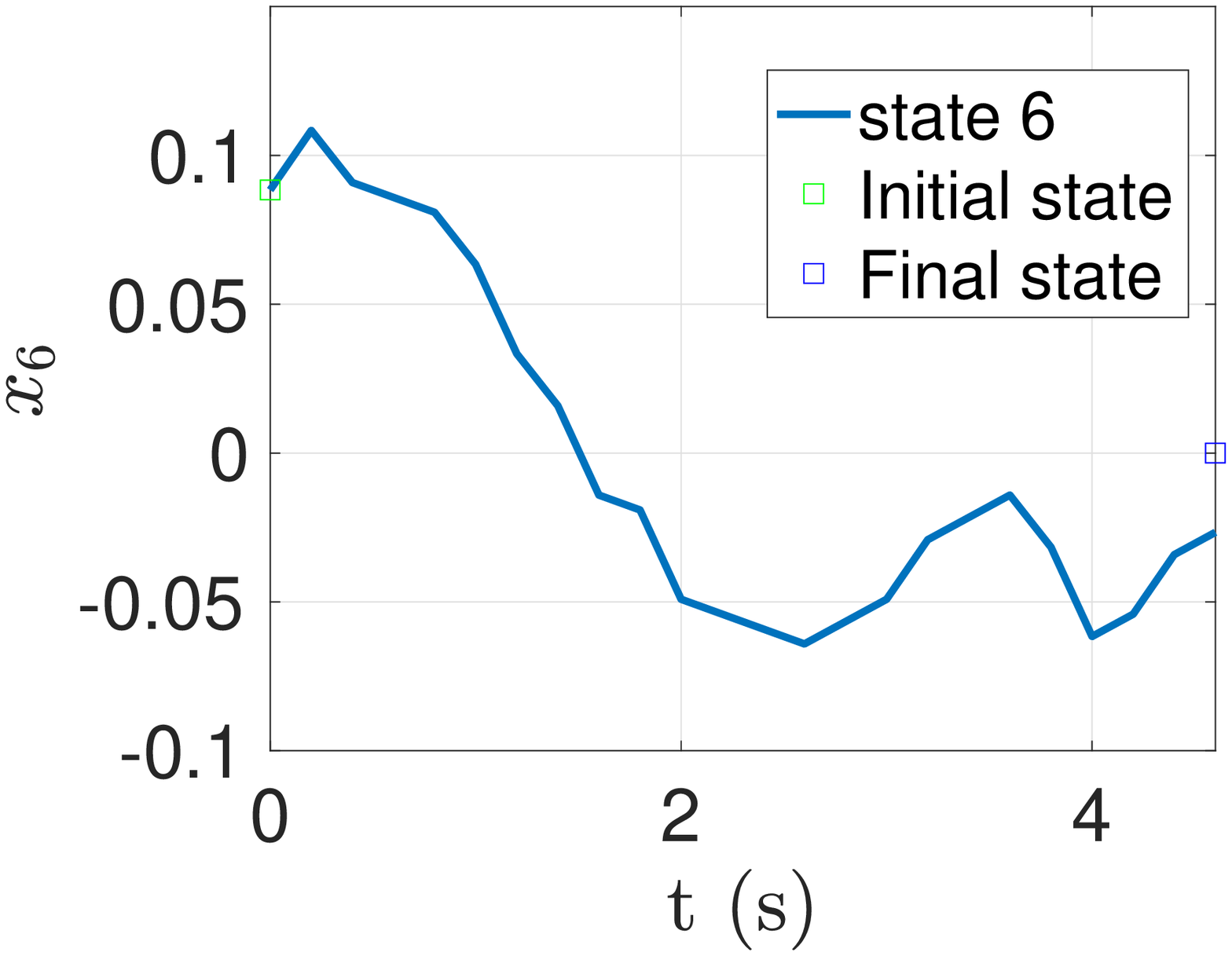}}
	\caption{The above shows the trajectories of each state components of the generated motion plan to the sample motion planning problem for biped system. In each figure in Figure \ref{fig:illustrationbiped}, green square denotes the corresponding component of $X_{0}$ and blue square denotes the corresponding component of $X_{f}$.}
	\label{fig:illustrationbiped}
\end{figure}}
\ifbool{conf}{\vspace{-0.3cm}}{}
}{
	\begin{example}(Actuated bouncing ball system in Example \ref{example:bouncingball}, revisited)\label{example:illustrationbb}
	This part illustrates that HyRRT algorithm is able to solve the instance of motion planning problem in Example \ref{example:bouncingball}.
	The inputs fed to the proposed algorithm are given as follows.
	\begin{enumerate}
	\item The input library $(\mathcal{U}_{C}, \mathcal{U}_{D})$ is constructed as follows. In this simulation, $\mathcal{U}_{C}$ only contains one input signal $[0, 0.1] \to \{0\}$. In this illustration, $\mathcal{U}_{D}$ is constructed as $\{0, 1, 2, 3, 4\}$. 
		\item The tolerance $\epsilon$ in (\ref{equation:tolerance}) is set to $0.2$, $K$ is set as infinity. The parameter $p_{n}$ is set as $0.5$.
	\end{enumerate}
	
	The simulation result is shown in Figure \ref{fig:hybridrrtresultbb}. The simulation is implemented in MATLAB software and processed by a $2.2$ GHz Intel Core i7 processor. The simulation takes $0.34$ seconds to finish.
	\begin{figure}[htbp]
		\centering
		\includegraphics[width = 0.23\textwidth]{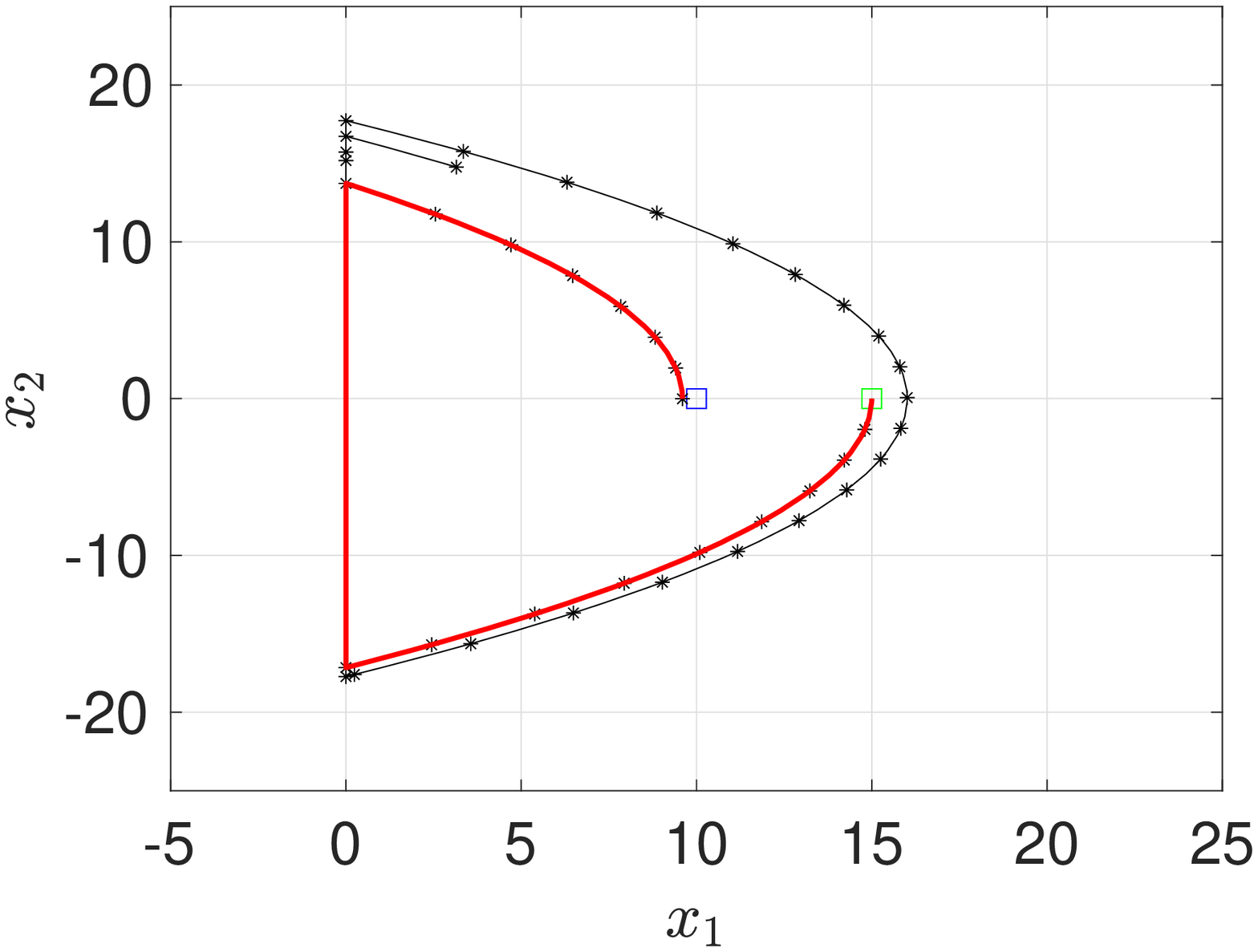}
		\caption{The above shows the search results of HyRRT algorithm to solve the sample motion planning problem in Example \ref{example:bouncingball}. Green square denotes $X_{0}$ and blue square denotes $X_{f}$. The states represented by vertices in the search tree are denoted by *'s. The lines between *'s denote the state trajectories of the solution pairs associated with edges in the search tree. The red trajectory denotes the state trajectory of a motion plan to the given motion planning problem.}\label{fig:hybridrrtresultbb}
	\end{figure}
\end{example}
	
}
The software tool also succeeds in finding motion plans for the actuated bouncing ball and point-mass robotics manipulator systems.

\section{Conclusion and Future Work}\label{section:conclusion}
In this paper, a HyRRT algorithm is proposed to solve motion planning problems for hybrid systems. The proposed algorithm is illustrated in the \ifbool{conf}{walking robot example}{bouncing ball example and walking robot example} and the results show its capacity to solve the \ifbool{conf}{problem}{problems}. In addition, this paper provides a result showing HyRRT algorithm is probabilistically complete under mild assumptions. 
Future research direction includes the optimal motion planning.

\ifbool{conf}{\vspace{-0.17cm}}{}
\bibliography{references.bib}
\bibliographystyle{IEEEtran}

\ifbool{conf}{}{
\appendix
\section{Appendix}
Lemma 2 in \cite{kleinbort2018probabilistic} is given as follows.
\begin{lemma}
	Let $\pi$, $\pi'$ be two trajectories, with the corresponding control functions $ \Upsilon(t)$, $\Upsilon'(t)$. Suppose that $x_{0} = \pi(0)$, $x'_{0} = \pi'(0)$ and $||x_{0} - x'_{0}|| \leq \delta$, for some constant $\delta > 0$. Let $T > 0$ be a time duration such that for all $t\in [0, T]$ it holds that $\Upsilon(t) = u$, $\Upsilon'(t)= u'$. That is, $\Upsilon$, $\Upsilon'$ remain fixed throughout $[0, T]$. Then
	$$
		||\pi(T) - \pi'(T)|| \leq e^{K_{x}T}\delta + K_{u}Te^{K_{x}T}\Delta u
	$$
	where $\Delta u = ||u - u'||$.
\end{lemma}
}
\end{document}